\pdfoutput=1
\documentclass[11pt]{article} % For LaTeX2e
\usepackage{times}
 \usepackage[letterpaper, left=1in, right=1in, top=1in, bottom=1in]{geometry}

\usepackage[utf8]{inputenc} % allow utf-8 input
\usepackage[T1]{fontenc}    % use 8-bit T1 fonts
\usepackage{url}            % simple URL typesetting
\usepackage{booktabs}       % professional-quality tables
\usepackage{amsfonts}       % blackboard math symbols
\usepackage{nicefrac}       % compact symbols for 1/2, etc.
\usepackage{microtype}      % microtypography
\usepackage{mkolar_definitions}
\usepackage{xspace}
\usepackage{algorithm,algorithmicx,algpseudocode}
\usepackage{color}
\usepackage{enumitem}
\usepackage{comment}

\usepackage[scaled=.9]{helvet}

\usepackage{url}
\usepackage{authblk}

\usepackage{parskip}

\usepackage[dvipsnames]{xcolor}
 \usepackage[colorlinks=true, linkcolor=blue, citecolor=blue]{hyperref}
\usepackage{natbib}
\bibpunct{(}{)}{;}{a}{,}{,}

\newenvironment{packed_enum}{
  \begin{enumerate}
    \setlength{\itemsep}{1pt}
    \setlength{\parskip}{-1pt}
    \setlength{\parsep}{0pt}
}{\end{enumerate}}

\newcommand{\insertprethmspacing}{\vspace{0.05in}}
\newcommand{\insertpostthmspacing}{\vspace{-0.05in}}

\newcommand{\poly}{\textrm{poly}}
\newcommand{\otil}{\tilde{\mathcal{O}}}
\newcommand{\order}{\mathcal{O}}
\newcommand{\Reg}{\textrm{Regret}}
\newcommand{\ramp}{\phi^\gamma}
\newcommand{\hinge}{\psi^\gamma}
\newcommand{\Vol}{\textrm{Vol}}
\newcommand{\lmc}{\textsc{Hinge-LMC}}
\newcommand{\ftl}{\textsc{SmoothFTL}}
\newcommand{\RRz}{\RR_{=0}}
\newcommand{\piramp}{\pi_{\textrm{ramp}}}
\newcommand{\pihinge}{\pi_{\textrm{hinge}}}
\newcommand{\Rhinge}{R^\psi}
\newcommand{\Rhathinge}{\hat{R}^\psi}
\newcommand{\expfour}{\textsc{Exp4}\xspace}
\newcommand{\hedge}{\textsc{Hedge}\xspace}
\newcommand{\linucb}{\textsc{LinUCB}\xspace}
\newcommand{\banditron}{\textsc{Banditron}\xspace}

\newcommand{\inftycover}{\cN_{\infty,\infty}}

\usepackage{amsthm}
\usepackage{mathtools}
\usepackage{amsmath}
\usepackage{bbm}
\usepackage{amssymb}

\usepackage{MnSymbol} %Actually conflicts with amssymb and others

\usepackage{xpatch}

\theoremstyle{definition}  %Sets style of subsequent newtheorems to 'definition'

\theoremstyle{plain}

\xpatchcmd{\proof}{\itshape}{\normalfont\proofnameformat}{}{}
\newcommand{\proofnameformat}{\bfseries}

\usepackage{prettyref}
\newcommand{\pref}[1]{\prettyref{#1}}
\newcommand{\pfref}[1]{Proof of \prettyref{#1}}
\newcommand{\savehyperref}[2]{\texorpdfstring{\hyperref[#1]{#2}}{#2}}
\newrefformat{eq}{\savehyperref{#1}{\textup{(\ref*{#1})}}}
\newrefformat{eqn}{\savehyperref{#1}{Equation~\ref*{#1}}}
\newrefformat{lem}{\savehyperref{#1}{Lemma~\ref*{#1}}}
\newrefformat{def}{\savehyperref{#1}{Definition~\ref*{#1}}}
\newrefformat{line}{\savehyperref{#1}{line~\ref*{#1}}}
\newrefformat{thm}{\savehyperref{#1}{Theorem~\ref*{#1}}}
\newrefformat{corr}{\savehyperref{#1}{Corollary~\ref*{#1}}}
\newrefformat{sec}{\savehyperref{#1}{Section~\ref*{#1}}}
\newrefformat{ssec}{\savehyperref{#1}{Subsection~\ref*{#1}}}
\newrefformat{app}{\savehyperref{#1}{Appendix~\ref*{#1}}}
\newrefformat{ass}{\savehyperref{#1}{Assumption~\ref*{#1}}}
\newrefformat{ex}{\savehyperref{#1}{Example~\ref*{#1}}}
\newrefformat{fig}{\savehyperref{#1}{Figure~\ref*{#1}}}
\newrefformat{alg}{\savehyperref{#1}{Algorithm~\ref*{#1}}}
\newrefformat{rem}{\savehyperref{#1}{Remark~\ref*{#1}}}
\newrefformat{conj}{\savehyperref{#1}{Conjecture~\ref*{#1}}}
\newrefformat{prop}{\savehyperref{#1}{Proposition~\ref*{#1}}}
\newrefformat{proto}{\savehyperref{#1}{Protocol~\ref*{#1}}}
\newrefformat{prob}{\savehyperref{#1}{Problem~\ref*{#1}}}
\newrefformat{claim}{\savehyperref{#1}{Claim~\ref*{#1}}}
\newrefformat{fact}{\savehyperref{#1}{Fact~\ref*{#1}}}

\DeclarePairedDelimiter{\abs}{\lvert}{\rvert} %
\DeclarePairedDelimiter{\brk}{[}{]}
\DeclarePairedDelimiter{\crl}{\{}{\}}
\DeclarePairedDelimiter{\prn}{(}{)}
\DeclarePairedDelimiter{\nrm}{\|}{\|}
\DeclarePairedDelimiter{\tri}{\langle}{\rangle}
\DeclarePairedDelimiter{\dtri}{\llangle}{\rrangle}

\DeclarePairedDelimiter{\floor}{\lfloor}{\rfloor}

\DeclareMathOperator{\En}{\mathbb{E}}

\newcommand{\ls}{\ell}
\newcommand{\pmo}{\crl*{\pm{}1}}
\newcommand{\eps}{\epsilon}
\newcommand{\veps}{\varepsilon}

\newcommand{\ldef}{\vcentcolon=}

\newcommand{\wt}[1]{\widetilde{#1}}

\def\ddefloop#1{\ifx\ddefloop#1\else\ddef{#1}\expandafter\ddefloop\fi}
\def\ddef#1{\expandafter\def\csname bb#1\endcsname{\ensuremath{\mathbb{#1}}}}
\ddefloop ABCDEFGHIJKLMNOPQRSTUVWXYZ\ddefloop
\def\ddefloop#1{\ifx\ddefloop#1\else\ddef{#1}\expandafter\ddefloop\fi}
\def\ddef#1{\expandafter\def\csname b#1\endcsname{\ensuremath{\mathbf{#1}}}}
\ddefloop ABCDEFGHIJKLMNOPQRSTUVWXYZ\ddefloop
\def\ddef#1{\expandafter\def\csname c#1\endcsname{\ensuremath{\mathcal{#1}}}}
\ddefloop ABCDEFGHIJKLMNOPQRSTUVWXYZ\ddefloop
\def\ddef#1{\expandafter\def\csname h#1\endcsname{\ensuremath{\widehat{#1}}}}
\ddefloop ABCDEFGHIJKLMNOPQRSTUVWXYZ\ddefloop
\def\ddef#1{\expandafter\def\csname hc#1\endcsname{\ensuremath{\widehat{\mathcal{#1}}}}}
\ddefloop ABCDEFGHIJKLMNOPQRSTUVWXYZ\ddefloop
\def\ddef#1{\expandafter\def\csname t#1\endcsname{\ensuremath{\widetilde{#1}}}}
\ddefloop ABCDEFGHIJKLMNOPQRSTUVWXYZ\ddefloop
\def\ddef#1{\expandafter\def\csname tc#1\endcsname{\ensuremath{\widetilde{\mathcal{#1}}}}}
\ddefloop ABCDEFGHIJKLMNOPQRSTUVWXYZ\ddefloop

\newcommand{\Holder}{H{\"o}lder}

\renewcommand{\mb}[1]{\boldsymbol{#1}}

\newcommand{\bz}{\mb{z}}
\newcommand{\bx}{\mb{x}}

\newcommand{\bv}{\mb{v}}
\newcommand{\bw}{\mb{w}}
\newcommand{\bu}{\mb{u}}
\newcommand{\bls}{\mb{\ls}}

\newcommand{\defeq}{\triangleq}

\newcommand{\Otilde}{\tilde{O}}

\newcommand{\expweights}{exponential weights\xspace}

\newcommand{\bistro}{\textsc{BISTRO}\xspace}

\newcommand{\grad}{\nabla}

\newcommand{\dmid}{\;||\;}
\newcommand{\KL}{\mathrm{KL}}

\title{\huge Contextual bandits with surrogate losses:\\Margin bounds and efficient algorithms}
\date{}
\author[1]{
Dylan J. Foster
\thanks{djfoster@cs.cornell.edu}}
\author[2]{
Akshay Krishnamurthy
\thanks{akshay@cs.umass.edu}}

\affil[1]{Cornell University, Ithaca, NY}
\affil[2]{Microsoft Research, New York, NY}

\begin{document}

\maketitle

\begin{abstract}
  We use surrogate losses to obtain several new regret bounds and new
  algorithms for contextual bandit learning. Using the ramp loss, we
  derive new margin-based regret bounds in terms of standard
  sequential complexity measures of a benchmark class of real-valued
  regression functions. Using the hinge loss, we derive an efficient
  algorithm with a $\sqrt{dT}$-type mistake bound against benchmark
  policies induced by $d$-dimensional regressors. Under realizability
  assumptions, our results also yield classical regret bounds.
\end{abstract}

\section{Introduction}
\label{sec:intro}
% !TEX root = arxiv.tex

%\dylan{Add classical multiclass margin loss for the information-theoretic results.}
%\dylan{Explain the classical $\sqrt{d/n}$ vs. $\sqrt{d/\gamma^2{}n}$ tradeoff in statistical learning.}
%% \akshay{Try coordinate domination at all scales of chaining.}

We study sequential prediction problems with partial feedback,
mathematically modeled as \emph{contextual
  bandits}~\citep{langford2008epoch}. In this formalism, a learner
repeatedly (a) observes a \emph{context}, (b) selects an
\emph{action}, and (c) receives a \emph{loss} for the chosen
action. The objective is to learn a policy for selecting actions with
low loss, formally measured via \emph{regret} with respect to a
class of benchmark policies.  Contextual bandit algorithms have
been successfully deployed in online recommendation
systems~\citep{agarwal2016making}, mobile health
platforms~\citep{tewari2017ads}, and elsewhere.

In this paper, we use surrogate loss functions to derive new
\emph{margin-based} algorithms and regret bounds for contextual
bandits. Surrogate loss functions are ubiquitous in supervised
learning
(cf.~\cite{zhang2004statistical,bartlett2006convexity,schapire2012boosting}).
Computationally, they are used to replace NP-hard optimization
problems with tractable ones, e.g., the hinge loss
makes binary classification amenable to convex programming
techniques. Statistically, they also enable sharper generalization
analysis for models including boosting, SVMs, and neural
networks~\citep{schapire2012boosting,anthony2009neural}, by replacing dependence on dimension in VC-type bounds with
distribution-dependent quantities. 
For example, to agnostically learn $d$-dimensional halfspaces the optimal rates for excess risk are $\sqrt{d/n}$ for the $0/1$ loss benchmark
and $\frac{1}{\gamma}\cdot\sqrt{1/n}$ for the $\gamma$-margin loss benchmark~\citep{kakade2009complexity}, meaning \emph{the margin bound removes explicit dependence on dimension}.
Curiously, surrogate losses have
seen limited use in partial information settings (some exceptions are
discussed below).  This paper demonstrates that these desirable
computational and statistical properties indeed extend to contextual
bandits.

In the first part of the paper we focus on statistical issues, namely
whether \emph{any algorithm} can achieve a generalization of the
classical margin bound from statistical
learning \citep{boucheron2005theory} in the adversarial contextual
bandit setting. Our aim here is to introduce a theory of learnability
for contextual bandits, in analogy with statistical and online
learning, and our results provide an information-theoretic benchmark
for future algorithm designers. We consider benchmark policies induced
by a class of real-valued regression functions and obtain a regret
bound in terms of the class' \emph{sequential metric entropy}, a
standard complexity measure in online
learning~\citep{rakhlin2015sequential}.
%% our regret
%% bound involves standard sequential complexity measures of $\Fcal$,
%% notably the \emph{sequential metric
%% entropy} \citep{rakhlin2015sequential}. 
As a consequence, we show that
$\otil(T^{\frac{d}{d+1}})$ regret is achievable for Lipschitz
contextual bandits in $d$-dimensional metric spaces, improving on a
recent result of~\citet{cesa2017algorithmic}, and that an
$\Otilde(T^{2/3})$ mistake bound is achievable for bandit multiclass
prediction in smooth Banach spaces, extending~\citet{kakade2008efficient}.

%% Technically, to provide an analogue of the classical margin theory, we
%% must overcome several challenges. First, since we operate in the
%% online adversarial setting, there is no generic algorithmic
%% counterpart to empirical risk minimization that we can use to analyze
%% statistical behavior of arbitrary classes. Instead, we build on the
%% non-constructive minimax analysis of~\citet{rakhlin2015sequential}
%% which, in the full-information setting, yields regret %% has been used to derive regret
%% bounds in terms of sequential analogues of standard complexity
%% measures including Rademacher complexity and metric entropy. Second,
%% since we work in the contextual bandit setting, we must extend these
%% arguments to incorporate partial information. To this end, we leverage
%% the adaptive minimax framework of~\citet{foster2015adaptive} along
%% with a careful ``adaptive" chaining argument.

Technically, these results build on the non-constructive minimax analysis
of~\citet{rakhlin2015sequential}, which, for the online adversarial
setting, prescribes a recipe for characterizing statistical behavior
of arbitrary classes, and thus provides a counterpart to empirical
risk minimization in statistical learning. Indeed, for
full-information problems, this approach yields regret bounds in terms
of sequential analogues of standard complexity measures including
Rademacher complexity and metric entropy. However, since we work in
the contextual bandit setting, we must extend these arguments to
incorporate partial information. To do so, we leverage the adaptive
minimax framework of~\citet{foster2015adaptive} along with a careful
``adaptive" chaining argument.

In the second part of the paper, we focus on computational issues and
derive two new algorithms using the hinge loss as a convex
surrogate. The first algorithm, \lmc, provably runs in polynomial time
and achieves a $\sqrt{dT}$-mistake bound against $d$-dimensional
benchmark regressors with convexity properties.  \lmc\xspace is the
first efficient algorithm with $\sqrt{dT}$-mistake bound for bandit
multiclass prediction using a surrogate loss without curvature, and so
it provides a new resolution to the open problem
of~\citet{abernethy2009efficient}.  This algorithm is based on the
\expweights update, along with Langevin Monte Carlo for efficient
sampling and a careful action selection scheme.  
The second algorithm
is much simpler: in the stochastic setting,
Follow-The-Leader with appropriate smoothing matches our information-theoretic results for sufficiently large classes.
%regressor classes
%with sufficiently large metric entropy growth rate.

%% provided the regressor class 
%% is sufficiently large (in terms of
%% metric entropy growth rate).

\subsection{Preliminaries}
\label{ssec:prelims}
Let $\Xcal$ denote a context space and $\Acal = \{1,\ldots,K\}$
a discrete action space. In the adversarial contextual bandits
problem, for each of $T$ rounds, an adversary chooses a pair
$(x_t,\ell_t)$ where $x_t \in \Xcal$ is the context and
$\ell_t \in [0,1]^K\triangleq\cL$ is a loss vector. The learner
observes the context $x_t$, chooses an action $a_t$, and incurs loss
$\ell_t(a_t) \in [0,1]$, which is also observed. The goal of the
learner is to minimize the cumulative loss over the $T$
rounds, and, in particular, we would like to design learning
algorithms that achieve low \emph{regret} against a class
$\Pi \subset \prn*{\Xcal \to \Acal}$ of benchmark policies:
\begin{equation*}
  \Reg(T,\Pi) \triangleq \sum_{t=1}^T\EE[\ell_t(a_t)] - \inf_{\pi \in \Pi} \sum_{t=1}^T \EE[\ell_t(\pi(x_t))].
\end{equation*}
In this paper, we always identify $\Pi$ with a class of vector-valued
regression functions $\Fcal \subset (\Xcal \to \RRz^K)$, where $\RRz^K
\triangleq \{s \in \RR^K: \sum_a s_a = 0\}$. We use the notation $f(x)
\in \RR^K$ to denote the vector-valued output and $f(x)_a$ to denote
the $a^{\textrm{th}}$ component.  Note that we are assuming $\sum_a
f(x)_a = 0$, which is a natural generalization of the standard
formulation for binary classification~\citep{bartlett2006convexity}
and appears in~\citet{pires2013cost}. Define
$B\triangleq\sup_{f\in\cF}\sup_{x\in\cX}\nrm*{f(x)}_{\infty}$ to be
the maximum value predicted by any regressor.

Our algorithms use \emph{importance weighting} to form unbiased loss
estimates. If at round $t$, the algorithm chooses action $a_t$ by
sampling from a distribution $p_t \in \Delta(\Acal)$, the loss
estimate is defined as
$\hat{\ell}_t(a) \triangleq \ell_t(a_t)\one\{a_t=a\}/p_t(a)$. Given
$p_t$, we also define a smoothed distribution as
$p_t^\mu \triangleq (1-K\mu)p_t + \mu$ for some parameter
$\mu \in [0,1/K]$.

We introduce two surrogate loss functions, the \emph{ramp loss} and
the \emph{hinge loss}, whose scalar versions are defined as $\ramp(s)
\triangleq \min(\max(1+s/\gamma,0),1)$ and $\hinge(s) \triangleq
\max(1+s/\gamma,0)$ respectively, for $\gamma > 0$. For $s\in\bbR^{K}$, $\ramp(s)$ and
$\hinge(s)$ are defined coordinate-wise.  We start with a simple
lemma, demonstrating how $\ramp,\hinge$ act as surrogates for
cost-sensitive multiclass losses.
\insertprethmspacing
\begin{lemma}[Surrogate Loss Translation]
\label{lem:calibration}
For $s \in \RRz^K$, define $\piramp(s),\pihinge(s) \in \Delta(\Acal)$ by
$\piramp(s)_a \propto \ramp(s_a)$ and $\pihinge(s)_a \propto \hinge(s_a)$. For any vector $\ell \in \RR^K_+$, we have
\begin{align*}
  \langle \piramp(s), \ell\rangle &\leq \langle \ell,\ramp(s)\rangle \leq \sum_{a \in \Acal}\ell(a)\textup{\one}\{s_a \ge -\gamma\}, \qquad\text{and}\qquad
  \langle \pihinge(s), \ell\rangle \leq K^{-1} \langle \ell,\hinge(s)\rangle.
\end{align*}
\end{lemma}
\insertpostthmspacing 
Based on this lemma, it will be convenient to
define $L_T^\gamma(f) \triangleq \sum_{t=1}^T\sum_{a \in \Acal}
\ell_t(a) \one\{f(x_t)_a \ge -\gamma\}$, which is the
\emph{margin-based cumulative loss} for the regressor
$f$. $L_{T}^{\gamma}$ should be seen as a cost-sensitive multiclass
analogue of the classical margin loss in statistical learning
\citep{boucheron2005theory}. We use the term ``surrogate loss''
here because these quantities upper bound the cost-sensitive loss:
$\ell(\argmax_a s_a) \leq \langle \ell,\ramp(s)\rangle \leq \langle
\ell, \hinge(s)\rangle$.\footnote{On a related note, the information-theoretic results we present are also compatible with the surrogate function $\theta^{\gamma}(s)_a \ldef \max\crl*{1+(s_a - \max_{a'}s_{a'})/\gamma,0}$, which also satisfies $\ell(\argmax_a s_a) \leq \langle \ell,\theta^{\gamma}(s)\rangle$. This leads to a perhaps more standard notion of multiclass margin bound but does not lead to efficient algorithms.
} In the sequel, $\piramp$ and $\pihinge$ are
used by our algorithms, but do not define the benchmark policy class,
since we compare directly to $L_T^\gamma$ or the surrogate loss.

\paragraph{Related work.}
Contextual bandit learning has been the subject of intense
investigation over the past decade. The most natural categorization of
these works is between parametric, realizability-based, and agnostic
approaches. Parametric methods
(e.g.,~\cite{abbasi2011improved,chu2011contextual}) assume a
(generalized) linear relationship between the losses and the
contexts/actions. Realizability-based methods generalize parametric
ones by assuming the losses are predictable by some abstract
regression
class~\citep{agarwal2012contextual,foster2018practical}. Agnostic
approaches
(e.g.,~\cite{auer2002nonstochastic,langford2008epoch,agarwal2014taming,rakhlin2016bistro,syrgkanis2016efficient,syrgkanis2016improved})
avoid realizability assumptions and instead compete with VC-type
policy classes for statistical tractability. Our work contributes to
all of these directions, as our margin bounds apply to the agnostic
adversarial setting and yield true regret bounds under realizability
assumptions.

A special case of contextual bandits is \emph{bandit multiclass
  prediction}, where the loss vector is zero for one action and one
for all others~\citep{kakade2008efficient}. Several recent papers
obtain surrogate regret bounds and efficient algorithms for this
setting when the benchmark regressor class $\Fcal$ consists of linear
functions~\citep{kakade2008efficient,hazan2011newtron,beygelzimerOZ17,foster2018logistic}. Our
work contributes to this line in two ways: our bounds and algorithms extend
%% we derive surrogate regret
%% bounds and efficient algorithms 
beyond linear/parametric classes, and
we consider the more general contextual bandit setting.

Our information-theoretic results on achievability are similar in spirit those of
\citet{daniely2013price}, who derive tight generic bounds for
bandit multiclass prediction in terms of the Littlestone
dimension. This result is incomparable to our own: their bounds are on
the $0/1$ loss regret directly rather than surrogate regret, but the
Littlestone dimension is not a tight complexity measure for
real-valued function classes in agnostic settings, which is our focus.

At a technical level, our work builds on several recent results.
%% incorporates a number of statistical
%% and algorithmic ideas from the literature.
 To derive achievable regret bounds, we use the adaptive minimax
 framework of \citet{foster2015adaptive}, along with a new adaptive
 chaining argument to control the supremum of a martingale
 process~\citep{rakhlin2015sequential}. Our \lmc\xspace algorithm is
 based on log-concave sampling~\citep{bubeck2015sampling}, and it uses
 randomized smoothing~\citep{duchi2012randomized} and the geometric
 resampling trick of~\citet{neu2013efficient}. We also use several
 ideas from classification
 calibration~\citep{zhang2004statistical,bartlett2006convexity}, and,
 in particular, the surrogate hinge loss we work with is studied
 by~\citet{pires2013cost}.

%%% Local Variables:
%%% mode: latex
%%% TeX-master: "paper"
%%% End:

\section{Achievable regret bounds}
\label{sec:minimax}
% !TEX root = arxiv.tex

This section provides generic surrogate regret bounds for contextual
bandits in terms of the sequential metric
entropy~\citep{RakSriTew14jmlr} of the regressor class $\cF$. Notably, our general techniques apply when the ramp loss is used as a surrogate, and so, via~\pref{lem:calibration}, they yield the main result of the section----a margin-based regret guarantee---as a special case.

To motivate our approach, consider a well-known reduction from bandits to full information online learning:
If a full information algorithm achieves a regret bound in terms of
the so-called \emph{local norms} $\sum_t\langle p_t,\ell_t^2\rangle$, then
running the full information algorithm on importance-weighted losses $\hat{\ls}_{t}(a)$
yields an expected regret bound for the bandit setting. For example, when $\Pi$ is finite,
\expfour{} \citep{auer2002nonstochastic} %%  exploits this observation for the case when $\Pi$ is finite,
uses \hedge~\citep{freund1997decision} as the full information algorithm, and obtains a deterministic regret bound of 
\begin{align}
  \label{eq:exp4}
  \Reg(T,\Pi) \leq \frac{\eta}{2} \sum_{t=1}^T \En_{\pi\sim{}p_t}\tri*{\pi(x_t), \hat{\ls}_{t}}^{2} + \frac{\log (|\Pi|)}{\eta},
\end{align}
where $\eta>0$ is the learning rate and $p_{t}$ is the distribution
over policies in $\Pi$ (inducing an action distribution) for round
$t$. Evaluating conditional expectations and optimizing 
$\eta$ yields a regret bound of $\order(\sqrt{KT\log(|\Pi|)})$,
which is optimal for contextual bandits with a finite policy class. 

To use this reduction beyond the finite class case and with
surrogate losses we face two challenges:
\begin{enumerate}[leftmargin=*]
\item \textbf{Infinite classes.} The natural approach of using a
  pointwise (or sup-norm) cover for $\cF$ is insufficient---not only
  because there are classes that have infinite pointwise covers yet
  are online-learnable, but also because it yields sub-optimal rates
  even when a finite pointwise cover is available. Instead, we
  establish existence of a full-information algorithm for large
  nonparametric classes that has 1) strong adaptivity to loss scaling
  as in \pref{eq:exp4} and 2) regret scaling with the sequential
  covering number for $\cF$, which is the correct generalization of
  the empirical covering number in statistical learning to the
  adversarial online setting. This is achieved via non-constructive
  methods.
  
\item \textbf{Variance control}. With surrogate losses, controlling
  the variance/local norm term $\EE_{\pi}\langle\pi(x_t), \hat{\ell}_t\rangle^2$ in the
  reduction from bandit to full information is more challenging, since
  the surrogate loss of a policy depends on the scale of the underlying regressor, not just the action it selects. To address this, we develop a new sampling scheme tailored to scale-sensitive losses.
\end{enumerate}

\paragraph{Full-information regret bound.}
We consider the following full information protocol, which in the sequel will be instantiated via reduction from contextual bandits. Let the context space $\cX$ and $\Acal$ be fixed as in \pref{ssec:prelims}, and consider a function class $\cG\subset(\cX\to\cS)$, where $\cS\subseteq{}\bbR^{K}_{+}$. The reader may think of $\cG$ as representing $\ramp\circ\cF$ or $\hinge\circ\cF$, i.e. the surrogate loss composed with the regressor class, so that $\cS$ (which is not necessarily convex) represents the image of the surrogate loss over $\cF$.

The online learning protocol is: For time $t=1,\ldots,T$, (1) the learner observes $x_t$ and chooses a distribution $p_t\in\Delta(\cS)$, (2) the adversary picks a loss vector $\ls_{t}\in{}\cL\subset{}\bbR^{K}_{+}$, (3) the learner samples outcome $s_t\sim{}p_t$ and experiences loss $\tri*{s_t, \ls_t}$. Regret against the benchmark class $\cG$ is given by
\begin{equation*}
  \sum_{t=1}^{T}\En_{s_t\sim{}p_t}\tri*{s_t,\ls_{t}} - \inf_{g\in\cG}\sum_{t=1}^{T}\tri*{g(x_t),\ls_t}.
\end{equation*}
As our complexity measure, we use a multi-output generalization of \emph{sequential covering numbers} introduced by \citet{RakSriTew14jmlr}. Define a \emph{$\cZ$-valued tree} $\bz$ to be a sequence of mappings $\bz_{t}:\pmo^{t-1}\to\cZ$.
The tree $\bz$ is a complete rooted binary tree with nodes labeled by elements of $\cZ$, where for any ``path'' $\eps\in\pmo^{T}$, $\bz_{t}(\eps)\triangleq\bz_{t}(\eps_{1:t-1})$ is the value of the node at level $t$ on the path $\eps$.
\begin{definition}
  \label{def:cover}
  For a function class $\cG\subset(\cX\to\bbR^{K})$ and $\cX$-valued tree $\bx$ of length $T$, the $L_{\infty}/\ls_{\infty}$ sequential covering number\footnote{Sequential coverings for $L_p/\ls_q$ can be defined similarly, but do not appear in the present paper.} for $\cG$ on $\bx$ at scale $\veps$, denoted by $\inftycover(\veps, \cG, \bx)$, is the cardinality of the smallest set $V$ of $\RR^K$-valued trees for which
  \begin{equation}
\forall{}g\in\cG\;\forall{}\eps\in\pmo^{T}\;\exists{}\bv\in{}V\textnormal{ s.t. }\max_{t\in\brk*{T}}\nrm*{g(\bx_{t}(\eps)) - \bv_{t}(\eps)}_{\infty}\leq\veps.
    \end{equation}
  Define $\inftycover(\veps, \cG, T) \defeq \sup_{\bx : \mathrm{length}(\bx)=T}\inftycover(\veps, \cG, \bx)$.
\end{definition}
We refer to $\log\cN_{\infty,\infty}$ as the \emph{sequential metric entropy}. Note that in the binary case, for learning unit $\ls_2$ norm linear functions in $d$ dimensions, the pointwise metric entropy is $O(d\log(1/\veps))$, whereas the sequential metric entropy is $O(d\log(1/\veps)\wedge\veps^{-2}\log(d))$, leading to improved rates in high dimension.

With this definition, we can now state our main theorem for full information. 
\insertprethmspacing
\begin{theorem}
  \label{thm:chaining_v3}
  Assume\footnote{Measuring loss in $\ls_1$ may seem restrictive, but it is natural when working with the $1$-sparse importance-weighted losses, and it enables us to cover the output space in $\ls_{\infty}$ norm.} $\sup_{\ls \in \cL} \nrm*{\ls}_{1}\leq{}R$ and $\sup_{s\in\cS}\nrm*{s}_{\infty}\leq{}B$. Fix any constants $\eta\in(0,1]$, $\lambda>0$, and $\beta>\alpha>0$. Then there exists an algorithm with the following deterministic regret guarantee:
  {\small
  \begin{align*}
    \sum_{t=1}^{T}\En_{s_t\sim{}p_t}\tri*{
      s, \ls_t} -\inf_{g\in\cG}\sum_{t=1}^{T}\tri*{g(x_t), \ls_t} &\leq{} 
      \frac{2\eta}{RB}\sum_{t=1}^{T}\En_{s_t\sim{}p_t}\tri*{s_t, \ls_t}^{2} +  \frac{4RB}{\eta}\log\cN_{\infty,\infty}(\beta/2, \cG, T) + 3e^2\alpha\sum_{t=1}^{T}\nrm*{\ls_{t}}_{1}\\
   &~~~~~ +24e\prn*{\frac{\lambda}{4R}\sum_{t=1}^{T}\nrm*{\ls_{t}}_{1}^{2} + \frac{R}{\lambda}}\int_{\alpha}^{\beta}\sqrt{\log\cN_{\infty,\infty}(\veps, \cG, T)}d\veps.
  \end{align*}}
  \end{theorem}
\insertpostthmspacing
  Observe that the bound involves the variance/local norms
  $\En_{s_t\sim{}p_t}\tri*{s_t, \ls_t}^{2}$, and has a very mild
  explicit dependence on the loss range $R$; this can be verified by
  optimizing over $\eta$ and $\lambda$.  This adaptivity to the loss
  range is crucial for our bandit reduction. Further observe that the
  bound contains a Dudley-type entropy integral, which is essential
  for obtaining sharp rates for complex nonparametric classes.

\insertpostthmspacing
\paragraph{Bandit reduction and variance control.}
To lift \pref{thm:chaining_v3} to contextual bandits we use the following reduction: %% with the ramp loss we use the following reduction:
First, initialize the full information algorithm from \pref{thm:chaining_v3} with $\cG=\ramp\circ \cF$. For each round $t$, receive $x_t$, and define $P_t(a) \defeq \En_{s_t\sim{}p_t}\frac{s_t(a)}{\sum_{a'\in\brk*{K}}s_t(a')}$ where $p_t$ is the full information algorithm's distribution. Then sample $a_t \sim P_t^\mu$, observe $\ls_t(a_t)$, and pass the importance-weighted loss $\hat{\ls}_{t}(a)$ back to the algorithm. 
For the hinge loss we use the same strategy, but with $\cG=\hinge\circ\cF$.

The following lemma shows that this strategy leads to sufficiently
small variance in the loss estimates. The definition of the action
distribution $P^\mu_t(a)$ in terms of the real-valued predictions is
crucial here.
\insertprethmspacing
\begin{lemma}
  \label{lem:ips_variance}
  Define a filtration $\cJ_{t}=\sigma((x_1,\ls_1,a_1),\ldots, (x_{t-1},\ls_{t-1},a_{t-1}), x_t, \ls_t)$. Then for any $\mu\in [0,1/K]$ the importance weighting strategy above guarantees
  \[
    \En_{a_t\sim P_t^\mu}\brk*{\En_{s_t\sim{}p_t}\tri*{s_t,\hat{\ls}_{t}}^{2}\mid{}\cJ_t} \leq{} \left\{
      \begin{array}{ll}
        K,&\quad\text{ for $\Scal\subset\Delta(\Acal)$.}\\
        K^{2},&\quad\text{ for $\Scal = \ramp\circ \Fcal$.}\\
        \prn*{1+\frac{B}{\gamma}}^{2}K^{2},&\quad\text{ for $\Scal = \hinge\circ\Fcal$.}
      \end{array}
    \right.
  \]
\end{lemma}
\insertpostthmspacing

\pref{thm:chaining_v3} and \pref{lem:ips_variance} together imply our central theorem: a chaining-based margin bound for contextual bandits, generalizing classical results in statistical learning (cf.~\citep{boucheron2005theory}).

\insertprethmspacing
\begin{theorem}[Contextual bandit margin bound]
\label{thm:chaining_ramp}

For any fixed constants $\beta>\alpha>0$, smoothing parameter $\mu\in(0,1)$ and margin loss parameter $\gamma>0$ there exists an adversarial contextual bandit strategy with expected regret against the $\gamma$-margin benchmark bounded as
{\small
  \begin{align}
    \label{eq:chaining_ramp}
    \En\brk*{\sum_{t=1}^{T}\ls_t(a_t)} 
    \leq{} &\inf_{f\in\cF}\En\brk*{L_T^\gamma(f)} + %\sum_{t=1}^{T}\sum_{a \in \Acal}\ls_t(a)\one\{f(x_t)_a \ge -\gamma\}} + 
        4\sqrt{2K^{2}T\log\cN_{\infty,\infty}(\beta/2, \cF, T)}  + \mu{}KT \\
    &~~~~~~~+  \frac{8}{\mu}\log\cN_{\infty,\infty}(\beta/2, \cF, T) + \frac{1}{\gamma}\prn*{3e^2\alpha{}KT + 24e\sqrt{\frac{KT}{\mu}}\int_{\alpha}^{\beta}\sqrt{\log\cN_{\infty,\infty}(\veps, \cF, T)}d\veps}.\notag
  \end{align}}
\end{theorem}
\insertpostthmspacing
We derive an analogous bound for the hinge loss in \pref{app:minimax}. The hinge loss bound differs only through stronger dependence on scale parameters.

Before showing the implications of \pref{thm:chaining_ramp} for specific classes $\cF$ we state a coarse upper bound in terms of the growth rate for the sequential metric entropy.
\insertprethmspacing
\begin{proposition}
  \label{prop:entropy_growth}
  Suppose that $\cF$ has sequential metric entropy growth $\log\cN_{\infty,\infty}(\veps, \cF, T)\propto\veps^{-p}$ for some $p>0$ (nonparametric case), or that $\log\cN_{\infty,\infty}(\veps, \cF, T)\propto{}d\log(1/\veps)$ (parametric case). Then there exists a contextual bandit strategy with the following regret guarantee:
  \begin{equation}
    \label{eq:entropy}
    \En\brk*{\sum_{t=1}^{T}\ls_t(a_t)} 
    \leq{} \inf_{f\in\cF}\En\brk*{L_T^\gamma(f)} + \left\{ % \sum_{t=1}^{T}\sum_{a \in \Acal}\ls_t(a)\one\{f(x_t)_a \ge -\gamma\}} + \left\{
      \begin{array}{ll}
        O(K\sqrt{dT\log(KT/\gamma)}),&~~\text{parametric case.}\\
        \Otilde((KT)^{\frac{p+2}{p+4}}\gamma^{-\frac{2p}{p+4}}),&~~\text{nonparametric w/~~} p\leq{} 2.\\
        \Otilde((KT)^{\frac{p}{p+1}}\gamma^{-\frac{p}{p+1}}),&~~ \text{nonparametric w/~~} p\geq{}2.
              \end{array}
      \right.
  \end{equation}
\end{proposition}
\insertpostthmspacing
\pref{prop:entropy_growth} recovers the parametric rate of $\sqrt{dT}$ seen with e.g., \linucb~\citep{chu2011contextual} but is most interesting for complex classes. The rate exhibits a phase change between the ``moderate complexity'' regime of $p\in(0,2]$ and the ``high complexity'' regime of $p\geq{}2$. This is visualized in \pref{fig:exponent}. 

\begin{figure}[t]
\begin{minipage}{0.5\textwidth}
  \begin{center}
  \includegraphics[width=\textwidth]{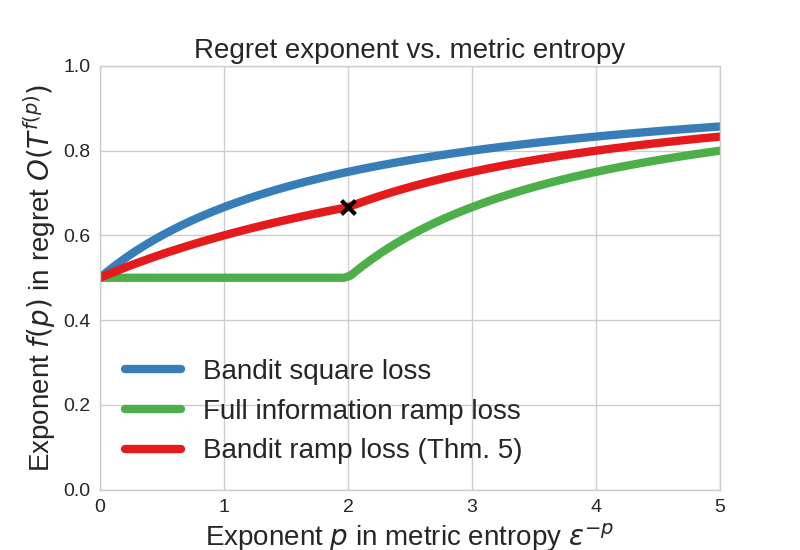}
\end{center}
\end{minipage}
\begin{minipage}{0.5\textwidth}
\caption{Regret bound exponent as a function of (sequential) metric entropy. The cross marks the point $p=2$ where the exponent from \pref{thm:chaining_ramp} changes growth rate. ``Full information'' refers to the optimal rate of $T^{\frac{1}{2}\vee{}(\frac{p-1}{p})}$ for the same setting under full information feedback \citep{RakSriTew14jmlr}. ``Square loss'' refers to the optimal rate of $T^{\frac{p+1}{p+2}}$ for Lipschitz contextual bandits over metric spaces of dimension $p$, which have sequential metric entropy $\veps^{-p}$, under square loss realizability \citep{slivkins2011contextual}.}
\label{fig:exponent}
\end{minipage}
\vspace{-0.5cm}
  \end{figure}
  
\insertprethmspacing
  \begin{remark}
    Under i.i.d. losses and hinge/ramp loss realizability, the standard tools of classification calibration \citep{bartlett2006convexity} can be used to deduce a proper policy regret bound from \pref{eq:chaining_ramp}. However, these realizability assumptions are somewhat non-standard, and moreover if one imposes the stronger assumption of a hard margin it is possible to derive improved rates~\citep{daniely2013price}. See also~\pref{app:comparison}.
  \end{remark}
\insertpostthmspacing
\insertprethmspacing
  \begin{remark}
Classical margin bounds typically hold for all
values of $\gamma$ simultaneously, but \pref{thm:chaining_ramp} requires that $\gamma$
is chosen in advance. Learning the best value of $\gamma$ online
appears challenging.
\end{remark}

\textbf{Rates for specific classes.}~~~~
We now instantiate our results for concrete classes of interest.
\insertprethmspacing
\begin{example}[Finite classes]
\label{ex:finite_cb}
In the finite class case there is an algorithm with
$O\prn*{K\sqrt{T\log\abs*{\cF}}}$ margin regret.  When
$\Pi \subset (\Xcal \to \Acal)$ is a finite policy class, directly 
reducing to \pref{thm:chaining_v3} yields the optimal
$O\prn*{\sqrt{KT\log\abs*{\Pi}}}$ policy regret, hinting at the
optimality of our approach.
\end{example}
\insertpostthmspacing
\insertprethmspacing
\begin{example}[Lipschitz CB]
\label{ex:lipschitz_cb}
  The class of Lipschitz functions over $\brk*{0,1}^{p}$ admits a sequential cover with metric entropy $\Otilde(\veps^{-p})$, so \pref{prop:entropy_growth} implies an $\otil(T^{\frac{p+2}{p+4}\vee\frac{p}{p+1}})$ regret bound.
Since our proof goes through~\pref{lem:calibration}, it also yields a \emph{policy regret} bound against the $\piramp(\cdot)$ class. 
Therefore, this result is directly comparable to the $\otil(T^{\frac{p+1}{p+2}})$ bound of~\citet{cesa2017algorithmic}, applied to the $\piramp$ policy class. Our bound achieves a smaller exponent for all values of $p$ (see \pref{fig:exponent}). 
\end{example}

Learnability in full information online learning is known to be characterized entirely by the \emph{sequential Rademacher complexity} of the hypothesis class~\citep{RakSriTew14jmlr}, and tight bounds on this quantity are known for standard classes including linear predictors, decision trees, and neural networks. The next example, a corollary of \pref{thm:chaining_ramp}, bounds contextual bandit margin regret in terms of sequential Rademacher complexity, which is defined for any scalar-valued function class $\cG\subseteq{}\prn*{\cX\to\bbR}$ as: %%, define the sequential Rademacher complexity via
\[
\mathfrak{R}(\cG,T) \triangleq \sup_{\bx}\En_{\eps}\sup_{g\in\cG}\sum_{t=1}^{T}\eps_tg(\bx_{t}(\eps)).
\]
\begin{example}
  \label{ex:rademacher}
  Let $\cF|_a\defeq\crl*{x\mapsto{}f(x)_a\mid{}f\in\cF}$ be the scalar restriction of $\cF$ to output coordinate $a$ and suppose that $\max_{a\in\brk*{K}}\mathfrak{R}(\cF|_{a},T)\geq{}1$ and $B\leq{}1$.\footnote{This restriction serves only to simplify calculations and can be relaxed.} Then there exists an adversarial contextual bandit algorithm with margin regret bound
  $\Otilde\prn*{\max_{a}K\prn*{\mathfrak{R}(\cF|_{a},T)/\gamma}^{2/3}T^{1/3}}$.
\end{example}
\vspace{-0.1cm}
Thus, for margin-based contextual bandits, full information
learnability is equivalent to bandit learnability. Since the optimal
regret in full information is
$\Omega(\max_{a}\mathfrak{R}(\cF|_a,T))$, it further shows that the
price of bandit information is at most
$\Otilde\prn*{\max_{a}K\prn*{T/\mathfrak{R}(\cF|_{a},T)}^{1/3}}$.
Note that while this bound is fairly user-friendly, it yields worse
rates than \pref{prop:entropy_growth} when translated to sequential
metric entropy, except when $p=2$~\citep{rakhlin2010online}.
%% Note
%% that while this bound is fairly user-friendly, 
%% plugging in the
%% sequential metric entropy upper bound on sequential Rademacher
%% complexity \citep{rakhlin2010online} yields worse rates than
%% \pref{prop:entropy_growth} except when $p=2$. 
For comparison,
\citet{RakSriTew14jmlr} obtain $\Otilde(\mathfrak{R}(\Fcal,T)/\gamma)$
margin regret in full information for binary classification. For
partial information, \bistro~\citep{rakhlin2016bistro} has an
$O(\sqrt{K T \mathfrak{R}(\Pi,T)})$ policy regret bound, which
involves the \emph{policy} complexity and a worse $T$ dependence than
our bound, but our bound (in terms of $\cF$) applies only to the
margin regret. A similar discussion applies to Theorem 4.4 of
\citet{lykouris2017small}.

%% , which is also policy-based and does not improve over \pref{prop:entropy_growth} in any parameter regime.

%% We now instantiate \pref{ex:rademacher} with linear classes. 
Instantiating \pref{ex:rademacher} with linear classes generalizes the $O(T^{2/3})$ dimension-independent guarantee of \banditron{}~\citep{kakade2008efficient} from Euclidean geometry to arbitrary uniformly convex Banach spaces, essentially the largest linear class for which online learning is possible~\citep{srebro2011universality}. The result also generalizes \banditron{} from multiclass to general contextual bandits and strengthens it from hinge loss to ramp loss. Note that many subsequent works~\citep{abernethy2009efficient, beygelzimerOZ17,foster2018logistic} obtain dimension-dependent $O(\sqrt{dT})$ bounds for bandit multiclass prediction, as we will in the next section, but, none have explored dimension-independent $O(T^{2/3})$-type rates, which are more appropriate for high-dimensional settings.

\insertprethmspacing
\begin{example}\label{ex:linear}
Let $\cX$ be the unit ball in a Banach space $(\mathfrak{B}, \nrm*{\cdot})$, and let $\cF$ be induced by stacking $K-1$
linear predictors\footnote{Only $K-1$ predictors are needed due to the sum-to-zero constraint of $\bbR^{K}_{=0}$.} each in the unit ball of the dual space $(\mathfrak{B}^{\star}, \nrm*{\cdot}_{\star})$. Suppose that $\nrm*{\cdot}$ has martingale type 2 \citep{Pisier75}, which means there exists $\Psi:\mathfrak{B}\to\bbR$ such that $\frac{1}{2}\nrm*{x}^{2}\leq{}\Psi(x)$ and $\Psi$ is $\beta$-smooth w.r.t. $\nrm*{\cdot}$.\footnote{Norms that satisfy this property with dimension-independent or logarithmic constants include $\ls_{p}$ for all $p\geq{}2$, Schatten $S_p$ norms for $p\geq{}2$ (including the spectral norm), and $(2,p)$ group norms for $p\geq{}2$ \citep{kakade2009complexity,kakade2012regularization}.} Then there exists a contextual bandit strategy with margin regret $O(K(T/\gamma)^{2/3})$. 
\end{example}

Beyond linear classes, we also obtain $\Otilde(K(T/\gamma)^{2/3})$
margin regret when each $\cF_a$ is a class of neural networks with
weights in each layer bounded in the $(1,\infty)$ group norm, or when each $\cF_a$ is a
class of bounded depth decision trees on finitely many decision
functions. These results follow by appealing to the existing sequential
Rademacher complexity bounds derived in~\citet{RakSriTew14jmlr}.

%% Applying \pref{ex:rademacher} and appealing to existing sequential Rademacher complexity bounds, there also exist contextual bandit strategies with margin regret $\Otilde(K(T/\gamma)^{2/3})$ whenever each $\cF|_a$ 

%% Appealing to existing sequential Rademacher complexity bounds, we derive regret bounds for a few more well-known function classes. In the interest of space, we state these bounds informally and refer the reader to~\citet{RakSriTew14jmlr} for quantitative bounds on the sequential Rademacher complexity.
%% \insertprethmspacing
%% \begin{example}
%% Suppose each $\cF|_{a}$ consists of a class of neural networks with weights in each layer bounded in the $(1,\infty)$ group norm, or consists of a class of bounded depth decision trees with a finite set of decision functions. Then there exists a strategy with margin regret $\Otilde(K(T/\gamma)^{2/3})$.
%% \end{example}
%% \insertpostthmspacing
As our last example, we consider $\ls_{p}$ spaces for $p<2$. These spaces fail to satisfy martingale type 2 in a dimension-independent fashion, but they do satisfy martingale type $p$ without dimension dependence, and so have sequential metric entropy of order $\veps^{-\frac{p}{p-1}}$ \citep{rakhlin2015equivalence}. Moreover, in $\bbR^{d}$ the $\ls_p$ spaces admit a pointwise cover with metric entropy $O(d\log(1/\veps))$, leading to the following dichotomy.
\insertprethmspacing
\begin{example}
Consider the setting of \pref{ex:linear}, with $(\mathfrak{B},\nrm*{\cdot})=(\bbR^{d}, \nrm*{\cdot}_{p})$ for $p\leq{}2$. Then there exists a contextual bandit strategy with margin regret $\Otilde(K(T/\gamma)^{\frac{p}{2p-1}}\wedge{}K\sqrt{dT\log(KT/\gamma)})$.
\end{example}
\insertpostthmspacing

%%% Local Variables:
%%% mode: latex
%%% TeX-master: "paper.tex"
%%% End:

\section{Efficient algorithms}
\label{sec:algorithms}
% !TEX root = arxiv.tex

We derive two new algorithms for contextual bandits
using the hinge loss $\hinge$.  The first algorithm, \lmc, focuses on
the parametric setting; it is based on a continuous version of
\expweights using a log-concave sampler.
%% , and is described in
%% \pref{ssec:hinge_lmc}. 
The second, \ftl, is simply Follow-The-Leader
with uniform smoothing. 
%% , described in \pref{ssec:ftl}. 
\ftl{} applies to the stochastic setting with classes that have ``high complexity'' in the sense of \pref{prop:entropy_growth}.

\subsection{Hinge-LMC}
\label{ssec:hinge_lmc}
For this section, we identify $\Fcal$ with a compact convex set
$\Theta\subset \RR^d$, using the notation $f(x;\theta) \in \RRz^K$ to
describe the parametrized function. We assume that
$\hinge(f(x;\theta)_a)$ is convex in $\theta$ for each $(x,a)$ pair,
$\sup_{x,\theta}\|f(x;\theta)\|_{\infty}\leq B$, $f(x;\cdot)_a$ is
$L$-Lipschitz in $\theta$ with respect to the $\ell_{2}$
norm, and that $\Theta$ contains the centered Euclidean ball of radius
$1$ and is contained within a Euclidean ball of radius $R$. These
assumptions are satisfied when $\cF$ is a linear class, under
appropriate boundedness conditions. 

\begin{figure*}
\begin{minipage}[t]{0.53\linewidth}
\begin{algorithm}[H]
\begin{algorithmic}
\State Input: Class $\Theta$, learning rate $\eta$, margin parameter $\gamma$. 
  \State Define $w_0(\theta) \defeq 1$ for all $\theta \in \Theta$.
  \For{$t = 1,\ldots,T$}
  \State Receive $x_t$, set $\theta_{t} \gets $ LMC($\eta w_{t-1}$).
  \State Set $p_t(\cdot;\theta_t) \propto \hinge(f(x_t;\theta_t))$.
  \State Set $p^\mu_t(\cdot;\theta_t) \defeq (1-K\mu)p_t + \mu$.
  \State Play $a_t \sim p_t^{\mu}(\cdot;\theta_t)$, observe $\ell_t(a_t)$. 
  \State {\textcolor{blue}{\small\texttt{// Geometric resampling.}}}
  \For{$m = 1,\ldots,M$}
  \State $\tilde{\theta}_t \gets $ LMC($\eta w_{t-1}$).
  \State Sample $\tilde{a}_t \sim p^\mu_t(\cdot;\tilde{\theta}_t)$, if $\tilde{a}_t = a_t$, break.
  \EndFor
  \State Set $m_{t}=m$, and 
$\tilde{\ell}_t(a) \defeq \ell_t(a_t)\cdot{}m_t\one\{a_t = a\}$.
  \State Update $w_{t}(\theta) \gets w_{t-1}(\theta) + \langle \tilde{\ell}_t, \hinge(f(x_t;\theta))\rangle$.
  \EndFor
\end{algorithmic}
\caption{\lmc}
\label{alg:hinge_lmc}
\end{algorithm}
\end{minipage}\hspace{0.005\textwidth}\begin{minipage}[t]{0.465\linewidth}
\begin{algorithm}[H]
\begin{algorithmic}
    \State \noindent Input: $F(\cdot)$, parameters
    $m,u,\lambda,N,\alpha$.
    \State {\textcolor{blue}{\small\texttt{// Parameter choices are in \pref{app:lmc}.}}}
    \State Set $\tilde{\theta}_0 \gets 0 \in \RR^d$
    \For{$k = 1,\ldots,N$}
    \State Draw $z_1,\ldots,z_m \iidsim \Ncal(0, u^2 I_{d})$. Define
    \begin{align*}
      \tilde{F}_k(\theta) \defeq \tfrac{1}{m}\textstyle\sum\nolimits_{i=1}^m F(\theta+z_i) + \frac{\lambda}{2} \|\theta\|_2^2
    \end{align*}
    \State Draw $\xi_k \sim \Ncal(0,I_d)$ and update
    \begin{align*}
      \tilde{\theta}_{k} \gets \Pcal_{\Theta}\left(\tilde{\theta}_{k-1} - \frac{\alpha}{2} \nabla \tilde{F}_k(\tilde{\theta}_{k-1}) + \sqrt{\alpha}\xi_k\right).
    \end{align*}
%%     \vspace{-0.08cm}
    \EndFor
    \State Return $\tilde{\theta}_N$. 
\end{algorithmic}
\caption{Langevin Monte Carlo (LMC)}
\label{alg:lmc}
\end{algorithm}
\end{minipage}
%% \vspace{-0.5cm}
\end{figure*}

The pseudocode for \lmc\xspace is displayed in
\pref{alg:hinge_lmc}, and all parameters settings are given
in~\pref{app:lmc}. The algorithm is a continuous variant of
exponential weights~\citep{auer2002nonstochastic}, where 
%% on the surrogate
%% hinge loss. 
at round $t$, we define the exponential weights
distribution via its density (w.r.t. the Lebesgue measure over
$\Theta$):
\begin{align*}
P_t(\theta) \propto \exp(-\eta w_{t-1}(\theta)), \qquad w_{t-1}(\theta) \defeq \sum_{s=1}^{t-1}\langle \tilde{\ell}_s, \hinge(f(x_s;\theta))\rangle,
\end{align*}

\noindent where $\eta$ is a learning rate and $\tilde{\ell}_s$ is a loss vector
estimate. At a high level, at each iteration the algorithm 
samples $\theta_t \sim P_t$, then samples the action $a_t$ from the
induced policy distribution $p_t(\cdot;\theta) =
\pihinge(f(x_t;\theta_t))$, appropriately smoothed.
%%  \propto \hinge(f(x_t;\theta_t)) \in
%% \Delta(\Acal)$. 
The algorithm plays $a_t$
%% , observes the loss
%% $\ell_t(a_t)$, 
and constructs a loss estimate $\tilde{\ell}_t \defeq
m_t\cdot\ell_t(a)\one\{a = a_t\}$, where $m_t$ is an approximate
importance weight computed by repeatedly sampling from $P_t$. This
vector $\tilde{\ell}_t$ is passed to exponential weights
to define the distribution at the next round. To
sample from $P_t$ we use Projected Langevin Monte Carlo
(LMC), displayed in \pref{alg:lmc}.

The algorithm has many important subtleties.  Apart from passing to
the hinge surrogate loss to obtain a tractable log-concave sampling
problem, by using the induced policy distribution $\pihinge(\cdot)$,
we are also able to control the local norm term in the exponential
weights regret bound.\footnote{This seems specialized to surrogates
that can be expressed as an inner product between the loss vector and
(a transformation of) the prediction, so it does not apply to standard
loss functions in bandit multiclass prediction.} Then, the analysis
for Projected LMC~\cite{bubeck2015sampling} requires a smooth
potential function, which we obtain by convolving with the gaussian
density, also known as randomized
smoothing~\citep{duchi2012randomized}.  We also use $\ls_2$
regularization for strong convexity and to overcome sampling errors
introduced by randomized smoothing.  Finally,
%% since we lack direct access to the sampling distribution, 
we use the
geometric resampling technique~\citep{neu2013efficient} to approximate
the importance weight by repeated sampling. 
%% Throughout it is
%% important to show that the large scaling on the loss estimates induced
%% by importance weighting does not degrade computational
%% performance. 
%% All of the components are analyzed in detail
%% in~\pref{app:lmc}.

%% The algorithm has many important subtleties.  Briefly, the analysis
%% for Projected LMC that we use, due to~\citet{bubeck2015sampling},
%% requires a smooth potential function, and we use the randomized
%% technique of~\citet{duchi2012randomized} to smooth the hinge loss by
%% convolving with a gaussian density (in expectation). Then, since the
%% gradients of this smooth function cannot be computed in closed form,
%% we use a coupling argument to show that the iterates on a sampled
%% approximation track the ideal iterates. Here, the $\ls_2$
%% regularization that we add plays an important role.  Finally, since we lack direct
%% access to the sampling distribution, we use the geometric resampling
%% technique~\citep{neu2013efficient} to approximate the importance
%% weight by repeated sampling. At all stages it is important to show
%% that the large scaling on the loss estimates induced by importance
%% weighting does not degrade computational performance. All of the
%% components are analyzed in detail in~\pref{app:lmc}.

Here, we state the main guarantee and its consequences. A more
complete theorem statement, with exact parameter specifications and the
precise running time is provided in~\pref{app:lmc} as \pref{thm:proj_lmc}.
\insertprethmspacing
\begin{theorem}[Informal]
  \label{thm:hinge_lmc}
  Under the assumptions of \pref{ssec:hinge_lmc}, \lmc{} with appropriate parameter settings runs in time $\poly(T,d,B,K,\tfrac{1}{\gamma},R,L)$ and guarantees
  \begin{align*}
    \EE \sum_{t=1}^T \ell_t(a_t) \leq \inf_{\theta \in \Theta} \frac{1}{K}\EE \sum_{t=1}^T\langle \ell_t, \hinge(f(x_t;\theta))\rangle + \otil\left(\frac{B}{\gamma}\sqrt{dT}\right).
  \end{align*}
  
\end{theorem}
\insertpostthmspacing
Since bandit multiclass prediction is a special case of contextual
bandits,~\pref{thm:hinge_lmc} immediately implies a
$\sqrt{dT}$-mistake bound for this setting. See~\pref{app:comparison} for more discussion.
\insertprethmspacing
\begin{corollary}[Bandit multiclass]
  \label{corr:bandit_multiclass}
  In the bandit multiclass setting,~\pref{alg:hinge_lmc}
  enjoys a mistake bound of $\otil((B/\gamma)\sqrt{dT})$ against
  the cost-sensitive $\gamma$-hinge loss and runs in polynomial time.
\end{corollary}

Additionally, under a realizability condition for the hinge loss, we
obtain a standard regret bound. For simplicity in defining the
condition, assume that for every $(x,\ell)$ pair, $\ell$ is a random
variable with conditional mean $\bar{\ell}$ (chosen by the adversary)
and $\bar{\ell}$ has a unique action with minimal loss.
\insertprethmspacing
\begin{corollary}[Realizable bound]
  \label{corr:realizable}
  In addition to the conditions above, assume that there exists
  $\theta^\star \in \Theta$ such that for every $(x,\ell)$ pair and for all $a \in \cA$, we have 
  $f(x;\theta^\star)_a \defeq K\gamma\one\{\bar{\ell}(a) \leq
  \min_{a'}\bar{\ell}(a')\} - \gamma$. Then \lmc\xspace runs in polynomial time and guarantees
  \begin{align*}
    \sum_{t=1}^T\EE \bar{\ell}_t(a_t) \leq \sum_{t=1}^T \EE \min_a \bar{\ell}(a) + \otil\prn*{\frac{B}{\gamma}\sqrt{dT}}.
  \end{align*}
\end{corollary}
\insertpostthmspacing
A few comments are in order: 
\begin{packed_enum}
%% \item On a technical level, 
%% For
%% this step, it is crucial that we sample from the induced policy
%% distributions $\pihinge(\cdot)$ rather than the more natural argmax
%% policy $\argmax_a f(x;\theta)_a$, which does not provide suitable
%% control. Our technique therefore seems specialized to surrogates that can
%% be expressed as an inner product between the loss vector and (a
%% transformation of) the prediction, which cannot be done for many loss functions used in bandit multiclass prediction.
\item The use of LMC for sampling is not strictly necessary. Other
  log-concave samplers do exist for non-smooth
  potentials~\citep{lovasz2007geometry}, which will remove the
  parameters $m,u,\lambda$, significantly simplify the algorithm, and
  even lead to a better run-time guarantee using current
  theory. However, we prefer to use LMC due to its success in Bayesian
  inference and deep learning, and its connections to incremental
  optimization methods. 
%% LMC, moreso than say Hit-and-Run
%%   \citep{lovasz2007geometry}, can easily be adapted to work quickly
%%   (in practice) when data arrives online.
%% Furthermore, while the
%%   runtime for LMC is quite large in theory (\pref{app:lmc}), the
%%   theoretical memory usage scales only linearly with the memory
%%   required to store a single context.  
  Note that more recent results in slightly different
  settings~\citep{raginsky2017non,dalalyan2017user,cheng2018sharp}
  suggest that it may be possible to substantially improve upon the
  LMC analysis that we use and even extend it to non-convex settings.  We
  are hopeful that the LMC approach will lead to a practically useful
  contextual bandit algorithm and plan to explore this direction
  further.
%% \item As mentioned above, while the algorithm is guaranteed to run in
%%   polynomial time, the dependence on $T$ and $d$ that we obtain is
%%   quite poor. In part, this inefficiency stems from the mixing-time
%%   analysis for Projected LMC~\citep{bubeck2015sampling}. More recent
%%   results in slightly different
%%   settings~\citep{raginsky2017non,dalalyan2017user,cheng2018sharp}
%%   suggest that it may be possible to substantially improve this
%%   analysis and even extend to non-convex settings. Similarly, we
%%   conjecture that Projected LMC can be analyzed without smoothness.
\item \pref{corr:bandit_multiclass} provides a new solution to the
  open problem of~\citet{abernethy2009efficient}. In fact, it is the
  first efficient $\sqrt{dT}$-type regret bound against a hinge loss
  benchmark, although our loss is slightly different from the
  multiclass hinge loss used by \citet{kakade2008efficient} in their
  $T^{2/3}$-regret \banditron{} algorithm (which motivated the open
  problem). All prior $\sqrt{dT}$-regret
  algorithms~\citep{hazan2011newtron,beygelzimerOZ17,foster2018logistic}
  use losses with curvature such as the multiclass logistic loss or
  the squared hinge loss. See~\pref{app:comparison} for a comparison
  between cost-sensitive and multiclass hinge losses.

\item 
  In~\pref{corr:realizable}, regret is measured relative to
  the policy that chooses the best action (in expectation) on
  \emph{every round}. As in prior
  results~\citep{abbasi2011improved,agarwal2012contextual}, this is
  possible because the realizability condition ensures that this
  policy is in our class. Note that here, a requirement for
  realizability is that $B \ge K\gamma$, and hence the dependence on
  $K$ is implicit and in fact slightly worse than the optimal rate \citep{chu2011contextual}.
\item For~\pref{corr:realizable}, the best points of comparisons are
  methods based on square-loss
  realizability~\citep{agarwal2012contextual,foster2018practical},
  although our condition is different. Compared with \linucb and
  variants~\citep{chu2011contextual,abbasi2011improved} specialized to
  $\ell_2/\ell_2$ geometry, our assumptions are somewhat weaker but
  these methods have slightly better guarantees for linear
  classes.\footnote{In the abstract linear setting we take $\cF$ to be
    the set of linear functions in the ball for some norm
    $\nrm*{\cdot}$ and contexts to be bounded in the dual norm
    $\nrm*{\cdot}_{\star}$. The runtime of \lmc\xspace will degrade
    (polynomially) with the ratio $\nrm*{\theta}/\nrm*{\theta}_2$, but
    the regret bound is the same for any such norm pair.}
%% Our assumptions are somewhat weaker than for \linucb and
%% variants~\citep{chu2011contextual,abbasi2011improved} that are
%% specialized to the $\ell_2/\ell_2$ geometry, but these methods have
%% slightly better guarantees for linear classes (again under square loss
%% realizability).
  Compared with~\citet{foster2018practical}, which is the only other
  efficient approach at a comparable level of generality, our
  assumptions on the regressor class are stronger, but we obtain
  better guarantees, in particular removing distribution-dependent parameters.
%% Note that the method
%%   of~\citet{foster2018practical} was recently shown to be one of best
%%   contextual bandit algorithms in a comprehensive empirical
%%   study~\citep{bietti2018contextual}, and since \lmc\xspace has stronger guarantees, we hope this empirical a property we hope will carry
%%   over to \lmc.
%% We impose stronger assumptions
%%   on the regressor class but obtain better
%%   guarantees~\citet{foster2018practical}, which is the only other
%%   efficient approach at a comparable level of generality.  
\end{packed_enum}

To summarize, \lmc\xspace is the first efficient $\sqrt{dT}$-regret
algorithm for bandit multiclass prediction using the hinge loss. It
also represents a new approach to adversarial contextual bandits,
yielding $\sqrt{dT}$ policy regret under hinge-based
realizability. Finally, while we lose the theoretical guarantees, the
algorithm easily extends to non-convex classes, which we expect to be
practically effective.

\subsection{\ftl}
\label{ssec:ftl}
A drawback of \lmc{} is that it only applies in the parametric
regime. We now introduce an efficient (in terms of queries to a hinge
loss minimization oracle) algorithm with a regret bound similar
to \pref{thm:chaining_ramp}, but in the stochastic setting, where $\crl*{(x_t,\ls_t)}_{t=1}^{T}$ are drawn i.i.d. from some joint
distribution $\cD$ over $\cX\times{}\bbR_{+}^{K}$.
%% . Precisely,
%% we work in the same model as \pref{ssec:prelims}, but where 
%% $\crl*{(x_t,\ls_t)}_{t=1}^{T}$ are drawn i.i.d. from some joint
%% distribution $\cD$ over $\cX\times{}\bbR_{+}^{K}$. 
Here we return to the abstract setting
with regression class $\Fcal$, and for simplicity, we
assume $B=1$.

The algorithm we analyze is simply Follow-The-Leader with uniform
smoothing and epoching, which we refer to as \ftl. 
We use an epoch schedule where the $m^{\textrm{th}}$ epoch lasts for $n_m \defeq 2^m$ rounds (starting with $m=0$). At the beginning of the $m^{\textrm{th}}$ epoch, we compute the
empirical importance weighted hinge-loss minimizer $\hat{f}_{m-1}$ using \emph{only} the data from the previous epoch. 
That is, we set
\begin{align*}
\hat{f}_{m-1} \defeq \argmin_{f \in \Fcal}\sum_{\tau=n_{m-1}}^{n_{m}-1} \langle \hat{\ell}_\tau, \hinge(f(x_\tau))\rangle.
\end{align*}
Then, for each round $t$ in the $m^{\textrm{th}}$ epoch, we sample $a_t$ from $p_t \defeq (1 - K\mu)\pihinge(\hat{f}_{m-1}(x_t)) +
\mu$.
The parameter $\mu\in(0,1/K]$ controls the smoothing. At time $t=1$ we simply take $p_1$ to be uniform.

\insertprethmspacing
\begin{theorem}%%[\ftl\xspace regret bound]
\label{thm:ftl}
Suppose that $\Fcal$ satisfies $\log
\Ncal_{\infty,\infty}(\veps,\Fcal,T) \propto \veps^{-p}$ for some $p
> 2$. Then in the stochastic setting, %with $\mu = \gamma^{\frac{-p}{p+1}}(KT)^{\frac{-1}{p+1}}$,
with $\mu = K^{-1}T^{\frac{-1}{p+1}}$,
\ftl\xspace enjoys the following expected regret guarantee\footnote{This result is stated in terms of the sequential cover $\Ncal_{\infty,\infty}$ to avoid additional definitions, but can easily be improved to depend on the classical (worst-case) covering number seen in statistical learning.}
\begin{align*}
  \sum_{t=1}^T \EE \ell_t(a_t) \leq \inf_{f \in \Fcal} \frac{T}{K} \EE \langle \ell, \hinge(f(x))\rangle + \Otilde\prn*{(T/\gamma)^{\frac{p}{p+1}}}.
  %\otil\left((KT)^{\frac{p}{p+1}}/\gamma\right).
\end{align*}
\end{theorem}
\insertpostthmspacing
This provides an algorithmic counterpart to~\pref{prop:entropy_growth} in the $p \ge 2$ regime. The algorithm is quite similar to \textsc{Epoch-Greedy} \citep{langford2008epoch}, and the main contribution here is to provide a careful analysis for large function classes. We leave obtaining an oracle-efficient algorithm that matches \pref{prop:entropy_growth} in the regime $p\in(0,2)$ as an open problem. 

A similar bound can be obtained for the ramp loss by simply replacing the hinge loss ERM.
%%  with that for the ramp loss. 
We analyze the hinge loss version because standard (e.g. linear) classes admit efficient hinge loss minimization oracles. Interestingly, the bound in \pref{thm:ftl} actually improves on \pref{prop:entropy_growth}, in that it is independent of $K$. This is due to the scaling of the hinge loss in \pref{lem:calibration}.

In~\pref{app:lipschitz}, we extend the analysis to the stochastic
Lipschitz contextual bandit setting. Here, instead of measuring regret
against the benchmark $\hinge\circ\cF$ we compare to the class of all
$1$-Lipschitz functions from $\cX$ to $\Delta(\cA$), where $\cX$ is a
metric space of bounded covering dimension. We show that \ftl\xspace
achieves $T^{\frac{p}{p+1}}$ regret with a 
%% against Lipschitz policies over a
$p$-dimensional context space and finite action space. This improves
on the $T^{\frac{p+1}{p+2}}$ bound of~\citet{cesa2017algorithmic}, as
in \pref{ex:lipschitz_cb}, yet the best available lower bound is
$T^{\frac{p-1}{p}}$~\citep{hazan2007online}.  Closing this gap remains
an intriguing open problem.

%%% Local Variables:
%%% mode: latex
%%% TeX-master: "paper.tex"
%%% End:

\section{Discussion}
This paper initiates a study of the utility of surrogate losses in
contextual bandit learning. We obtain new margin-based regret bounds
in terms of sequential complexity notions on the benchmark class,
improving on the best known rates for Lipschitz contextual bandits and
providing dimension-independent bounds for linear classes. On the
algorithmic side, we 
% design a new algorithm for adversarial contextual
% bandits based on exponential weights and log-concave sampling, which
provide the first solution to the open problem
of~\citet{abernethy2009efficient} with a non-curved loss and we also
show that Follow-the-Leader with uniform smoothing performs
well in nonparametric settings.

Yet, several open problems remain. First, our bounds in
\pref{sec:minimax} are likely suboptimal in the dependence on $K$, and
improving this is a natural direction. Other questions involve
deriving stronger lower bounds (e.g., for the non-parametric setting)
and adapting to the margin parameter. We also hope to experiment with
\lmc, and develop a better understanding of computational-statistical
tradeoffs with surrogate losses. We look forward to studying these
questions in future work.

\textbf{Acknowledgements.}~~~~ We thank Haipeng Luo, Karthik Sridharan, Chen-Yu Wei, and Chicheng Zhang for several helpful discussions. D.F. acknowledges the support of the NDSEG PhD fellowship and Facebook PhD fellowship.

%% \vfill
%% \newpage

\bibliography{refs}

\begin{thebibliography}{49}
\providecommand{\natexlab}[1]{#1}
\providecommand{\url}[1]{\texttt{#1}}
\expandafter\ifx\csname urlstyle\endcsname\relax
  \providecommand{\doi}[1]{doi: #1}\else
  \providecommand{\doi}{doi: \begingroup \urlstyle{rm}\Url}\fi

\bibitem[Abbasi-Yadkori et~al.(2011)Abbasi-Yadkori, P{\'a}l, and
  Szepesv{\'a}ri]{abbasi2011improved}
Yasin Abbasi-Yadkori, D{\'a}vid P{\'a}l, and Csaba Szepesv{\'a}ri.
\newblock Improved algorithms for linear stochastic bandits.
\newblock In \emph{Advances in Neural Information Processing Systems}, 2011.

\bibitem[Abernethy and Rakhlin(2009)]{abernethy2009efficient}
Jacob~D. Abernethy and Alexander Rakhlin.
\newblock An efficient bandit algorithm for ${O}(\sqrt{T})$-regret in online
  multiclass prediction?
\newblock In \emph{Conference on Learning Theory}, 2009.

\bibitem[Agarwal et~al.(2012)Agarwal, Dud{\'\i}k, Kale, Langford, and
  Schapire]{agarwal2012contextual}
Alekh Agarwal, Miroslav Dud{\'\i}k, Satyen Kale, John Langford, and Robert~E.
  Schapire.
\newblock Contextual bandit learning with predictable rewards.
\newblock In \emph{Artificial Intelligence and Statistics}, 2012.

\bibitem[Agarwal et~al.(2014)Agarwal, Hsu, Kale, Langford, Li, and
  Schapire]{agarwal2014taming}
Alekh Agarwal, Daniel~J. Hsu, Satyen Kale, John Langford, Lihong Li, and
  Robert~E. Schapire.
\newblock Taming the monster: A fast and simple algorithm for contextual
  bandits.
\newblock In \emph{International Conference on Machine Learning}, 2014.

\bibitem[Agarwal et~al.(2016)Agarwal, Bird, Cozowicz, Hoang, Langford, Lee, Li,
  Melamed, Oshri, Ribas, Sen, and Slivkins]{agarwal2016making}
Alekh Agarwal, Sarah Bird, Markus Cozowicz, Luong Hoang, John Langford, Stephen
  Lee, Jiaji Li, Dan Melamed, Gal Oshri, Oswaldo Ribas, Siddhartha Sen, and
  Aleksandrs Slivkins.
\newblock Making contextual decisions with low technical debt.
\newblock \emph{arXiv:1606.03966}, 2016.

\bibitem[Anthony and Bartlett(2009)]{anthony2009neural}
Martin Anthony and Peter~L. Bartlett.
\newblock \emph{Neural network learning: Theoretical foundations}.
\newblock Cambridge University Press, 2009.

\bibitem[Auer et~al.(2002)Auer, Cesa-Bianchi, Freund, and
  Schapire]{auer2002nonstochastic}
Peter Auer, Nicol{\`o} Cesa-Bianchi, Yoav Freund, and Robert~E. Schapire.
\newblock The nonstochastic multiarmed bandit problem.
\newblock \emph{SIAM Journal on Computing}, 2002.

\bibitem[Bartlett et~al.(2006)Bartlett, Jordan, and
  McAuliffe]{bartlett2006convexity}
Peter~L. Bartlett, Michael~I. Jordan, and Jon~D. McAuliffe.
\newblock Convexity, classification, and risk bounds.
\newblock \emph{Journal of the American Statistical Association}, 2006.

\bibitem[Beygelzimer et~al.(2017)Beygelzimer, Orabona, and
  Zhang]{beygelzimerOZ17}
Alina Beygelzimer, Francesco Orabona, and Chicheng Zhang.
\newblock Efficient online bandit multiclass learning with
  $\tilde{{O}}(\sqrt{{T}})$ regret.
\newblock In \emph{International Conference on Machine Learning}, 2017.

\bibitem[Boucheron et~al.(2005)Boucheron, Bousquet, and
  Lugosi]{boucheron2005theory}
St{\'e}phane Boucheron, Olivier Bousquet, and G{\'a}bor Lugosi.
\newblock Theory of classification: A survey of some recent advances.
\newblock \emph{ESAIM: Probability and Statistics}, 2005.

\bibitem[Boucheron et~al.(2013)Boucheron, Lugosi, and
  Massart]{boucheron2013concentration}
St{\'e}phane Boucheron, G{\'a}bor Lugosi, and Pascal Massart.
\newblock \emph{Concentration inequalities: A nonasymptotic theory of
  independence}.
\newblock Oxford University Press, 2013.

\bibitem[Bubeck et~al.(2018)Bubeck, Eldan, and Lehec]{bubeck2015sampling}
S{\'e}bastien Bubeck, Ronen Eldan, and Joseph Lehec.
\newblock Sampling from a log-concave distribution with projected langevin
  monte carlo.
\newblock \emph{Discrete and Computational Geometry}, 2018.

\bibitem[Cesa-Bianchi and Lugosi(2006)]{PLG}
Nicol{\`o} Cesa-Bianchi and Gabor Lugosi.
\newblock \emph{Prediction, Learning, and Games}.
\newblock Cambridge University Press, 2006.

\bibitem[Cesa-Bianchi et~al.(2017)Cesa-Bianchi, Gaillard, Gentile, and
  Gerchinovitz]{cesa2017algorithmic}
Nicol{\`o} Cesa-Bianchi, Pierre Gaillard, Claudio Gentile, and S{\'e}bastien
  Gerchinovitz.
\newblock Algorithmic chaining and the role of partial feedback in online
  nonparametric learning.
\newblock In \emph{Conference on Learning Theory}, 2017.

\bibitem[Cheng et~al.(2018)Cheng, Chatterji, Abbasi-Yadkori, Bartlett, and
  Jordan]{cheng2018sharp}
Xiang Cheng, Niladri~S. Chatterji, Yasin Abbasi-Yadkori, Peter~L. Bartlett, and
  Michael~I. Jordan.
\newblock Sharp convergence rates for langevin dynamics in the nonconvex
  setting.
\newblock \emph{arXiv:1805.01648}, 2018.

\bibitem[Chu et~al.(2011)Chu, Li, Reyzin, and Schapire]{chu2011contextual}
Wei Chu, Lihong Li, Lev Reyzin, and Robert~E. Schapire.
\newblock Contextual bandits with linear payoff functions.
\newblock In \emph{International Conference on Artificial Intelligence and
  Statistics}, 2011.

\bibitem[Dalalyan and Karagulyan(2017)]{dalalyan2017user}
Arnak~S. Dalalyan and Avetik~G. Karagulyan.
\newblock User-friendly guarantees for the langevin monte carlo with inaccurate
  gradient.
\newblock \emph{arXiv:1710.00095}, 2017.

\bibitem[Daniely and Halbertal(2013)]{daniely2013price}
Amit Daniely and Tom Halbertal.
\newblock The price of bandit information in multiclass online classification.
\newblock In \emph{Conference on Learning Theory}, 2013.

\bibitem[Duchi et~al.(2012)Duchi, Bartlett, and
  Wainwright]{duchi2012randomized}
John~C. Duchi, Peter~L. Bartlett, and Martin~J. Wainwright.
\newblock Randomized smoothing for stochastic optimization.
\newblock \emph{SIAM Journal on Optimization}, 2012.

\bibitem[Foster et~al.(2015)Foster, Rakhlin, and Sridharan]{foster2015adaptive}
Dylan~J. Foster, Alexander Rakhlin, and Karthik Sridharan.
\newblock Adaptive online learning.
\newblock In \emph{Advances in Neural Information Processing Systems}, 2015.

\bibitem[Foster et~al.(2018{\natexlab{a}})Foster, Agarwal, Dud{\'\i}k, Luo, and
  Schapire]{foster2018practical}
Dylan~J. Foster, Alekh Agarwal, Miroslav Dud{\'\i}k, Haipeng Luo, and Robert~E.
  Schapire.
\newblock Practical contextual bandits with regression oracles.
\newblock \emph{International Conference on Machine Learning},
  2018{\natexlab{a}}.

\bibitem[Foster et~al.(2018{\natexlab{b}})Foster, Kale, Luo, Mohri, and
  Sridharan]{foster2018logistic}
Dylan~J. Foster, Satyen Kale, Haipeng Luo, Mehryar Mohri, and Karthik
  Sridharan.
\newblock Logistic regression: The importance of being improper.
\newblock \emph{Conference on Learning Theory}, 2018{\natexlab{b}}.

\bibitem[Freund and Schapire(1997)]{freund1997decision}
Yoav Freund and Robert~E. Schapire.
\newblock A decision-theoretic generalization of on-line learning and an
  application to boosting.
\newblock \emph{Journal of Computer and System Sciences}, 1997.

\bibitem[Hazan and Kale(2011)]{hazan2011newtron}
Elad Hazan and Satyen Kale.
\newblock Newtron: an efficient bandit algorithm for online multiclass
  prediction.
\newblock In \emph{Advances in Neural Information Processing Systems}, 2011.

\bibitem[Hazan and Megiddo(2007)]{hazan2007online}
Elad Hazan and Nimrod Megiddo.
\newblock Online learning with prior knowledge.
\newblock In \emph{Conference on Learning Theory}, 2007.

\bibitem[Kakade et~al.(2008)Kakade, Shalev-Shwartz, and
  Tewari]{kakade2008efficient}
Sham~M. Kakade, Shai Shalev-Shwartz, and Ambuj Tewari.
\newblock Efficient bandit algorithms for online multiclass prediction.
\newblock In \emph{International Conference on Machine learning}, 2008.

\bibitem[Kakade et~al.(2009)Kakade, Sridharan, and
  Tewari]{kakade2009complexity}
Sham~M. Kakade, Karthik Sridharan, and Ambuj Tewari.
\newblock On the complexity of linear prediction: Risk bounds, margin bounds,
  and regularization.
\newblock In \emph{Advances in Neural Information Processing Systems}, 2009.

\bibitem[Kakade et~al.(2012)Kakade, Shalev-Shwartz, and
  Tewari]{kakade2012regularization}
Sham~M. Kakade, Shai Shalev-Shwartz, and Ambuj Tewari.
\newblock Regularization techniques for learning with matrices.
\newblock \emph{Journal of Machine Learning Research}, 2012.

\bibitem[Langford and Zhang(2008)]{langford2008epoch}
John Langford and Tong Zhang.
\newblock The epoch-greedy algorithm for multi-armed bandits with side
  information.
\newblock In \emph{Advances in Neural Information Processing Systems}, 2008.

\bibitem[Lov{\'a}sz and Vempala(2007)]{lovasz2007geometry}
L{\'a}szl{\'o} Lov{\'a}sz and Santosh Vempala.
\newblock The geometry of logconcave functions and sampling algorithms.
\newblock \emph{Random Structures \& Algorithms}, 2007.

\bibitem[Lykouris et~al.(2018)Lykouris, Sridharan, and
  Tardos]{lykouris2017small}
Thodoris Lykouris, Karthik Sridharan, and {\'{E}}va Tardos.
\newblock Small-loss bounds for online learning with partial information.
\newblock \emph{Conference on Learning Theory}, 2018.

\bibitem[Narayanan and Rakhlin(2017)]{narayanan2017efficient}
Hariharan Narayanan and Alexander Rakhlin.
\newblock Efficient sampling from time-varying log-concave distributions.
\newblock \emph{Journal of Machine Learning Research}, 2017.

\bibitem[Neu and Bart{\'o}k(2013)]{neu2013efficient}
Gergely Neu and G{\'a}bor Bart{\'o}k.
\newblock An efficient algorithm for learning with semi-bandit feedback.
\newblock In \emph{International Conference on Algorithmic Learning Theory},
  2013.

\bibitem[Pires et~al.(2013)Pires, Szepesvari, and Ghavamzadeh]{pires2013cost}
Bernardo~\'{A}vila Pires, Csaba Szepesvari, and Mohammad Ghavamzadeh.
\newblock Cost-sensitive multiclass classification risk bounds.
\newblock In \emph{International Conference on Machine Learning}, 2013.

\bibitem[Pisier(1975)]{Pisier75}
Gilles Pisier.
\newblock Martingales with values in uniformly convex spaces.
\newblock \emph{Israel Journal of Mathematics}, 1975.

\bibitem[Raginsky et~al.(2017)Raginsky, Rakhlin, and
  Telgarsky]{raginsky2017non}
Maxim Raginsky, Alexander Rakhlin, and Matus Telgarsky.
\newblock Non-convex learning via stochastic gradient langevin dynamics: a
  nonasymptotic analysis.
\newblock In \emph{Conference on Learning Theory}, 2017.

\bibitem[Rakhlin and Sridharan(2015)]{RakSri15a}
Alexander Rakhlin and Karthik Sridharan.
\newblock Online nonparametric regression with general loss functions.
\newblock \emph{arxiv:1501.06598}, 2015.

\bibitem[Rakhlin and Sridharan(2016)]{rakhlin2016bistro}
Alexander Rakhlin and Karthik Sridharan.
\newblock {BISTRO}: An efficient relaxation-based method for contextual
  bandits.
\newblock In \emph{International Conference on Machine Learning}, 2016.

\bibitem[Rakhlin and Sridharan(2017)]{rakhlin2015equivalence}
Alexander Rakhlin and Karthik Sridharan.
\newblock On equivalence of martingale tail bounds and deterministic regret
  inequalities.
\newblock \emph{Conference on Learning Theory}, 2017.

\bibitem[Rakhlin et~al.(2010)Rakhlin, Sridharan, and Tewari]{rakhlin2010online}
Alexander Rakhlin, Karthik Sridharan, and Ambuj Tewari.
\newblock Online learning: Random averages, combinatorial parameters, and
  learnability.
\newblock \emph{Advances in Neural Information Processing Systems}, 2010.

\bibitem[Rakhlin et~al.(2015{\natexlab{a}})Rakhlin, Sridharan, and
  Tewari]{RakSriTew14jmlr}
Alexander Rakhlin, Karthik Sridharan, and Ambuj Tewari.
\newblock Online learning via sequential complexities.
\newblock \emph{Journal of Machine Learning Research}, 2015{\natexlab{a}}.

\bibitem[Rakhlin et~al.(2015{\natexlab{b}})Rakhlin, Sridharan, and
  Tewari]{rakhlin2015sequential}
Alexander Rakhlin, Karthik Sridharan, and Ambuj Tewari.
\newblock Sequential complexities and uniform martingale laws of large numbers.
\newblock \emph{Probability Theory and Related Fields}, 2015{\natexlab{b}}.

\bibitem[Schapire and Freund(2012)]{schapire2012boosting}
Robert~E. Schapire and Yoav Freund.
\newblock \emph{Boosting: Foundations and algorithms}.
\newblock MIT press, 2012.

\bibitem[Slivkins(2011)]{slivkins2011contextual}
Aleksandrs Slivkins.
\newblock Contextual bandits with similarity information.
\newblock In \emph{Conference on Learning Theory}, 2011.

\bibitem[Srebro et~al.(2011)Srebro, Sridharan, and
  Tewari]{srebro2011universality}
Nathan Srebro, Karthik Sridharan, and Ambuj Tewari.
\newblock On the universality of online mirror descent.
\newblock In \emph{Advances in Neural Information Processing Systems}, 2011.

\bibitem[Syrgkanis et~al.(2016{\natexlab{a}})Syrgkanis, Krishnamurthy, and
  Schapire]{syrgkanis2016efficient}
Vasilis Syrgkanis, Akshay Krishnamurthy, and Robert~E. Schapire.
\newblock Efficient algorithms for adversarial contextual learning.
\newblock In \emph{International Conference on Machine Learning},
  2016{\natexlab{a}}.

\bibitem[Syrgkanis et~al.(2016{\natexlab{b}})Syrgkanis, Luo, Krishnamurthy, and
  Schapire]{syrgkanis2016improved}
Vasilis Syrgkanis, Haipeng Luo, Akshay Krishnamurthy, and Robert~E Schapire.
\newblock Improved regret bounds for oracle-based adversarial contextual
  bandits.
\newblock In \emph{Advances in Neural Information Processing Systems},
  2016{\natexlab{b}}.

\bibitem[Tewari and Murphy(2017)]{tewari2017ads}
Ambuj Tewari and Susan~A. Murphy.
\newblock From ads to interventions: Contextual bandits in mobile health.
\newblock In \emph{Mobile Health}, 2017.

\bibitem[Zhang(2004)]{zhang2004statistical}
Tong Zhang.
\newblock Statistical analysis of some multi-category large margin
  classification methods.
\newblock \emph{Journal of Machine Learning Research}, 2004.

\end{thebibliography}

\appendix
\section{Calibration lemmas}
\label{app:calibration}
% !TEX root = arxiv.tex

\begin{proof}[Proof of~\pref{lem:calibration}]
We start with the ramp loss. First since $s \in \RRz^K$, we know that
the normalization term in $\piramp(s)$ is
\begin{align*}
\sum_{a \in \Acal} \ramp(s_a) \ge 1,
\end{align*}
from which the first inequality follows. The second inequality follows
from the fact that $s_a \le -\gamma$ implies that $\piramp(s)_a = 0$,
along with the trivial fact that $\piramp(s)_a \leq 1$.

The hinge loss claim is also straightforward, since here the
normalization is
\begin{align*}
\sum_{a \in \Acal} \hinge(s_a) = \sum_{a \in \Acal} \max\{1+s_a/\gamma,0\} \ge \sum_a 1+\frac{s_a}{\gamma} \ge K\tag*\qedhere.
\end{align*}
\end{proof}

\begin{lemma}[Hinge loss realizability]
\label{lem:hinge_realizable}
Let $\ell \in \RR^K_+$ and let $a^\star = \argmin_{a \in \Acal}
\ell_a$. Define $s \in \RRz^K$ via $s_a \triangleq K\gamma\one\{a = a^\star\}
- \gamma$. Then we have
\begin{align*}
\langle \ell, \hinge(s)\rangle = K \langle \ell, \pihinge(s)\rangle = K \ell_{a^\star}.
\end{align*}
\end{lemma}
\begin{proof}
For this particular $s$, the normalizing constant in the definition of $\pihinge$ is
\begin{align*}
\sum_{a \in \Acal} \max\left( 1 + \frac{K\gamma\one\{a = a^\star\} - \gamma}{\gamma}, 0\right) = K,
\end{align*}
and so the first equality follows. The second equality is also
straightforward since the score for every action except $a^\star$ is
clamped to zero.
\end{proof}

\begin{proof}[\pfref{lem:ips_variance}]~\\
  For the case when $\cS\subset\Delta(\Acal)$, this claim is a well-known property of importance weighting:
  \begin{align*}
    \En\brk*{\En_{s_t\sim{}p_t} \tri*{s_t,\hat{\ls}_t}^2\mid\cJ_t} &= \sum_{a \in \brk*{K}} P_t^\mu(a) \frac{\En_{s_t\sim{}p_t} \ls_t^2(a)s_t^2(a)}{(P_t^\mu(a))^2} \leq  \sum_{a \in \brk*{K}} \frac{\EE_{s_t\sim p_t} s_t^2(a)}{P_t^\mu(a)}\\
    & \leq \sum_{a \in \brk*{K}} \frac{\EE_{s_t\sim p_t} s_t(a)}{P_t^\mu(a)} = \sum_{a \in \brk*{K}} \frac{P_t(a)}{(1-K\mu)P_t(a) + \mu}.
  \end{align*}
  Here we use \Holder's inequality twice, using that
  $\|\ls\|_\infty \leq 1$ and $s \in \Delta(\Acal)$.  Now, since the
  function $x\mapsto{}1/(1-K\mu + \mu/x)$ is concave in $x$, it
  follows that
  \begin{align*}
      \sum_{a\in\brk*{K}}\frac{P_t(a)}{(1-K\mu{})P_t(a)  + \mu} 
      &= \sum_{a\in\brk*{K}}\frac{1}{(1-K\mu{})  + \mu/P_t(a)}
        \leq K\frac{1}{(1-K\mu{})  + K\mu/\sum_{a\in\brk*{K}}P_t(a)}
        =K,
  \end{align*}
  which proves the claim for $\cS \subset \Delta(\Acal)$.

    We proceed in the same fashion for both the ramp and hinge loss. Recall the definition $P_t^\mu(a) = (1-K\mu) \EE_{s_t \sim p_t} \frac{s_t(a)}{\sum_{a' \in \brk*{K}} s_t(a')} + \mu$. We have
    \begin{align*}
      \En\brk*{\En_{s_t\sim{}p_t}\tri*{s_t,\hat{\ls}_{t}}^{2}\mid{}\cJ_t} 
      &= \sum_{a\in\brk*{K}}P_{t}^{\mu}(a)\frac{\En_{s_{t}\sim{}p_t}\ls_{t}^{2}(a)s_{t}^{2}(a)}{(P_{t}^{\mu}(a))^{2}}
      = \sum_{a\in\brk*{K}}\frac{\En_{s_{t}\sim{}p_t}\ls_{t}^{2}(a)s_{t}^{2}(a)}{P_{t}^{\mu}(a)}\\
      &\leq \sum_{a\in\brk*{K}}\frac{\En_{s_{t}\sim{}p_t}s_{t}^{2}(a)}{P_{t}^{\mu}(a)} 
      \leq \max_{a\in\brk*{K}}\max_{s\in\cS}s(a)\cdot\sum_{a\in\brk*{K}}\frac{\En_{s_{t}\sim{}p_t}s_{t}(a)}{P_{t}^{\mu}(a)}\\
      &= \max_{a\in\brk*{K}}\max_{s\in\cS}s(a)\cdot\sum_{a\in\brk*{K}}\frac{\En_{s_{t}\sim{}p_t}s_{t}(a)}{(1-\mu{}K)\En_{s_t\sim{}p_t}\frac{s_t(a)}{\sum_{a'\in\brk*{K}}s_t(a')} + \mu} \\
     &\leq K\cdot\prn*{\max_{a\in\brk*{K}}\max_{s\in\cS}s(a)}^{2}\cdot\sum_{a\in\brk*{K}}\frac{\En_{s_t\sim{}p_t}\frac{s_t(a)}{\sum_{a'\in\brk*{K}}s_t(a')}}{(1-K\mu)\En_{s_t\sim{}p_t}\frac{s_t(a)}{\sum_{a'\in\brk*{K}}s_t(a')} + \mu}.
    \end{align*}
    Here we first apply the definition of $\hat{\ell}_t$ and cancel
    out one factor of $P_t^\mu$ in the denomator. Then we apply
    \Holder's inequality, using that $s_t(a) \ge 0$. Expanding
    the definition $P_t^\mu$ and using the upper bound
    $\sum_{a' \in \brk*{K}} s_t(a') \le K \max_a \max_s s_t(a)$,
    yields the final expression.

    Now, let
    $q_{a} \triangleq
    \En_{s_t\sim{}p_t}\frac{s_t(a)}{\sum_{a'\in\brk*{K}}s_t(a')}$, and apply the concavity argument above. This yields
    \[
      K^2\cdot\prn*{\max_{a\in\brk*{K}}\max_{s\in\cS}s(a)}^{2}.
    \]
    For the set $\cS$ induced by the ramp loss we have $\max_{a\in\brk*{K}}\max_{s\in\cS}s(a)\leq{}1$, and for the set $\cS$ induced by the hinge loss we have $\max_{a\in\brk*{K}}\max_{s\in\cS}s(a)\leq{}(1+\frac{B}{\gamma})$.
  \end{proof}

\section{Comparing Multiclass Loss Functions and Notions of Realizability}
\label{app:comparison}
\newcommand{\mchinge}{\textrm{MC-hinge}}
\newcommand{\cchinge}{\textrm{CC-hinge}}

While our surrogate loss functions apply to general cost-sensitive
classification, when specialized to the multiclass zero-one feedback, as in bandit
multiclass prediction, they are somewhat non-standard. In this
appendix we provide a discussion of the differences, focusing on the
hinge loss. 

Let us detail the multiclass setting: On each round, the adversary
chooses a pair $(x,y^\star)$ where $x \in \Xcal$, $y^\star \in \Acal$
and shows $x$ to the learner. The learner then makes a prediction
$\hat{y}$. The 0/1-loss for the learner is $\one\{\hat{y} \ne
y\}$. Using a class of regression functions $\Gcal \subset (\Xcal \to \RR^K)$,
the standard multiclass hinge loss for a regressor $g \in \Gcal$ is:
\begin{align*}
\ell_{\mchinge}^\gamma(g, (x,y^\star)) \defeq \max\{ 1 - \gamma^{-1}(g(x)_{y^\star} - \max_{y \ne y^\star} g(x)_y), 0\}.
\end{align*}
On the other hand, for our results we assume that the regressor class
$\Fcal \subset (\Xcal\to \RRz^K)$,  and the cost-sensitive hinge loss that we
use here is:
\begin{align*}
\ell_{\cchinge}^\gamma(f,(x,y^\star)) \defeq \sum_{y \ne y^\star} \hinge(f(x)_y) = \sum_{y \ne y^\star} \max\{1 + \gamma^{-1}f(x)_y,0\}.
\end{align*}
More precisely in~\pref{corr:bandit_multiclass}, we are measuring the
benchmark using $\ell_{\cchinge}$ and our bound is
\begin{align*}
\EE \sum_{t=1}^T \ell_t(a_t) \leq \inf_{\theta \in \Theta} \frac{1}{K}\sum_{t=1}^T \ell_{\cchinge}^\gamma(f(\cdot;\theta),(x_t,y_t^\star)) + \Otilde\prn*{\frac{B}{\gamma}\sqrt{dT}}.
\end{align*}
On the other hand, the open problem of~\citet{abernethy2009efficient}
asks for a $\sqrt{dT}$ bound when the benchmark is measured using
$\ell_{\mchinge}$. As we will see, the two loss functions are somewhat
different.

Let us first standardize the function classes. By rebinding $f_g(x)_y
\defeq g(x)_y - K^{-1}\sum_{y'}g(x)_{y'}$ we can easily construct a
``sum-to-zero" class from an unconstrained class $\Gcal \subset (\Xcal
\to \RR^K)$, and with this definition, the cost-sensitive hinge loss
for any function $g \in \cG$ is:
\begin{align*}
\ell_{\cchinge}^\gamma(g,(x,y^\star)) \defeq \sum_{y \ne y^\star} \hinge(f_g(x)_y) = \sum_{y \ne y^\star} \max\{1 - \gamma^{-1}(K^{-1}\sum_{y'}g(x)_{y'} - g(x)_y),0\}.
\end{align*}
The main proposition in this appendix is that if the cost-sensitive
hinge loss is zero, then so is the multiclass hinge loss, while the
converse is not true. 
\begin{proposition}
  \label{prop:cc_to_mc}
  We have the following implication
\begin{align*}
  \ell_{\cchinge}^\gamma(g,(x,y^\star)) = 0 \Rightarrow \ell_{\mchinge}^{K\gamma}(g,(x,y^\star)) = 0. 
\end{align*}
The converse does not hold: For any $\gamma,\tilde{\gamma}>0$ and $K \geq
0$ there exists a function $g$ and an $(x,y^\star)$ pair such that
$\ell_{\mchinge}^\gamma(g,(x,y^\star)) = 0$ but
$\ell_{\cchinge}^{\tilde{\gamma}}(g,(x,y^\star)) \geq 1$.
\end{proposition}
Note that $\ell_{\mchinge}^{\gamma_1} \geq \ell_{\mchinge}^{\gamma_2}$
whenever $\gamma_1 \geq \gamma_2$, so the first implication also
holds when the RHS is replaced with $\ell_{\mchinge}^\gamma$.

The proposition implies that cost-sensitive hinge realizability ---
that there exists a predictor $g^\star$ such that
$\ell_{\cchinge}(g^\star, (x,y^\star)) = 0$ for all rounds --- is a
\emph{strictly stronger} condition than multiclass hinge
realizability.
%% This means that there exist problems for which $\min_{g
%%   \in \Gcal} \sum_{t=1}^T \ell_{\mchinge}^\gamma(g,(x_t,y_t^\star)) =
%% 0$ but for which $\min_{g \in \Gcal}\sum_{t=1}^T
%% \ell_{\cchinge}^{\tilde{\gamma}}(g,(x_t,y_t^\star)) = \Omega(T)$ for
%% all $\tilde{\gamma}>0$. 
Under multiclass separability assumptions~\citep{kakade2008efficient},
the bandit multiclass surrogate benchmark is zero, while our
cost-sensitive benchmark may still be large, and so the cost-sensitive
translation approach cannot be used to obtain a sublinear upper bound
on the number of mistakes made by the learner. In this vein,
\lmc\xspace does not completely resolve the open problem
of~\citet{abernethy2009efficient}, since the loss function we use can
result in a weaker bound than desired. Note that the cost-sensitive
surrogate loss may also prevent us from exploiting \emph{small-loss}
structure in the multiclass surrogate to obtain fast rates.

On the other hand, our cost-sensitive surrogate losses are applicable
in a much wider range of problems, and cost-sensitive structure is
common in contextual bandit settings. As such, we believe that
designing algorithms for this more general setting is valuable.

\begin{proof}[\pfref{prop:cc_to_mc}]
  If $\ell_{\cchinge}^\gamma(g,(x,y^\star)) = 0$ then we know that for all $y\ne y^\star$, we must have
  \begin{align*}
    \forall y \ne y^\star: \ \frac{1}{K}\sum_{y'}g(x)_{y'} - g(x)_y \geq \gamma.
  \end{align*}
  Adding these inequalities together for $y \ne y^\star$ and subtracting
  $g(x)_{y^\star}$ from both sides we get
  \begin{align*}
    & \frac{K-1}{K}\sum_{y'}g(x)_{y'} - \sum_{y \ne y^\star}g(x)_y -g(x)_{y^\star} \geq (K-1)\gamma - g(x)_{y^\star}\\
    & \Rightarrow g(x)_{y^\star} \geq (K-1)\gamma + \frac{1}{K}\sum_{y}g(x)_y\\
    & \Rightarrow g(x)_{y^\star} \geq K\gamma + g(x)_y \ \forall y \ne y^\star.
  \end{align*}
The last line follows from re-using the original inequality and proves
the desired implication.
%% \begin{align*}
%% \ell_{\cchinge}^\gamma(g,(x,y^\star)) = 0 \Rightarrow \ell_{\mchinge}^{K\gamma}(g,(x,y^\star)) = 0 \Rightarrow \ell_{\mchinge}^\gamma(g,(x,y^\star)) = 0.
%% \end{align*}

On the other hand, if $\ell_{\mchinge}^{K\gamma}(g,(x,y^\star)) = 0$
it does not imply that $\ell_{\cchinge}^{\gamma}(g,(x,y^\star)) =
0$. To see why, suppose $K=3$ and assume $y^\star=y_1$ with $y_2,y_3$
as the other labels. Set $g(x)_{y_1}=3\gamma$, $g(x)_{y_2}=0$ and
$g(x)_{y_3} = -3\gamma$. With these predictions, we have $\sum_y
g(x)_y = 0$ and also that $\ell_{\mchinge}^{3\gamma}(g,(x,y_1))=
0$. On the other hand since $g(x)_{y_2}=\sum_y g(x)_y = 0$, we get:
\begin{align*}
\ell_{\cchinge}^{\tilde{\gamma}}(g,(x,y)) \geq \max\{1 - \tilde{\gamma}^{-1}( \sum_yg(x)_y - g(x)_{y_2}),0\} = 1,
\end{align*}
for any $\tilde{\gamma}$. This proves that the converse cannot be
true.
\end{proof}

\section{Proofs from \pref{sec:minimax}}
\label{app:minimax}
% !TEX root = paper.tex

Let us start with an intermediate result, which will simplify the proof of~\pref{thm:chaining_v3}.
\begin{theorem}
    \label{thm:chaining_v2}
  Assume $\nrm*{\ls}_{1}\leq{}1$ for all $\ls\in\cL$\footnote{Measuring loss in $\ls_1$ may seem restrictive, but this is natural when working with importance-weighted losses since these are $1$-sparse, and by duality this enables us to cover in $\ls_{\infty}$ norm on the output space.} and $\sup_{s\in\cS}\nrm*{s}_{\infty}\leq{}1$. Further assume that $\cS$ and $\cL$ are compact. Fix any constants $\eta\in(0,1]$, $\lambda>0$, and $\beta>\alpha>0$. Then there exists an algorithm with the following deterministic regret guarantee:
  {\small
  \begin{align*}
    \sum_{t=1}^{T}\En_{s_t\sim{}p_t}\tri*{
      s, \ls_t} -\inf_{g\in\cG}\sum_{t=1}^{T}\tri*{g(x_t), \ls_t} &\leq{} 
      2\eta\sum_{t=1}^{T}\En_{s_t\sim{}p_t}\tri*{s_t, \ls_t}^{2} +  \frac{4}{\eta}\log\cN_{\infty,\infty}(\beta/2, \cG, T) + 3e^2\alpha\sum_{t=1}^{T}\nrm*{\ls_{t}}_{1}\\
   &~~~~~ + 24e\prn*{\frac{\lambda}{4}\sum_{t=1}^{T}\nrm*{\ls_{t}}_{1}^{2} + \frac{1}{\lambda}}\int_{\alpha}^{\beta}\sqrt{\log\cN_{\infty,\infty}(\veps, \cG, T)}d\veps.
  \end{align*}}
  \end{theorem}
  The difference here is that have set $R,B=1$. The first part of this
  section will be devoted to proving this theorem, and \pref{thm:chaining_v3} will follow from this result via \pref{corr:full_info_scaled}.

\subsection{Preliminaries}

\begin{definition}[Cover for a collection of trees]
  \label{def:cover_tree}
  For a collection of $\bbR^K$-valued trees $U$ of length $T$, we let $\inftycover(\veps, U)$, denote the cardinality of the smallest set $V$ of $\bbR^K$ valued trees for which
  \[
    \forall{}\bu\in{}U\;\forall{}\eps\in\pmo^{T}\;\exists{}\bv\in{}V\textnormal{ s.t. }\max_{t\in\brk*{T}}\nrm*{\bu_{t}(\eps) - \bv_{t}(\eps)}_{\infty}\leq\veps.
  \]
\end{definition}
\begin{definition}[$L_{\infty}/\ls_{\infty}$ radius]
For a function class $\cF$, define \[\mathrm{rad}_{\infty,\infty}(\cF,T) = \min\crl*{\veps\mid{}\log\cN_{\infty,\infty}(\veps,\cF,T)=0}.\] For a collection $U$ of trees, define $\mathrm{rad}_{\infty,\infty}(U) = \min\crl*{\veps\mid{}\log\cN_{\infty,\infty}(\veps,U)=0}$.
\end{definition}

The following two lemmas are Freedman-type inequalities for Rademacher
tree processes that we will use in the sequel. The first has an
explicit dependence on the range, while the second does not.
  \begin{lemma}
  \label{lem:offset_finite_R}
  For any collection of $\brk*{-R,+R}$-valued trees $V$ of length $T$, for any $\eta>0$ and $\alpha>0$,
\begin{align*}
\En_{\eps}\sup_{\bv\in{}V}\brk*{
      \sum_{t=1}^{T}\eps_{t}\prn*{\bv_{t}(\eps) - \eta{}\bv_{t}^{2}(\eps)}   - \alpha\eta{}\bv_{t}^{2}(\eps)
      } \leq{} 2\log\abs*{V}\cdot{}\prn*{\frac{1}{\alpha\eta}\vee\frac{\eta{}R^{2}}{\alpha}}.
\end{align*}

\end{lemma}
\begin{proof}[\pfref{lem:offset_finite_R}]
  Take $V$ to be finite without loss of generality (otherwise the bound is vacuous). As a starting point, for any $\lambda>0$ we have
  \begin{align*}
    &\En_{\eps}\sup_{\bv\in{}V}\brk*{
      \sum_{t=1}^{T}\eps_{t}\prn*{\bv_{t}(\eps) - \eta{}\bv_{t}^{2}(\eps)}   - \alpha\eta{}\bv_t^{2}(\eps)
      } \\
    &\leq\frac{1}{\lambda}\log\prn*{\sum_{\bv\in{}V}\En_{\eps}\exp\prn*{
      \sum_{t=1}^{T}\eps_{t}\lambda\prn*{\bv_{t}(\eps) - \eta{}\bv_{t}^{2}(\eps)}   - \lambda\alpha\eta{}\bv_{t}^{2}(\eps)
      }}.
      \intertext{Applying the standard Rademacher mgf bound $\En_{\eps}e^{\lambda\eps}\leq{}e^{\frac{1}{2}\lambda^{2}}$ conditionally at each time starting from $t=T$, this is upper bounded by}
    &\leq\frac{1}{\lambda}\log\prn*{\sum_{\bv\in{}V}\max_{\eps}\exp\prn*{
      \sum_{t=1}^{T}\frac{1}{2}\lambda^{2}\prn*{\bv_{t}(\eps) - \eta{}\bv_{t}^{2}(\eps)}^{2}   - \lambda\alpha\eta{}\bv_{t}^{2}(\eps)
      }}.
  \end{align*}
  Since $v$ takes values in $\brk*{-R,+R}$, the exponent at time $t$ can be upper bounded as 
  \[
    \frac{1}{2}\lambda^{2}\prn*{\bv_{t}(\eps) - \eta{}\bv_{t}^{2}(\eps)}^{2}   - \lambda\alpha\eta{}\bv_{t}^{2}(\eps)
    \leq{} \lambda^{2}\prn*{1+\eta^{2}R^{2}}\bv_t^{2}(\eps)   - \lambda\alpha\eta{}\bv_{t}^{2}(\eps).
  \]
By setting $\lambda=\frac{1}{2}\min\crl*{\alpha\eta,\alpha/(\eta{}R^{2})}$, this is bounded by zero, which leads to a final bound of $\log\abs*{V}/\lambda$.  
\end{proof}

  \begin{lemma}
  \label{lem:offset_finite_simple}
  For any collection of trees $V$ of length $T$, for any $\eta>0$,
\begin{align*}
\En_{\eps}\sup_{\bv\in{}V}\brk*{
      \sum_{t=1}^{T}\eps_{t}\bv_{t}(\eps)   - \eta{}\bv_{t}^{2}(\eps)
      } \leq{} \frac{\log\abs*{V}}{2\eta}.
\end{align*}

\end{lemma}
\begin{proof}[\pfref{lem:offset_finite_simple}]
  Take $V$ to be finite without loss of generality. As in the proof
  of~\pref{lem:offset_finite_R}, using the standard Rademacher mgf
  bound and working backward from $T$, for any $\lambda>0$ we have
  \begin{align*}
    \En_{\eps}\sup_{\bv\in{}V}\brk*{
      \sum_{t=1}^{n}\eps_{t}\bv_{t}(\eps)   - \eta{}\bv_{t}^{2}(\eps)
      } 
    &\leq\frac{1}{\lambda}\log\prn*{\sum_{\bv\in{}V}\En_{\eps}\exp\prn*{
      \sum_{t=1}^{n}\eps_{t}\lambda\bv_{t}(\eps)   - \eta{}\lambda\bv_{t}^{2}(\eps)
      }}\\
    &\leq\frac{1}{\lambda}\log\prn*{\sum_{\bv\in{}V}\max_{\eps}\exp\prn*{
      \sum_{t=1}^{n}\frac{1}{2}\lambda^{2}\bv_{t}(\eps)^{2}   - \eta{}\lambda\bv_{t}^{2}(\eps)
      }}.
  \end{align*}
The exponent at time $t$ is
  \[
    \frac{1}{2}\lambda^{2}\bv_{t}^{2}(\eps)   - \eta{}\lambda\bv_{t}^{2}(\eps).
  \]
By setting $\lambda=2\eta$, this is exactly zero, which leads to a final bound of $\log\abs*{V}/\lambda$.  
\end{proof}

\begin{lemma}
\label{lem:symmetrize}
Let $\cZ$, $\cW$, and $\cG$ be abstract sets and let functions $A_{g}:\cW\times{}\cZ\times{}\cZ\to{}\bbR$ and $B_{g}:\cW\times{}\cZ\times{}\cZ\to\bbR$ be given for each element $g\in\cG$. Suppose that for any $z,z'\in\cZ$  and $w\in\cW$ it holds that $A(w, z,z')=-A(w, z',z)$ and $B(w, z,z')=B(w, z',z)$.
Then
\begin{align}
&\dtri*{\sup_{w_t\in\cW}\sup_{q_{t}\in\Delta(\cZ)}\En_{z_t,z'_t\sim{}q_t}}_{t=1}^{T}\sup_{g\in\cG}\sum_{t=1}^{T}A_g(w_t, z_t,z'_t) + B_g(w_t,z_t,z'_t)
\\&\leq{} \dtri*{\sup_{w_t\in\cW}\sup_{q_{t}\in\Delta(\cZ)}\En_{\eps_t}\En_{z_t,z'_t\sim{}q_t}}_{t=1}^{T}\sup_{g\in\cG}\sum_{t=1}^{T}\eps_tA_g(w_t, z_t,z'_t) + B_g(w_t,z_t,z'_t),
\end{align}
where $\eps$ is a sequence of independent Rademacher random variables.
\end{lemma}
See \pref{ssec:chaining_proof} for a discussion of the $\dtri*{\star}$ notation used in the above lemma statement.
\begin{proof}[\pfref{lem:symmetrize}]
See proof of Lemma 3 in \cite{rakhlin2010online}.
\end{proof}

% \begin{proof}[\pfref{thm:chaining_v2}]
\subsection{Proof of \pref{thm:chaining_v2}}
\label{ssec:chaining_proof}
Before proceeding, we note that this proof uses a number of techniques which are now somewhat standard in minimax analysis of online learning, and the reader may wish to refer to, e.g., \cite{RakSriTew14jmlr} for a comprehensive introduction to this type of analysis.

  Let $\eta_1,\eta_2,\eta_3>0$ be fixed constants to be chosen later in the proof, and define
  \[
    B(p_{1:T}, \ls_{1:T}) \triangleq \underbrace{\eta_{1}\sum_{t=1}^{T}\nrm*{\ls_{t}}_{1} + \eta_{2}\sum_{t=1}^{T}\nrm*{\ls_{t}}_{1}^{2}}_{\triangleq B_{1}(\ls_{1:T})} + \underbrace{2\eta_{3}\sum_{t=1}^{T}\En_{s\sim{}p_t}\tri*{s, \ls_t}^{2}}_{\triangleq B_{2}(p_{1:T},\ls_{1:T})}.
  \]
  We consider a game where the goal of the learner is to achieve
  regret bounded by $B$, plus some additive constant that will depend
  on $\eta_1,\eta_2,\eta_3$, and the complexity of the class $\cF$.
  The value of the game is given by:
\begin{align*}
\cV \triangleq \dtri*{\sup_{x_t\in\cX}\inf_{p_t\in\Delta(\cS)}\sup_{\ls_t\in\cL}\En_{s\sim p_t}}_{t=1}^{T}\brk*{
\sum_{t=1}^{T}\tri*{s, \ls_t} -\inf_{g\in\cG}\sum_{t=1}^{n}\tri*{g(x_t), \ls_t}- B(p_{1:T}, \ls_{1:T})
}.
\end{align*}
Here we are using the notation $\dtri{\star}_{t=1}^{T}$ to denote sequential application of the operator $\star$ (indexed by $t$) from time $t=1,\ldots,T$, following e.g. \citep{RakSriTew14jmlr}. This notation means that first the adversary
chooses $x_1$, then the learner chooses $p_1$, and then the adversary chooses
$\ell_1$ while the learner samples $s_1$ and suffers the loss
$\langle s_1,\ell_1\rangle$. Then we proceed to round $2$ and so on,
so that the learner is trying to minimize the (offset) regret after $T$
rounds while the adversary is trying to maximize it. If we show that
$\cV\leq{}C$ for some constant $C$ then we have established existence of a
randomized strategy that achieves an adaptive regret bound of
$B(\cdot)+C$.  See~\citet{foster2015adaptive} for a more extensive discussion of this principle.

\subsubsection{Minimax swap}

At time $t$ the value to go is given by
{\small
\begin{align*}
\sup_{x_t\in\cX}\inf_{p_t\in\Delta(\cS)}\sup_{\ls_t\in\cL}\biggl[&
\En_{s\sim{}p_t}\tri*{s, \ls_t} - 2\eta_3\En_{s\sim{}p_t}\tri*{s,\ls_t}^{2} - \eta_1\nrm*{\ls_t}_1 - \eta_2\nrm*{\ls_t}_{1}^{2} \\
&+ 
\dtri*{\sup_{x_\tau\in\cX}\inf_{p_\tau\in\Delta(\cS)}\sup_{\ls_\tau\in\cL}}_{\tau=t+1}^{T}\brk*{
\sum_{\tau=t+1}^{T}\En_{s\sim{}p_\tau}\tri*{s, \ls_\tau} -\inf_{g\in\cG}\sum_{\tau=1}^{T}\tri*{g(x_\tau), \ls_\tau}- B(p_{\tau+1:T}, \ls_{\tau+1:T})
}
\biggr].
\end{align*}
}Note that the benchmark's loss is only evaluated at the end, while we
are incorporating the adaptive term into the instantaneous value.
Convexifying the $\ls_t$ player by allowing them to select a randomized strategy $q_t$, this is equal to
{\small
  \begin{align*}
\sup_{x_t\in\cX}\inf_{p_t\in\Delta(\cS)}\sup_{q_t\in\Delta(\cL)}\En_{\ls_t\sim{}q_t}\biggl[&
\En_{s\sim{}p_t}\tri*{s, \ls_t} - 2\eta_3\En_{s\sim{}p_t}\tri*{s,\ls_t}^{2} - \eta_1\nrm*{\ls_t}_1 - \eta_2\nrm*{\ls_t}_{1}^{2} \\
&+ 
\dtri*{\sup_{x_\tau\in\cX}\inf_{p_\tau\in\Delta(\cS)}\sup_{\ls_\tau\in\cL}}_{\tau=t+1}^{T}\brk*{
\sum_{\tau=t+1}^{T}\En_{s\sim{}p_\tau}\tri*{s, \ls_\tau} -\inf_{g\in\cG}\sum_{\tau=1}^{T}\tri*{g(x_\tau), \ls_\tau}- B(p_{\tau+1:T}, \ls_{\tau+1:T})
}
\biggr].
  \end{align*}
% \begin{align*}
% \sup_{x_t\in\cX}\inf_{p_t\in\Delta_{\cF}}\sup_{q_t\in\Delta_{\cC}}\En_{\bls_{t}\sim{}q_t}\biggl[&
% \En_{f_{t}\sim{}p_{t}}\brk*{\bls_{t}(f_t(x_t))} - \En_{f_{t}\sim{}p_{t}}\brk*{\bls^{2}_{t}(f_t(x_t))} \\
% &+ 
% \dtri*{\sup_{x_s\in\cX}\inf_{p_s\in\Delta_{\cF}}\sup_{\bls_{s}\in\cC}}_{s=t+1}^{n}\brk*{
% \sum_{s=t+1}^{n}\En_{f_{s}\sim{}p_{s}}\brk*{\bls_{s}(f_s(x_s))} -\inf_{f\in\cF}\sum_{s=1}^{n}\bls_{s}(f(x_s))- 2\eta\sum_{s=t+1}^{n}\En_{f_{s}\sim{}p_{s}}\brk*{\bls^{2}_{s}(f_s(x_s))}
% }
% \biggr].
% \end{align*}
}This quantity is convex in $p_{t}$ and linear in $q_{t}$ so, under
the compactness assumption on $\cS$ and $\cL$, the minimax theorem implies that this is equal to
{\small
\begin{align*}
\sup_{x_t\in\cX}\sup_{q_t\in\Delta(\cL)}\inf_{p_t\in\Delta(\cS)}\En_{\ls_t\sim{}q_t}\biggl[&
\En_{s\sim{}p_t}\tri*{s, \ls_t} - 2\eta_3\En_{s\sim{}p_t}\tri*{s,\ls_t}^{2} - \eta_1\nrm*{\ls_t}_1 - \eta_2\nrm*{\ls_t}_{1}^{2} \\
&+ 
\dtri*{\sup_{x_\tau\in\cX}\inf_{p_\tau\in\Delta(\cS)}\sup_{\ls_\tau\in\cL}}_{\tau=t+1}^{T}\brk*{
\sum_{\tau=t+1}^{T}\En_{s\sim{}p_{\tau}}\tri*{s, \ls_\tau} -\inf_{g\in\cG}\sum_{\tau=1}^{T}\tri*{g(x_\tau), \ls_\tau}- B(p_{\tau+1:T}, \ls_{\tau+1:T})
} \biggr].
% \biggl[&
% \En_{f_{t}\sim{}p_{t}}\brk*{\bls_{t}(f_t(x_t))} - \En_{f_{t}\sim{}p_{t}}\brk*{\bls^{2}_{t}(f_t(x_t))} \\
% &+ 
% \dtri*{\sup_{x_s\in\cX}\inf_{p_s\in\Delta_{\cF}}\sup_{\bls_{s}\in\cC}}_{s=t+1}^{n}\brk*{
% \sum_{s=t+1}^{n}\En_{f_{s}\sim{}p_{s}}\brk*{\bls_{s}(f_s(x_s))} -\inf_{f\in\cF}\sum_{s=1}^{n}\bls_{s}(f(x_s))- 2\eta\sum_{s=t+1}^{n}\En_{f_{s}\sim{}p_{s}}\brk*{\bls^{2}_{s}(f_s(x_s))}
% }
% \biggr].
\end{align*}
}Repeating this analysis at each timestep and expanding the terms from $B_2$, we arrive at the expression
{\small
\begin{align*}
  \cV = \dtri*{\sup_{x_t\in\cX}\sup_{q_t\in\Delta(\cL)}\inf_{p_t\in\Delta(\cS)}\En_{\ls_t\sim{}q_t}}_{t=1}^{T}\brk*{
    \sum_{t=1}^{T}\En_{s\sim{}p_{t}}\brk*{\tri*{s, \ls_t} - 2\eta_3\tri*{s, \ls_t}^{2}} -\inf_{g\in\cG}\sum_{t=1}^{T}\tri*{g(x_t), \ls_t} -B_1(\ls_{1:T})}.
\end{align*}}
\subsubsection{Upper bound by martingale process}
We now use a standard ``rearrangement'' trick (see \citep{RakSriTew14jmlr}, Theorem 1) to show that 
{\small
\begin{align*}
  \cV&= \dtri*{\sup_{x_t\in\cX}\sup_{q_t\in\Delta(\cL)}\En_{\ls_t\sim{}q_t}}_{t=1}^{T}\brk*{\sup_{g\in\cG}\brk*{
    \sum_{t=1}^{T}\inf_{p_t\in\Delta(\cS)}\En_{s\sim{}p_{t}}\En_{\ls'_t\sim{}q_t}\brk*{\tri*{s, \ls'_t} - 2\eta_3\tri*{s, \ls'_t}^{2}} -  \sum_{t=1}^{T}\tri*{f(x_t), \ls_t}
                     } -B_1(\ls_{1:T})},
                     \end{align*}
}where $\ls'_{1:T}$ is a sequence of ``tangent'' samples, where $\ls'_t$ is an independent copy of $\ls_t$ conditioned on $\ls_{1:t-1}$. This can be seen by working backwards from time $T$. Indeed, at time $T$, expanding the $\dtri*{\star}_{t=1}^{T}$ operator, we have
{\small
                      \begin{align*}
  \cV = \dtri*{\cdots}_{t=1}^{T-1}
  \sup_{x_T\in\cX}\sup_{q_T\in\Delta(\cL)}\inf_{p_T\in\Delta(\cS)}\En_{\ls_T\sim{}q_T}
    \biggl[&\sum_{t=1}^{T-1}\En_{s\sim{}p_{t}}\brk*{\tri*{s, \ls_t} - 2\eta_3\tri*{s, \ls_t}^{2}}
    + \En_{s\sim{}p_{T}}\brk*{\tri*{s, \ls_T} - 2\eta_3\tri*{s, \ls_T}^{2}}  \\
    &-\inf_{g\in\cG}\sum_{t=1}^{T}\tri*{g(x_t), \ls_t} -B_1(\ls_{1:T}) \biggr].
\end{align*}
}Using linearity of expectation:
{\small
                      \begin{align*}
= \dtri*{\cdots}_{t=1}^{T-1}
  \sup_{x_T\in\cX}\sup_{q_T\in\Delta(\cL)}\inf_{p_T\in\Delta(\cS)}\En_{\ls_T\sim{}q_T}
    \biggl[&\sum_{t=1}^{T-1}\En_{s\sim{}p_{t}}\brk*{\tri*{s, \ls_t} - 2\eta_3\tri*{s, \ls_t}^{2}}
    + \En_{\ls'_T\sim{}q_T}\En_{s\sim{}p_{T}}\brk*{\tri*{s, \ls'_T} - 2\eta_3\tri*{s, \ls'_T}^{2}}  \\
    &-\inf_{g\in\cG}\sum_{t=1}^{T}\tri*{g(x_t), \ls_t} -B_1(\ls_{1:T}) \biggr].
\end{align*}
}Using that only a single term has functional dependence on $p_{T}$:
{\small
                      \begin{align*}
= \dtri*{\cdots}_{t=1}^{T-1}
  \sup_{x_T\in\cX}\sup_{q_T\in\Delta(\cL)}\En_{\ls_T\sim{}q_T}
    \biggl[&\sum_{t=1}^{T-1}\En_{s\sim{}p_{t}}\brk*{\tri*{s, \ls_t} - 2\eta_3\tri*{s, \ls_t}^{2}}
    + \inf_{p_T\in\Delta(\cS)}\En_{\ls'_T\sim{}q_T}\En_{s\sim{}p_{T}}\brk*{\tri*{s, \ls'_T} - 2\eta_3\tri*{s, \ls'_T}^{2}}  \\
    &-\inf_{g\in\cG}\sum_{t=1}^{T}\tri*{g(x_t), \ls_t} -B_1(\ls_{1:T}) \biggr].
\end{align*}
}Expanding the infimum over $g\in\cG$:
{\small
                      \begin{align*}
= \dtri*{\cdots}_{t=1}^{T-1}
  \sup_{x_T\in\cX}\sup_{q_T\in\Delta(\cL)}\En_{\ls_T\sim{}q_T}
    \sup_{g\in\cG}\biggl[&\sum_{t=1}^{T-1}\En_{s\sim{}p_{t}}\brk*{\tri*{s, \ls_t} - 2\eta_3\tri*{s, \ls_t}^{2}}
    + \inf_{p_T\in\Delta(\cS)}\En_{\ls'_T\sim{}q_T}\En_{s\sim{}p_{T}}\brk*{\tri*{s, \ls'_T} - 2\eta_3\tri*{s, \ls'_T}^{2}}  \\
    &-\sum_{t=1}^{T}\tri*{g(x_t), \ls_t} -B_1(\ls_{1:T}) \biggr].
\end{align*}
}We handle time $T-1$ in a similar fashion by first splitting the $\dtri*{\star}_{t=1}^{T-1}$ operator:
{\small                      \begin{align*}
= \dtri*{\cdots}_{t=1}^{T-2}&
  \sup_{x_{T-1}\in\cX}\sup_{q_{T-1}\in\Delta(\cL)}\inf_{p_{T-1}\in\Delta(\cS)}\En_{\ls_{T-1}\sim{}q_{T-1}}\sup_{x_T\in\cX}\sup_{q_T\in\Delta(\cL)}\En_{\ls_T\sim{}q_T}\\
    \sup_{g\in\cG}\biggl[&\sum_{t=1}^{T-2}\En_{s\sim{}p_{t}}\brk*{\tri*{s, \ls_t} - 2\eta_3\tri*{s, \ls_t}^{2}}
       + \En_{s\sim{}p_{T-1}}\brk*{\tri*{s, \ls_{T-1}} - 2\eta_3\tri*{s, \ls_{T-1}}^{2}}\\
    &+ \inf_{p_T\in\Delta(\cS)}\En_{\ls'_T\sim{}q_T}\En_{s\sim{}p_{T}}\brk*{\tri*{s, \ls'_T} - 2\eta_3\tri*{s, \ls'_T}^{2}}-\sum_{t=1}^{T}\tri*{g(x_t), \ls_t} -B_1(\ls_{1:T}) \biggr].
\end{align*}
}Rearranging the supremums to make dependence on terms from time $T-1$ clear:
{\small
\begin{align*}
= \dtri*{\cdots}_{t=1}^{T-2}&
  \sup_{x_{T-1}\in\cX}\sup_{q_{T-1}\in\Delta(\cL)}\inf_{p_{T-1}\in\Delta(\cS)}\En_{\ls_{T-1}\sim{}q_{T-1}}\\
    \biggl[&\sum_{t=1}^{T-2}\En_{s\sim{}p_{t}}\brk*{\tri*{s, \ls_t} - 2\eta_3\tri*{s, \ls_t}^{2}}
       + \En_{s\sim{}p_{T-1}}\brk*{\tri*{s, \ls_{T-1}} - 2\eta_3\tri*{s, \ls_{T-1}}^{2}}\\
    &+ \sup_{x_T\in\cX}\sup_{q_T\in\Delta(\cL)}\En_{\ls_T\sim{}q_T}\sup_{g\in\cG}\brk*{\inf_{p_T\in\Delta(\cS)}\En_{\ls'_T\sim{}q_T}\En_{s\sim{}p_{T}}\brk*{\tri*{s, \ls'_T} - 2\eta_3\tri*{s, \ls'_T}^{2}}-\sum_{t=1}^{T}\tri*{g(x_t), \ls_t} -B_1(\ls_{1:T})} \biggr].
\end{align*}
}Using linearity of expectation and moving the infimum over $q_{T-1}$:
{\small
\begin{align*}
= \dtri*{\cdots}_{t=1}^{T-2}&
  \sup_{x_{T-1}\in\cX}\sup_{q_{T-1}\in\Delta(\cL)}\En_{\ls_{T-1}\sim{}q_{T-1}}\\
    \biggl[&\sum_{t=1}^{T-2}\En_{s\sim{}p_{t}}\brk*{\tri*{s, \ls_t} - 2\eta_3\tri*{s, \ls_t}^{2}}
       + \inf_{p_{T-1}\in\Delta(\cS)}\En_{\ls'_{T-1}\sim{}q_{T-1}}\En_{s\sim{}p_{T-1}}\brk*{\tri*{s, \ls'_{T-1}} - 2\eta_3\tri*{s, \ls'_{T-1}}^{2}}\\
    &+ \sup_{x_T\in\cX}\sup_{q_T\in\Delta(\cL)}\En_{\ls_T\sim{}q_T}\sup_{g\in\cG}\brk*{\inf_{p_T\in\Delta(\cS)}\En_{\ls'_T\sim{}q_T}\En_{s\sim{}p_{T}}\brk*{\tri*{s, \ls'_T} - 2\eta_3\tri*{s, \ls'_T}^{2}}-\sum_{t=1}^{T}\tri*{g(x_t), \ls_t} -B_1(\ls_{1:T})} \biggr].
\end{align*}
}The last step is to move the supremums from time $t=T$ and the supremum over $g\in\cG$ outside the entire expression, similar to what was done at time $t=T$.
{\small
\begin{align*}
= \dtri*{\cdots}_{t=1}^{T-2}&
  \sup_{x_{T-1}\in\cX}\sup_{q_{T-1}\in\Delta(\cL)}\En_{\ls_{T-1}\sim{}q_{T-1}}\sup_{x_T\in\cX}\sup_{q_T\in\Delta(\cL)}\En_{\ls_T\sim{}q_T}\sup_{g\in\cG}\\
    \biggl[&\sum_{t=1}^{T-2}\En_{s\sim{}p_{t}}\brk*{\tri*{s, \ls_t} - 2\eta_3\tri*{s, \ls_t}^{2}}
       + \inf_{p_{T-1}\in\Delta(\cS)}\En_{\ls'_{T-1}\sim{}q_{T-1}}\En_{s\sim{}p_{T-1}}\brk*{\tri*{s, \ls'_{T-1}} - 2\eta_3\tri*{s, \ls'_{T-1}}^{2}}\\
    &+ \inf_{p_T\in\Delta(\cS)}\En_{\ls'_T\sim{}q_T}\En_{s\sim{}p_{T}}\brk*{\tri*{s, \ls'_T} - 2\eta_3\tri*{s, \ls'_T}^{2}}-\sum_{t=1}^{T}\tri*{g(x_t), \ls_t} -B_1(\ls_{1:T}) \biggr].
\end{align*}
}Repeating this argument down from time $t=T-2$ to time $t=1$ yields the result.
             
             To conclude this portion of the proof, we move to an upper bound by choosing the infimum over $p_t$ at each timestep $t$ to match $g$, which is possible because each infimum now occurs inside the expression for which the supremum over $g\in\cG$ is taken:
{\small
\begin{align}
 \cV &= \dtri*{\sup_{x_t\in\cX}\sup_{q_t\in\Delta(\cL)}\En_{\ls_t\sim{}q_t}}_{t=1}^{T}\brk*{\sup_{g\in\cG}\brk*{
    \sum_{t=1}^{T}\inf_{p_t\in\Delta(\cS)}\En_{s\sim{}p_{t}}\En_{\ls'_t\sim{}q_t}\brk*{\tri*{s, \ls'_t} - 2\eta_3\tri*{s, \ls'_t}^{2}} -  \sum_{t=1}^{T}\tri*{f(x_t), \ls_t}
                     } -B_1(\ls_{1:T})} \notag \\
  &\leq \dtri*{\sup_{x_t\in\cX}\sup_{q_t\in\Delta(\cL)}\En_{\ls_t\sim{}q_t}}_{t=1}^{T}\brk*{\sup_{g\in\cG}\brk*{
    \sum_{t=1}^{T}\En_{\ls'_t\sim{}q_t}\brk*{\tri*{g(x_t), \ls'_t} - 2\eta_3\tri*{g(x_t), \ls'_t}^{2}} -  \sum_{t=1}^{T}\tri*{g(x_t), \ls_t}
                     } -B_1(\ls_{1:T})} \notag \\
  &= \dtri*{\sup_{x_t\in\cX}\sup_{q_t\in\Delta(\cL)}\En_{\ls_t\sim{}q_t}}_{t=1}^{T}\brk*{\sup_{g\in\cG}\brk*{
    \sum_{t=1}^{T}\En_{\ls'_t\sim{}q_t}\brk*{\tri*{g(x_t), \ls'_t}} - \tri*{g(x_t), \ls_t} -  2\eta_3\sum_{t=1}^{T}\En_{\ls'_t\sim{}q_t}\brk*{\tri*{g(x_t), \ls'_t}^{2}}
                     } -B_1(\ls_{1:T})}\label{eq:pre_sym}.
\end{align}}

\subsubsection{Symmetrization}

Introduce the notation $H(x) = x - \eta_3 x^2$.
We now claim that the quantity appearing in \pref{eq:pre_sym} is bounded by
\begin{align}
  \label{eq:symmetrized}
  2\cdot\sup_{\bx}\sup_{\bls}\En_{\eps}\biggl[\sup_{g\in\cG}\biggl[ 
                          \sum_{t=1}^{T}\eps_{t}H(\tri*{g(\bx_t(\eps)),\bls_t(\eps)}) % \prn*{\tri*{g(\bx_t(\eps)), \bls_t(\eps)} - \eta_3\tri*{g(\bx_t(\eps)),\bls_t(\eps)}^{2}}  
                          -  \eta_3\sum_{t=1}^{T}\tri*{g(\bx_t(\eps)), \bls_t(\eps)}^{2}
              \biggr] -  B_{1}(\bls_{1:T}(\eps))\biggr],
\end{align}
where the supremum ranges over all $\cX$-valued trees $\bx$ and $\cL$-valued trees $\bls$, both of length $T$.

  The value
  {\small
  \begin{align*}
        \dtri*{\sup_{x_t\in\cX}\sup_{q_t\in\Delta(\cL)}\En_{\ls_t\sim{}q_t}}_{t=1}^{T}\brk*{\sup_{g\in\cG}\brk*{
        \sum_{t=1}^{T}\En_{\ls'_t\sim{}q_t}\brk*{\tri*{g(x_t), \ls'_t}} - \tri*{g(x_t), \ls_t} -  2\eta_3\sum_{t=1}^{T}\En_{\ls'_t\sim{}q_t}\brk*{\tri*{g(x_t), \ls'_t}^{2}}
    } -B_1(\ls_{1:T})},
    \end{align*}
    }by adding and subtracting the same term, is equal to
  {\small
    \begin{align*}
    &\dtri*{\sup_{x_t\in\cX}\sup_{q_t\in\Delta(\cL)}\En_{\ls_t\sim{}q_t}}_{t=1}^{T}\biggl[\sup_{g\in\cG}\biggl[ 
      \sum_{t=1}^{T}\En_{\ls'_t\sim{}q_t}\brk*{\tri*{g(x_t), \ls'_t} - \eta_3\tri*{g(x_t),\ls'_t}^{2}} - \prn*{\tri*{g(x_t), \ls_t}-\eta_3\tri*{g(x_t), \ls_t}^{2}} \\
      & \qquad \qquad \qquad \qquad \qquad \qquad \qquad -  \eta_3\sum_{t=1}^{T}\prn*{\En_{\ls_t\sim{}q_t}\brk*{\tri*{g(x_t), \ls_t}^{2}} + \tri*{g(x_t), \ls_t}^{2}}
      \biggr] -B_1(\ls_{1:T})\biggr] \\
    = & \dtri*{\sup_{x_t\in\cX}\sup_{q_t\in\Delta(\cL)}\En_{\ls_t\sim{}q_t}}_{t=1}^{T}\biggl[\sup_{g\in\cG}\biggl[ 
      \sum_{t=1}^{T}\En_{\ls'_t\sim{}q_t}\brk*{H(\tri*{g(x_t), \ls'_t})} - H(\tri*{g(x_t), \ls_t})  \\
      & \qquad \qquad \qquad \qquad \qquad \qquad \qquad  -  \eta_3\sum_{t=1}^{T}\prn*{\En_{\ls_t\sim{}q_t}\brk*{\tri*{g(x_t), \ls_t}^{2}} + \tri*{g(x_t), \ls_t}^{2}}
      \biggr] -B_1(\ls_{1:T})\biggr].
    \end{align*}
    }Using Jensen's inequality, this is upper bounded by
      \begin{align}
    \dtri*{\sup_{x_t\in\cX}\sup_{q_t\in\Delta(\cL)}\En_{\ls_t,\ls'_t\sim{}q_t}}_{t=1}^{T}\biggl[\sup_{g\in\cG}\biggl[ 
      \sum_{t=1}^{T}H(\tri*{g(x_t), \ls'_t}) - H(\tri*{g(x_t), \ls_t})  & \notag\\
      -  \eta_3\sum_{t=1}^{T}\prn*{\tri*{g(x_t), \ls'_t}^{2} + \tri*{g(x_t), \ls_t}^{2}}
      &\biggr] -B_1(\ls_{1:n})\biggr],\label{eq:tangent}
      \end{align}
      where $\ls'_{1:T}$ is a tangent sequence.
      We now claim that this is equal to
            \begin{align*}
    &\dtri*{\sup_{x_t\in\cX}\sup_{q_t\in\Delta(\cL)}\En_{\ls_t,\ls'_t\sim{}q_t}}_{t=1}^{T}\biggl[\sup_{g\in\cG}\biggl[ 
      \sum_{t=1}^{T}H(\tri*{g(x_t), \ls'_t}) - H(\tri*{g(x_t), \ls_t})   \\
      & \qquad \qquad \qquad \qquad \qquad \qquad \qquad -  \eta_3\sum_{t=1}^{T}\prn*{\tri*{g(x_t), \ls'_t}^{2} + \tri*{g(x_t), \ls_t}^{2}}
      \biggr] -\frac{1}{2}B_1(\ls_{1:T})- \frac{1}{2}B_1(\ls'_{1:T})\biggr].
            \end{align*}
          This can be seen as follows: Let $Q$ be the joint
          distribution over $\ls_1,\ldots,\ls_T$ obtaining the
          supremum above, or if the supremum is not obtained let it be
          any point in a limit sequence approaching the supremum. Then
          the value of the $B_{1}$ term in \pref{eq:tangent} is equal
          to (respectively, $\veps$-close to)
            \begin{align*}
              \En_{Q}B_{1}(\ls_{1:T}) &= \eta_{1}\sum_{t=1}^{T}\En_{Q}\nrm*{\ls_{t}}_{1} + \eta_{2}\sum_{t=1}^{T}\En_{Q}\nrm*{\ls_{t}}_{1}^{2} \\
                                      &= \eta_{1}\sum_{t=1}^{T}\En_{\ls_{1:t-1}}\En\brk*{\nrm*{\ls_{t}}_{1}\mid{}\ls_{1:t-1}} + \eta_{2}\sum_{t=1}^{T}\En_{\ls_{1:t-1}}\En\brk*{\nrm*{\ls_{t}}_{1}^{2}\mid{}\ls_{1:t-1}} \\
                                      &= \eta_{1}\sum_{t=1}^{T}\En_{\ls_{1:t-1}}\En\brk*{\nrm*{\ls_{t}'}_{1}\mid{}\ls_{1:t-1}} + \eta_{2}\sum_{t=1}^{T}\En_{\ls_{1:t-1}}\En\brk*{\nrm*{\ls_{t}'}_{1}^{2}\mid{}\ls_{1:t-1}}\\
                                      &= \En_{\ls_{1:T}}\En_{\ls'_{1:T}\mid{}\ls_{1:T}}B_1(\ls'_{1:T}).
            \end{align*}
            Replacing $\ls_t$ with $\ls_t'$ follows from the
            definition of the tangent sequence, since $\ls_t'$ and
            $\ls_t$ are identically distributed, conditioned on
            $\ls_{1:t-1}$.  This shows that we can replace
            $B_1(\ls_{1:T})$ with
            $B_1(\ls_{1:T})/2 + B_1(\ls'_{1:T})/2$ above, since we are
            working with the expectation.

            We have now established that \pref{eq:tangent} is equal to
            \begin{align*}
    \dtri*{\sup_{x_t\in\cX}\sup_{q_t\in\Delta(\cL)}\En_{\ls_t,\ls'_t\sim{}q_t}}_{t=1}^{T}\biggl[&\sup_{g\in\cG}\biggl[ 
      \sum_{t=1}^{T}\underbrace{H(\tri*{g(x_t), \ls'_t})}_{A_1} - \underbrace{H(\tri*{g(x_t), \ls_t})}_{A_2} -  \eta_3\biggl(\sum_{t=1}^{T}\underbrace{\tri*{g(x_t), \ls'_t}^{2} + \tri*{g(x_t), \ls_t}^{2}}_{A_3}\biggr)
              \biggr]  \\
       &-\frac{\eta_1}{2}\biggl(\sum_{t=1}^{T}\underbrace{\nrm*{\ls_t}_{1} + \nrm*{\ls'_t}_{1}}_{A_4}\biggr)-\frac{\eta_2}{2}\biggl(\sum_{t=1}^{T}\underbrace{\nrm*{\ls_t}_{1}^{2} + \nrm*{\ls'_t}_{1}^{2}}_{A_5}\biggr)\biggr].
            \end{align*}
            
            Fix a time $t$ and suppose the values of $\ls_t$ and $\ls'_t$ are exchanged. In this case the value of $A_{1}-A_{2}$ is switched to $A_{2}-A_{1}$, while the values of $A_3$, $A_4$, and $A_5$ are left unchanged. Appealing to \pref{lem:symmetrize}, we can therefore introduce Rademacher random variables $\eps_1,\ldots,\eps_T$ with equality as follows:
            {\small
                        \begin{align*}
                          \dtri*{\sup_{x_t\in\cX}\sup_{q_t\in\Delta(\cL)}\En_{\ls_t,\ls'_t\sim{}q_t}\En_{\eps_t}}_{t=1}^{T}\biggl[&\sup_{g\in\cG}\biggl[ 
      \sum_{t=1}^{T}\eps_{t}\prn*{H(\tri*{g(x_t), \ls'_t}) - H(\tri*{g(x_t), \ls_t})} -  \eta_3\prn*{\sum_{t=1}^{T}\tri*{g(x_t), \ls'_t}^{2} + \tri*{g(x_t), \ls_t}^{2}}
              \biggr]  \\
      & -\frac{\eta_1}{2}\prn*{\sum_{t=1}^{T}\nrm*{\ls_t}_{1} + \nrm*{\ls'_t}_{1}}-\frac{\eta_2}{2}\prn*{\sum_{t=1}^{T}\nrm*{\ls_t}_{1}^{2} + \nrm*{\ls'_t}_{1}^{2}}\biggr].
                        \end{align*}
                        }Splitting the supremum, this is upper bounded by
                        {\small
                        \begin{align*}
                          & 2\cdot\dtri*{\sup_{x_t\in\cX}\sup_{q_t\in\Delta(\cL)}\En_{\ls_t\sim{}q_t}\En_{\eps_t}}_{t=1}^{T}\biggl[\sup_{g\in\cG}\biggl[ 
                          \sum_{t=1}^{T}\eps_{t}H(\tri*{g(x_t), \ls_t})
                          -  \eta_3\sum_{t=1}^{T}\tri*{g(x_t), \ls_t}^{2}
                                         \biggr] -\frac{\eta_1}{2}\sum_{t=1}^{T}\nrm*{\ls_t}_{1} -\frac{\eta_2}{2}\sum_{t=1}^{T}\nrm*{\ls_t}_{1}^{2}\biggr] \\
                          & = 2\cdot\dtri*{\sup_{x_t\in\cX}\sup_{\ls_t\in\cL}\En_{\eps_t}}_{t=1}^{T}\biggl[\sup_{g\in\cG}\biggl[ 
                          \sum_{t=1}^{T}\eps_{t}H(\tri*{g(x_t), \ls_t})
                          -  \eta_3\sum_{t=1}^{T}\tri*{g(x_t), \ls_t}^{2}
                            \biggr] -\frac{\eta_1}{2}\sum_{t=1}^{T}\nrm*{\ls_t}_{1} -\frac{\eta_2}{2}\sum_{t=1}^{T}\nrm*{\ls_t}_{1}^{2}\biggr] \\
                          & = 2\cdot\sup_{\bx}\sup_{\bls}\En_{\eps}\biggl[\sup_{g\in\cG}\biggl[ 
                          \sum_{t=1}^{T}\eps_{t}H(\tri*{g(\bx_t(\eps)), \bls_t(\eps)})
                          -  \eta_3\sum_{t=1}^{T}\tri*{g(\bx_t(\eps)), \bls_t(\eps)}^{2}
              \biggr] -\frac{\eta_1}{2}\sum_{t=1}^{T}\nrm*{\bls_t(\eps)}_{1} -\frac{\eta_2}{2}\sum_{t=1}^{T}\nrm*{\bls_t(\eps)}_{1}^{2}\biggr].
            \end{align*}
          }The first equality is somewhat subtle, but holds because at time $T$, the expression is linear in $q_T$ so
          it is maximized at a point $\ls_T$, allowing us to work
          backwards to remove the $q_t$ distributions.

\subsubsection{Introducing a coarse cover}

We now break the process appearing in \pref{eq:symmetrized} into multiple terms, each of which will be handled by covering. Consider any fixed pair of trees $\bx$, $\bls$. Note that with the trees fixed \pref{eq:pre_sym} is at most
\begin{align*}
2\cdot\En_{\eps}\sup_{g\in\cG}\biggl[ 
  \sum_{t=1}^{T}\eps_{t}H(\tri*{g(\bx_t(\eps)), \bls_t(\eps)})
  -  \eta_3\sum_{t=1}^{T}\tri*{g(\bx_t(\eps)), \bls_t(\eps)}^{2}
  \biggr] -  \En_{\eps}B_{1}(\bls_{1:T}(\eps)).
\end{align*}
We will focus on the supremum for now. We begin by adapting a trick
from~\citet{RakSri15a} to introduce a coarse sequential cover at scale
$\beta$. Let $V'$ be a cover for $\cG$ on the tree $\bx$ with respect
to $L_{\infty}/\ls_{\infty}$ at scale $\beta/2$. Then the size of $V'$
is $\cN_{\infty,\infty}(\beta/2,\cG,\bx)$, and
\[\max_{g\in\cG}\max_{\eps\in\pmo^{T}}\min_{\bv'\in{}V'}\max_{t\in\brk*{T}}\nrm*{g(\bx_{t}(\eps))-\bv'_{t}(\eps)}_{\infty}\leq{}\beta/2.\] Recall
that since $g(x)\in\bbR^{K}_{+}$ for all $g\in\cG$, we may take each
$\bv'\in{}V'$ to have non-negative coordinates without loss of
generality. Likewise, it follows that we may take each $\bv'\in{}V'$
to have
$\nrm*{\bv'_{t}(\eps)}_{\infty}\leq{}\sup_{x\in\cX}\sup_{g\in\cG}\nrm*{g(x)}_{\infty}$
without loss of generality.

We construct a new $\beta$-cover $V^{1}$ from $V'$ by defining for
each tree $\bv'\in{}V'$ a new tree $\bv$ as follows:
\begin{align*}
  \forall{}\eps\in\pmo^{T}\;\forall{}t\in\brk*{T}\;\forall{}a\in\brk*{K}:\quad{} \bv_{t}(\eps)_{a} = \max\crl*{\bv'_{t}(\eps)_{a}-\beta/2,0}.
\end{align*}
It is easy to verify that for each time $t$ and path $\eps$ we have
$\nrm*{\bv_{t}(\eps)-\bv'_{t}(\eps)}_{\infty}\leq{}\beta/2$, so
$V^{1}$ is indeed a $\beta$-cover with respect to
$L_{\infty}/\ls_{\infty}$. More importantly, for each $g\in\cG$ and
path $\eps$, there exists a tree $\bv\in{}V'$ that is $\beta$-close in
the $L_{\infty}/\ls_{\infty}$ sense and has
$\bv_{t}(\eps)_{a}\leq{}g(\bx_{t}(\eps))_{a}$ coordinate-wise. We will
let $\bv^{1}[\eps,g]$ denote this tree, and it is constructed by
taking the $\beta/2$-close tree $\bv'$ promised by the definition of
$V'$, then performing the clipping operation above to get the
corresponding $\beta$-close element of $V^1$. The clipping operation
and $\beta/2$ closeness of $\bv'$ imply that for each time
$t\in\brk*{T}$ and coordinate $a\in\brk*{K}$,
\begin{align*}
  \bv^{1}_{t}[\eps,g]_{a} - g(\bx_{t}(\eps))_{a} &= \max\crl*{\bv'_{t}(\eps)_{a}-\beta/2,0} - g(\bx_{t}(\eps))_{a} \\
                                               &\leq{}  \max\crl*{\nrm*{\bv'_{t}(\eps)-g(\bx_{t}(\eps))}_{\infty} + g(\bx_{t}(\eps))_{a} -\beta/2,0} - g(\bx_{t}(\eps))_{a} \\
                                               &\leq{}  \max\crl*{g(\bx_{t}(\eps))_{a},0} - g(\bx_{t}(\eps))_{a} = 0.
\end{align*}
This establishes the desired ordering on coordinates. Returning to the
process at hand, we have
\begin{align*}
  &\En_{\eps}\sup_{g\in\cG}\biggl[ 
                          \sum_{t=1}^{T}\eps_{t}H\tri*{g(\bx_t(\eps)), \bls_t(\eps)})
                          -  \eta_3\sum_{t=1}^{T}\tri*{g(\bx_t(\eps)), \bls_t(\eps)}^{2}
                 \biggr].
\end{align*}
Now we add and subtract terms involving the covering element $\bv^{1}(\eps,g)$:
\begin{align*}
  =
    \En_{\eps}\sup_{g\in\cG}\biggl[&
                          \sum_{t=1}^{T}\eps_{t}H(\tri*{\bv_{t}^{1}[\eps,g], \bls_t(\eps)})
  -  \eta_3\sum_{t=1}^{T}\tri*{g(\bx_t(\eps)), \bls_t(\eps)}^{2} \\
  &+ \sum_{t=1}^{T}\eps_{t}H(\tri*{g(\bx_t(\eps)), \bls_t(\eps)}) - \eps_{t}H(\tri*{\bv_{t}^{1}[\eps,g], \bls_t(\eps)})
  \biggr].
\end{align*}
We now invoke the coordinate domination property of $\bv^{1}[\eps,g]$
described above. Observe that since $g(\bx_{t}(\eps))$,
$\bv_{t}^{1}[\eps,g]$, and $\bls_{t}(\eps)$ are all nonnegative
coordinate-wise, it holds that
$\tri*{\bv_{t}^{1}[\eps,g], \bls_t(\eps)}^{2}
\leq{}\tri*{g(\bx_t(\eps)), \bls_t(\eps)}^{2}$. Consequently, we can
replace the offset term (not involving $\eps_t$) with a similar term
involving $\bv_t^{1}[\eps,g]$
\begin{align*}
\leq{}
    \En_{\eps}\sup_{g\in\cG}\biggl[&
                          \sum_{t=1}^{T}\eps_{t}H(\tri*{\bv_{t}^{1}[\eps,f], \bls_t(\eps)})
  -  \eta_3\sum_{t=1}^{T}\tri*{\bv_{t}^{1}[\eps,g], \bls_t(\eps)}^{2} \\
  &+ \sum_{t=1}^{T}\eps_{t}H(\tri*{f(\bx_t(\eps)), \bls_t(\eps)}) - \eps_{t}H(\tri*{\bv_{t}^{1}[\eps,f], \bls_t(\eps)})
  \biggr].
\end{align*}
Splitting the supremum and gathering terms, this implies that $\cV$ is upper bounded by
\begin{align*}
    &\underbrace{\En_{\eps}\sup_{\bv^{1}\in{}V^{1}}\biggl[
                          \sum_{t=1}^{T}\eps_{t}H(\tri*{\bv_{t}^{1}(\eps), \bls_t(\eps)})
  -  \eta_3\sum_{t=1}^{T}\tri*{\bv_{t}^{1}(\eps), \bls_t(\eps)}^{2}\biggr]}_{(\star)} \\
 &+ \underbrace{\En_{\eps}\sup_{g\in\cG}\biggl[\sum_{t=1}^{T}\eps_{t}H(\tri*{g(\bx_t(\eps)), \bls_t(\eps)}) - \eps_{t}H(\tri*{\bv_{t}^{1}[\eps,g], \bls_t(\eps)})
  \biggr]-\En_{\eps}B_{1}(\bls_{1:T}(\eps))}_{(\star\star)}.
\end{align*}

\subsubsection{Bounding $(\star)$}
We appeal to \pref{lem:offset_finite_R} with a class of real-valued trees $U\defeq\crl*{\eps\mapsto\prn*{\tri*{\bv_{t}^{1}(\eps),\bls_{t}(\eps)}}_{t\leq{}T}\mid{} \bv^{1}\in{}V^{1}}$. The class $U$ has range contained in $\brk*{-1,+1}$, since $\abs*{\tri*{\bv_{t}^{1}(\eps),\bls_{t}(\eps)}}\leq{}\nrm*{\bv_{t}^1(\eps)}_{\infty}\nrm*{\bls_{t}(\eps)}_{1}\leq{}1$, where these norm bounds are by assumption on $\cG$ and $\cL$. Recall that $H(x) = x - \eta_3x^2$. We therefore conclude that 
\begin{align*}
(\star) &= \En_{\eps}\sup_{\bv^{1}\in{}V^{1}}\biggl[
                          \sum_{t=1}^{T}\eps_{t}H(\tri*{\bv_{t}^{1}(\eps), \bls_t(\eps)})
  -  \eta_3\sum_{t=1}^{T}\tri*{\bv_{t}^{1}(\eps), \bls_t(\eps)}^{2}\biggr]\\
  &\leq{} 2\frac{1+\eta_{3}^{2}}{\eta_{3}}\log\abs*{V^{1}} 
  = 2\frac{1+\eta_{3}^{2}}{\eta_{3}}\log\cN_{\infty,\infty}(\beta/2, \cG, \bx).
\end{align*}

\subsubsection{Bounding $(\star\star)$}
Fix $\alpha>0$ and let $N=\floor*{\log(\beta/\alpha)}-1$. For each $i\geq{}1$ define $\veps_{i}=\beta{}e^{-(i-1)}$, and for each $i>1$ let $V^{i}$ be a sequential cover of $\cG$ on $\bx$ at scale $\veps_{i}$ with respect to $L_{\infty}/\ls_{\infty}$ (keeping in mind that $V^{1}$ is defined as in the preceding section). For a given path $\eps\in\pmo^{T}$ and $g\in\cG$, let
$\bv^{i}[\eps,g]$ denote the $\veps_{i}$-close element of $V^{i}$. Below, we will only evaluate $H(x) = x - \eta_3x^2$ over the domain $\brk*{-1,+1}$; it is $(1+2\eta_3)$-Lipschitz over this domain. Then the leading term of $(\star\star)$ is equal to 
\begin{align*}
&\En_{\eps}\sup_{g\in\cG}\biggl[\sum_{t=1}^{T}\eps_{t}\prn*{H\prn*{\tri*{g(\bx_t(\eps)), \bls_t(\eps)}} - H\prn*{\tri*{\bv_{t}^{1}[\eps,g], \bls_t(\eps)}}}
  \biggr].
\end{align*}
  Introducing the covering elements defined above to this expression, we have the equality
  \begin{align*}
    = & \En_{\eps}\sup_{g\in\cG}\biggl[\sum_{t=1}^{T}\eps_{t}\prn*{H\prn*{\tri*{g(\bx_t(\eps)), \bls_t(\eps)}} - H\prn*{\tri*{\bv_{t}^{N}[\eps,g], \bls_t(\eps)}}}\\
    & \qquad \qquad + \sum_{i=1}^{N-1}\sum_{t=1}^{T}\eps_{t}\prn*{H\prn*{\tri*{\bv_{t}^{i+1}[\eps,g], \bls_t(\eps)}} - H\prn*{\tri*{\bv_{t}^{i}[\eps,g], \bls_t(\eps)}}}
      \biggr]\\
  % \end{align*}
  % \begin{align*}
    \leq{}&\underbrace{\En_{\eps}\sup_{g\in\cG}\biggl[\sum_{t=1}^{T}\eps_{t}\prn*{H\prn*{\tri*{g(\bx_t(\eps)), \bls_t(\eps)}} - H\prn*{\tri*{\bv_{t}^{N}[\eps,g], \bls_t(\eps)}}}\biggr]}_{\triangleq C_{N}} \\
    & \qquad \qquad + \sum_{i=1}^{N-1}\underbrace{\En_{\eps}\sup_{g\in\cG}\biggl[\sum_{t=1}^{T}\eps_{t}\prn*{H\prn*{\tri*{\bv_{t}^{i+1}[\eps,g], \bls_t(\eps)}} - H\prn*{\tri*{\bv_{t}^{i}[\eps,g], \bls_t(\eps)}}}
  \biggr]}_{\triangleq C_{i}}.
\end{align*}
\subsubsection{Bounding $C_{N}$} We first bound $C_{N}$ in terms of one of the terms appearing in $B_1$.
\begin{align*}
  C_N=&\En_{\eps}\sup_{g\in\cG}\biggl[\sum_{t=1}^{T}\eps_{t}\prn*{H\prn*{\tri*{g(\bx_t(\eps)), \bls_t(\eps)}} - H\prn*{\tri*{\bv_{t}^{N}[\eps,g], \bls_t(\eps)}}}\biggr] \\
  & \leq{} \En_{\eps}\biggl[\sum_{t=1}^{T}\sup_{g\in\cG}\abs*{H\prn*{\tri*{g(\bx_t(\eps)), \bls_t(\eps)}} - H\prn*{\tri*{\bv_{t}^{N}[\eps,g], \bls_t(\eps)}}}\biggr] \\
  & \leq{} (1+2\eta_3)\En_{\eps}\biggl[\sum_{t=1}^{T}\sup_{g\in\cG}\abs*{\tri*{g(\bx_t(\eps)), \bls_t(\eps)} - \tri*{\bv_{t}^{N}[\eps,g], \bls_t(\eps)}}\biggr] \\
  & \leq{} (1+2\eta_3)\En_{\eps}\biggl[\sum_{t=1}^{T}\nrm*{\bls_{t}(\eps)}_{1}\sup_{g\in\cG}\nrm*{g(\bx_t(\eps))-\bv_{t}^{N}[\eps,g]}_{\infty}\biggr] \\
  & \leq{} (1+2\eta_3)\max_{\eps}\sup_{g\in\cG}\max_{t\in\brk*{T}}\nrm*{g(\bx_t(\eps))-\bv_{t}^{N}[\eps,g]}_{\infty}\cdot\En_{\eps}\biggl[\sum_{t=1}^{T}\nrm*{\bls_{t}(\eps)}_{1}\biggr] \\
      & \leq{} (1+2\eta_3)e^{2}\alpha\cdot\En_{\eps}\biggl[\sum_{t=1}^{T}\nrm*{\bls_{t}(\eps)}_{1}\biggr].
\end{align*}
The first inequality uses that $\eps_t \in \pmo$, while the second
uses the Lipschitzness of $H$ over $\brk*{-1,+1}$.  The third and
fourth are both applications of \Holder's inequality, first to the
$\ls_1/\ls_\infty$ dual pairing, and then to for the distributions
over $L_1/L_\infty$. Finally, the definition of the covering element
$\bv_{t}^{N}$---in particular, that it is an
$L_{\infty}/\ls_{\infty}$-cover---implies that the supremum term is
bounded by $\veps_N\leq e^2\cdot\alpha$, which yields the final bound. 
\subsubsection{Bounding $C_{i}$}
Our goal is to bound
\begin{align*}
  C_i = \En_{\eps}\sup_{g\in\cG}\biggl[\sum_{t=1}^{T}\eps_{t}\prn*{H\prn*{\tri*{\bv_{t}^{i+1}[\eps,g], \bls_t(\eps)}} - H\prn*{\tri*{\bv_{t}^{i}[\eps,g], \bls_t(\eps)}}}
  \biggr].
\end{align*}
We define a class $W$ of real-valued trees as follows. Let $1\leq{}a\leq{}\abs*{V^{i}}$ and $1\leq{}b\leq{}\abs*{V^{i+1}}$, and fix an arbitrary ordering $\bv^{a}\in{}V^{i}$ and $\bv^{b}\in{}V^{i+1}$ of the elements of $V^{i}/V^{i+1}$. For each pair $(a,b)$ define a tree $\bw^{(a,b)}$ via 
\begin{align*}
  \bw_{t}^{(a,b)}(\eps) = \left\{
    \begin{array}{ll}
      H\prn*{\tri*{\bv_{t}^{b}(\eps),\bls_{t}(\eps)}} - H\prn*{\tri*{\bv_{t}^{a}(\eps),\bls_{t}(\eps)}},&\quad
                                                                                               \exists{}g\in\cG\text{ s.t. } \bv^{a}=\bv[\eps,g]^{i}, \bv^{b}=\bv[\eps,g]^{i+1},\\
      0,&\quad\text{otherwise.}
    \end{array}
    \right.
\end{align*}
Then $C_{i}$ is bounded by
\begin{align*}
\En_{\eps}\sup_{\bw\in{}W}\sum_{t=1}^{T}\eps_{t}\bw_{t}(\eps).
\end{align*}
Then \pref{lem:offset_finite_simple} implies that for any fixed $\eta>0$,
\begin{align*}
\En_{\eps}\sup_{\bw\in{}W}\brk*{\sum_{t=1}^{T}\eps_{t}\bw_{t}(\eps) - \eta\bw_{t}^{2}(\eps)} \leq{} \frac{\log\abs*{W}}{2\eta}.
\end{align*}
Rearranging and applying subadditivity of the supremum, this implies
\begin{align*}
\En_{\eps}\sup_{\bw\in{}W}\sum_{t=1}^{T}\eps_{t}\bw_{t}(\eps) \leq{} \eta\cdot\En_{\eps}\sup_{\bw\in{}W}\sum_{t=1}^{T}\bw_{t}^{2}(\eps) + \frac{\log\abs*{W}}{2\eta}.
\end{align*}
Optimizing over $\eta$ (which is admissible because the statement above is a deterministic inequality) leads to a further bound of
\begin{align*}
\En_{\eps}\sup_{\bw\in{}W}\sum_{t=1}^{T}\eps_{t}\bw_{t}(\eps) \leq{} \sqrt{2\En_{\eps}\sup_{\bw\in{}W}\sum_{t=1}^{T}\bw_{t}^{2}(\eps)\cdot\log\abs*{W}}.
\end{align*}
We proceed to bound each term in the square root. For the logarithmic term, by construction we have $\abs*{W}\leq{}\abs*{V^{i}}\abs*{V^{i+1}}\leq{}\abs*{V^{i+1}}^{2} = \cN_{\infty,\infty}(\veps_{i+1}, \cG, \bx)^{2}$.

For the variance, let $\bw^{(a,b)}\in{}W$ and the path $\eps$ be fixed. There are two cases: Either $\bw(\eps)=\mb{0}$, or there exists $g\in\cG$, such that $\bv^{a}=\bv[\eps,g]^{i}$ and $\bv^{b}=\bv[\eps,g]^{i+1}$. The former case is trivial while for the latter, in a similar way to the bound for $C_N$, we get
\begin{align*}
  \sum_{t=1}^{T}\bw_{t}^{(a,b)}(\eps)^{2} &= \sum_{t=1}^{T}\prn*{H\prn*{\tri*{\bv_{t}^{i+1}[\eps,g], \bls_t(\eps)}} - H\prn*{\tri*{\bv_{t}^{i}[\eps,g], \bls_t(\eps)}}}^{2}\\
                                          &\leq{} (1+2\eta_{3})^{2}\sum_{t=1}^{T}\prn*{\tri*{\bv_{t}^{i+1}[\eps,g], \bls_t(\eps)} - \tri*{\bv_{t}^{i}[\eps,g], \bls_t(\eps)}}^{2} \\
                                          &\leq{} (1+2\eta_{3})^{2}\sum_{t=1}^{T}\nrm*{\bls_{t}(\eps)}_{1}^{2}\nrm*{\bv_{t}^{i+1}[\eps,g] - \bv_{t}^{i}[\eps,g]}_{\infty}^{2} \\
                                          &\leq{} (1+2\eta_{3})^{2}\max_{\eps'}\max_{t\in\brk*{T}}\nrm*{\bv_{t}^{i+1}[\eps',g] - \bv_{t}^{i}[\eps',g]}_{\infty}^{2}\cdot\sum_{t=1}^{T}\nrm*{\bls_{t}(\eps)}_{1}^{2}.
\end{align*}
Where we have used Lipschitzness of $H$ in the first inequality and \Holder's inequality in the second and third. 

Finally, using the $L_{\infty}/\ls_{\infty}$ cover property of $\bv^{i}[\eps,g]$ and $\bv^{i+1}[\eps,g]$ and the triangle inequality, we have
\begin{align*}
  & \max_{\eps}\max_{t\in\brk*{T}}\nrm*{\bv_{t}^{i+1}[\eps,g] - \bv_{t}^{i}[\eps,g]}_{\infty} \\
  &\leq \max_{\eps}\max_{t\in\brk*{T}}\nrm*{\bv_{t}^{i+1}[\eps,g] - g(\bx_{t}(\eps))}_{\infty} +   \max_{\eps}\max_{t\in\brk*{T}}\nrm*{g(\bx_{t}(\eps)) - \bv_{t}^{i}[\eps,g]}_{\infty} \\
  &\leq \veps_{i} + \veps_{i+1} \leq{}2\veps_{i}.
\end{align*}
We have just shown that for every sequence $\eps$ and every
$\bw^{(a,b)}\in{}W$,
$\sum_{t=1}^{T}\bw_{t}^{(a,b)}(\eps)^{2}\leq{}
4(1+2\eta_{3})^{2}\veps_{i}^{2}\cdot\sum_{t=1}^{T}\nrm*{\bls_{t}(\eps)}_{1}^{2}$. It
follows that
\begin{align*}
\En_{\eps}\sup_{\bw\in{}W}\sum_{t=1}^{T}\bw_{t}(\eps)^{2}\leq{} 4(1+2\eta_{3})^{2}\veps_{i}^{2}\cdot\En_{\eps}\sum_{t=1}^{T}\nrm*{\bls_{t}(\eps)}_{1}^{2}.
\end{align*}
Plugging this bound back into the main inequality, we have shown
\begin{align*}
  \En_{\eps}\sup_{\bw\in{}W}\sum_{t=1}^{T}\eps_{t}\bw_{t}(\eps) \leq{} 4e(1+2\eta_{3})\veps_{i+1}\sqrt{\En_{\eps}\sum_{t=1}^{T}\nrm*{\bls_{t}(\eps)}_{1}^{2}\cdot\log\cN_{\infty,\infty}(\veps_{i+1}, \cG, \bx)}.
\end{align*}

\subsubsection{Final bound on $(\star\star)$}
Collecting terms, we have shown that
{\small
\begin{align}
\notag&(\star\star) \\&\leq{} (1+2\eta_3)e^2\alpha\cdot\En_{\eps}\biggl[\sum_{t=1}^{T}\nrm*{\bls_{t}(\eps)}_{1}\biggr] + 4e(1+2\eta_{3})\sqrt{\En_{\eps}\sum_{t=1}^{T}\nrm*{\bls_{t}(\eps)}_{1}^{2}}\sum_{i=1}^{N-1}\veps_{i+1}\sqrt{\log\cN_{\infty,\infty}(\veps_{i+1}, \cG, \bx)} - \En_{\eps}B_{1}(\bls_{1:T}(\eps)).\label{eq:stst_inter}
\end{align}
}Following the standard Dudley chaining proof, we have 
\begin{align*}
  \sum_{i=1}^{N-1}\veps_{i+1}\sqrt{\log\cN_{\infty,\infty}(\veps_{i+1}, \cG, \bx)} 
  &\leq{}\sum_{i=1}^{N}\veps_{i}\sqrt{\log\cN_{\infty,\infty}(\veps_{i}, \cG, \bx)} \leq{} 2\sum_{i=1}^{N}(\veps_{i}-\veps_{i+1})\sqrt{\log\cN_{\infty,\infty}(\veps_i, \cG, \bx)} 
  \\
  & \leq{} 2\int_{\veps_{N+1}}^{\beta}\sqrt{\log\cN_{\infty,\infty}(\veps, \cG, \bx)}d\veps
                                                                                   \leq{} 2\int_{\alpha}^{\beta}\sqrt{\log\cN_{\infty,\infty}(\veps, \cG, \bx)}d\veps \\
            &                                                                       \leq{} 2\int_{\alpha}^{\beta}\sqrt{\log\cN_{\infty,\infty}(\veps, \cG, T)}d\veps.
\end{align*}
Where we are using the definition of $N$, which implies that $\alpha\leq{}\veps_{N+1}$. 
%\akshay{I think typo here, first inequality should have $\Ncal(\varepsilon_{i+1})$ on the right hand side?}

Now recall the definition of $B_1(\bls_{1:T}(\eps))$:
\begin{align*}
  B_1(\bls_{1:T}(\eps)) = \eta_1\sum_{t=1}^T\nrm*{\bls_t(\eps)}_{1} + \eta_2\sum_{t=1}^T\nrm*{\bls_t(\eps)}_{1}^{2}
\end{align*}

Taking $\eta_{1}\geq{}(1+2\eta_{3})e^2 \alpha$, the first term in $B_1$ cancels out the first term in \pref{eq:stst_inter}, leaving us with
% this is upper bounded by
\begin{align*}
  (\star\star)& \leq 8e(1+2\eta_{3})\sqrt{\En_{\eps}\sum_{t=1}^{T}\nrm*{\bls_{t}(\eps)}_{1}^{2}}\int_{\alpha}^{\beta}\sqrt{\log\cN_{\infty,\infty}(\veps, \cG, T)}d\veps - \eta_{2}\En_{\eps}\sum_{t=1}^{T}\nrm*{\bls_{t}(\eps)}_{1}^{2}\\
% The AM-GM inequality implies that for any $\eta_{4}>0$, this is bounded by
% \[
  & \leq 8e(1+2\eta_{3})\prn*{\frac{\eta_{4}}{4}\En_{\eps}\sum_{t=1}^{T}\nrm*{\bls_{t}(\eps)}_{1}^{2} + \frac{1}{\eta_4}}\int_{\alpha}^{\beta}\sqrt{\log\cN_{\infty,\infty}(\veps, \cG, T)}d\veps - \eta_{2}\En_{\eps}\sum_{t=1}^{T}\nrm*{\bls_{t}(\eps)}_{1}^{2}.
\end{align*}% \]
Where the last step applies for any $\eta_4>0$ by the AM-GM inequality. For any $\eta_{2}\geq{}8e(1+2\eta_{3})\eta_4\cdot\int_{\alpha}^{\beta}\sqrt{\log\cN_{\infty,\infty}(\veps, \cG, T)}d\veps$, the first and third terms cancel, leaving us with an upper bound of
\begin{align*}
  (\star\star) \leq \frac{8e(1+2\eta_{3})}{\eta_4}\int_{\alpha}^{\beta}\sqrt{\log\cN_{\infty,\infty}(\veps, \cG, T)}d\veps.
\end{align*}
This term does not depend on the trees $\bx$ or $\bls$, so we are done with $(\star\star)$.

\subsubsection{Final bound}
Under the assumptions on $\eta_1,\eta_2,\eta_3,\eta_4,\alpha$, and $\beta$, the bounds on $(\star)$ and $(\star\star)$ we have established imply
\begin{align*}
\cV \leq{} 2\frac{1+\eta_{3}^{2}}{\eta_{3}}\log\cN_{\infty,\infty}(\beta/2, \cG, T) + \frac{8e(1+2\eta_{3})}{\eta_4}\int_{\alpha}^{\beta}\sqrt{\log\cN_{\infty,\infty}(\veps, \cG, T)}d\veps.
\end{align*}
The definition of $\cV$ implies that there exists an algorithm with regret bounded by $\cV + B(p_{1:T}, \ls_{1:T})$ on every sequence. The final regret inequality is
  \begin{align*}
    &\sum_{t=1}^{T}\En_{s\sim{}p_t}\tri*{s, \ls_t} -\inf_{g\in\cG}\sum_{t=1}^{T}\tri*{g(x_t), \ls_t} \\
    &\leq{} 
      2\eta_{3}\sum_{t=1}^{T}\En_{s\sim{}p_t}\tri*{s, \ls_t}^{2} +  2\frac{1+\eta_{3}^{2}}{\eta_{3}}\log\cN_{\infty,\infty}(\beta/2, \cG, T) \\
    &~~~~~~~~~~~~+ 8e(1+2\eta_{3})\prn*{\frac{\eta_{4}}{4}\sum_{t=1}^{T}\nrm*{\ls_{t}}_{1}^{2} + \frac{1}{\eta_4}}\int_{\alpha}^{\beta}\sqrt{\log\cN_{\infty,\infty}(\veps, \cG, T)}d\veps + (1+2\eta_3)e^2\alpha\sum_{t=1}^{T}\nrm*{\ls_{t}}_{1}.
  \end{align*}
  To obtain the bound in the theorem statement, we rebind
  $\eta = \eta_3, \lambda=\eta_4$ and use the assumption
  $\eta \leq 1$.

  \subsection{Proofs for remaining results}

Our bandit results require a generalization of \pref{thm:chaining_v2} to the case where losses and the class $\cG$ may not be bounded by $1$.
\begin{corollary}
\label{corr:full_info_scaled}
  Suppose we are in the setting of \pref{thm:chaining_v2}, but with the bounds $\nrm*{\ls}_{1}\leq{}R$ for all $\ls\in\cL$ and $\nrm*{s}_{\infty}\leq{}B$ for all $s\in\cS$.
  For any constants $\eta\in(0,1]$, $\lambda>0$, and $\beta>\alpha>0$, there exists an algorithm making predictions in $\cS$ that attains a regret guarantee of
  {\small
  \begin{align*}
    \sum_{t=1}^{T}\En_{s_t\sim{}p_t}\tri*{
      s_t, \ls_t} -\inf_{g\in\cG}\sum_{t=1}^{T}\tri*{g(x_t), \ls_t} &\leq{} 
      \frac{2\eta}{RB}\sum_{t=1}^{T}\En_{s_t\sim{}p_t}\tri*{s_t, \ls_t}^{2} +  \frac{4RB}{\eta}\log\cN_{\infty,\infty}(\beta/2, \cG, T) + 3e^2\alpha\sum_{t=1}^{T}\nrm*{\ls_{t}}_{1}\\
   &~~~~~ + 24e\prn*{\frac{\lambda}{4R}\sum_{t=1}^{T}\nrm*{\ls_{t}}_{1}^{2} + \frac{R}{\lambda}}\int_{\alpha}^{\beta}\sqrt{\log\cN_{\infty,\infty}(\veps, \cG, T)}d\veps.
  \end{align*}
}Furthermore, if upper bounds $\sum_{t=1}^{T}\nrm*{\ls_{t}}_{1}^{2}\leq{}V$ and $\sum_{t=1}^{T}\En_{s_t\sim{}p_t}\tri*{s_t,\ls_t}^{2}\leq{}\wt{V}$ are known in advance, $\eta$ and $\lambda$ can be selected to guarantee regret
    {\small
  \begin{align*}
    &\sum_{t=1}^{T}\En_{s_t\sim{}p_t}\tri*{s_t, \ls_t} -\inf_{g\in\cG}\sum_{t=1}^{T}\tri*{g(x_t), \ls_t} \\
    &\leq{} 
      8\sqrt{\wt{V}\cdot\log\cN_{\infty,\infty}(\beta/2, \cT, T)} + 8RB\log\cN_{\infty,\infty}(\beta/2, \cG, T) \\
    &~~~~~~~~~~~~+ 24e\sqrt{V}\int_{\alpha}^{\beta}\sqrt{\log\cN_{\infty,\infty}(\veps, \cG, T)}d\veps + 3e\alpha\sum_{t=1}^{T}\nrm*{\ls_{t}}_{1}.
  \end{align*}
  }
  \end{corollary}
  \begin{proof}[\pfref{corr:full_info_scaled}]
	Apply \pref{thm:chaining_v2} with losses $\ls_{t}/R$ and class $\cG/B$. The preconditions of the theorem are satisified, so it implies existence of an algorithm making predictions in $\cS/B$ with regret bound
  {\small
  \begin{align*}
    \frac{1}{R}\sum_{t=1}^{T}\En_{s_t\sim{}p_t}\tri*{
      s_t, \ls_t} -\frac{1}{R}\inf_{g'\in\cG/B}\sum_{t=1}^{T}\tri*{g'(x_t), \ls_t} &\leq{} 
      \frac{2\eta}{R^{2}}\sum_{t=1}^{T}\En_{s_t\sim{}p_t}\tri*{s_t, \ls_t}^{2} +  \frac{4}{\eta}\log\cN_{\infty,\infty}(\beta/2, \cG/B, T) + \frac{3e^2\alpha}{R}\sum_{t=1}^{T}\nrm*{\ls_{t}}_{1}\\
   &~~~~~ + 24e\prn*{\frac{\lambda}{4R^{2}}\sum_{t=1}^{T}\nrm*{\ls_{t}}_{1}^{2} + \frac{1}{\lambda}}\int_{\alpha}^{\beta}\sqrt{\log\cN_{\infty,\infty}(\veps, \cG/B, T)}d\veps.
  \end{align*}
  }Rescaling both sides by $BR$ and letting $\hat{s}_t=s_t\cdot{}B$ (so $\hat{s}_t\in\cS$), this implies
    {\small
  \begin{align*}
\sum_{t=1}^{T}\En_{\hat{s}_t\sim{}p_t}\tri*{
      \hat{s}_t, \ls_t} -\inf_{g\in\cG}\sum_{t=1}^{T}\tri*{g(x_t), \ls_t} &\leq{} 
      \frac{2\eta}{RB}\sum_{t=1}^{T}\En_{\hat{s}_t\sim{}p_t}\tri*{\hat{s}_t, \ls_t}^{2} +  \frac{4RB}{\eta}\log\cN_{\infty,\infty}(\beta/2, \cG/B, T) + 3e^2\alpha{}B\sum_{t=1}^{T}\nrm*{\ls_{t}}_{1}\\
   &~~~~~ + 24e\prn*{\frac{\lambda{}B}{4R}\sum_{t=1}^{T}\nrm*{\ls_{t}}_{1}^{2} + \frac{RB}{\lambda}}\int_{\alpha}^{\beta}\sqrt{\log\cN_{\infty,\infty}(\veps, \cG/B, T)}d\veps.\\
&\leq{} 
      \frac{2\eta}{RB}\sum_{t=1}^{T}\En_{\hat{s}_t\sim{}p_t}\tri*{\hat{s}_t, \ls_t}^{2} +  \frac{4RB}{\eta}\log\cN_{\infty,\infty}(\beta{}B/2, \cG, T) + 3e^2\alpha{}B\sum_{t=1}^{T}\nrm*{\ls_{t}}_{1}\\
   &~~~~~ + 24e\prn*{\frac{\lambda{}B}{4R}\sum_{t=1}^{T}\nrm*{\ls_{t}}_{1}^{2} + \frac{RB}{\lambda}}\int_{\alpha}^{\beta}\sqrt{\log\cN_{\infty,\infty}(\veps{}B, \cG, T)}d\veps.
  \end{align*}}
  Using a change of variables in the Dudley integral, we get
        {\small
  \begin{align*}
                                                                            & \leq{}
      \frac{2\eta}{RB}\sum_{t=1}^{T}\En_{\hat{s}_t\sim{}p_t}\tri*{\hat{s}_t, \ls_t}^{2} +  \frac{4RB}{\eta}\log\cN_{\infty,\infty}(\beta{}B/2, \cG, T) + 3e^2\alpha{}B\sum_{t=1}^{T}\nrm*{\ls_{t}}_{1}\\
   &~~~~~ + 24e\prn*{\frac{\lambda{}}{4R}\sum_{t=1}^{T}\nrm*{\ls_{t}}_{1}^{2} + \frac{R}{\lambda}}\int_{\alpha{}B}^{\beta{}B}\sqrt{\log\cN_{\infty,\infty}(\veps{}, \cG, T)}d\veps.
  \end{align*}
  }The final result follows by rebinding $\alpha'=\alpha{}B$ and $\beta'=\beta{}B$.

  For the second claim, apply the upper bounds to obtain
  \begin{align*}
    &\frac{2\eta}{RB}\tilde{V} + \frac{4RB}{\eta}\log\cN_{\infty,\infty}(\beta/2,\cG,T) + 3e^2\alpha B\sum_{t=1}^T\nrm*{\ls_t}_1\\
    & + 24e\left(\frac{\lambda}{4R}V + \frac{R}{\lambda}\right)\int_{\alpha}^\beta\sqrt{\log\cN_{\infty,\infty}(\epsilon,\cG,T)}d\epsilon.
  \end{align*}
  Now set $\lambda = 2R/\sqrt{V}$ and
  $\eta = \sqrt{2}RB
  \sqrt{\log\cN_{\infty,\infty}(\beta/2,\cG,T)/\tilde{V}} \wedge 1$ to
  obtain the claimed bound. Note that the range term arises from the
  constraint that $\eta \in (0,1]$.
  \end{proof}

\begin{proof}[\pfref{thm:chaining_ramp}]Recall that we use the reduction:
\begin{itemize}[leftmargin=*]
\item Initialize full information algorithm whose existence is guaranteed by \pref{thm:chaining_v2} with $\cG=\ramp\circ\cF$:
\item For time $t=1,\ldots,T$:
  \begin{itemize}
  \item Receive $x_{t}$ and define $P_{t}(a) \triangleq \En_{s_{t}\sim{}p_t}\frac{s_{t}(a)}{\sum_{a'\in\brk*{K}}s_t(a')}$, where $p_t$ is the output of the full information algorithm at time $t$.
  \item Sample action $a_{t}\sim{}P_{t}^{\mu}$ and feed importance-weighted loss $\hat{\ls}_{t}(a)=\one\crl*{a_t=a}\ls_{t}(a)/P_{t}^{\mu}(a)$ into the full information algorithm.
  \end{itemize}
\end{itemize}
With this setup, \pref{corr:full_info_scaled} guarantees that the following deterministic regret inequality holds for every sequence of outcomes (i.e. for every sequence $a_1,\ldots,a_T$ sampled by the algorithm):
  {\small
  \begin{align*}
    &\sum_{t=1}^{T}\En_{s_t\sim{}p_t}\tri*{
      s_t, \hat{\ls}_t} -\inf_{f\in\cF}\sum_{t=1}^{T}\tri*{\ramp(f(x_t)), \hat{\ls}_t} \\
      &\leq{} 
      \frac{2\eta}{RB}\sum_{t=1}^{T}\En_{s_t\sim{}p_t}\tri*{s_t, \hat{\ls}_t}^{2} +  \frac{4RB}{\eta}\log\cN_{\infty,\infty}(\beta/2, \ramp\circ\cF, T) + 3e^2\alpha\sum_{t=1}^{T}\nrm*{\hat{\ls}_{t}}_{1}\\
   &~~~~~ + 24e\prn*{\frac{\lambda}{4R}\sum_{t=1}^{T}\nrm*{\hat{\ls}_{t}}_{1}^{2} + \frac{R}{\lambda}}\int_{\alpha}^{\beta}\sqrt{\log\cN_{\infty,\infty}(\veps, \ramp\circ\cF, T)}d\veps,
  \end{align*}
  }where the boundedness of the ramp loss implies $B\leq{}1$ and the smoothing factor $\mu$ in $P^{\mu}_t$ guarantees $R\leq{}1/\mu$. Taking expectation over the draw of $a_1,\ldots,a_{T}$, for any fixed $f\in\cF$ we obtain the inequality 
  {\small
  \begin{align*}
    &\En\brk*{\sum_{t=1}^{T}\En_{s_t\sim{}p_t}\tri*{
      s_t, \hat{\ls}_t} -\sum_{t=1}^{T}\tri*{\ramp(f(x_t)), \hat{\ls}_t}} \\
      &\leq{} 
      \En\biggl[\frac{2\eta}{1/\mu}\sum_{t=1}^{T}\En\brk*{\En_{s_t\sim{}p_t}\tri*{s_t, \hat{\ls}_t}^{2}\mid{}\cJ_t} +  \frac{4}{\eta\mu}\log\cN_{\infty,\infty}(\beta/2, \ramp\circ\cF, T) + 3e^2\alpha\sum_{t=1}^{T}\En\brk*{\nrm*{\hat{\ls}_{t}}_{1}\mid{}\cJ_t}\\
   &~~~~~~~~ + 24e\prn*{\frac{\lambda}{4/\mu}\sum_{t=1}^{T}\En\brk*{\nrm*{\hat{\ls}_{t}}_{1}^{2}\mid\cJ_t} + \frac{1}{\lambda\mu}}\int_{\alpha}^{\beta}\sqrt{\log\cN_{\infty,\infty}(\veps, \ramp\circ\cF, T)}d\veps\biggr],
  \end{align*}
}where the filtration $\cJ_t$ is defined as in
\pref{lem:ips_variance}.  Using that the importance weighted
losses are unbiased, we have that the left-hand side is equal to
\begin{align*}
 \En\brk*{\sum_{t=1}^{T}\En_{s_t\sim{}p_t}\tri*{
      s_t, \ls_t} -\sum_{t=1}^{T}\tri*{\ramp(f(x_t)), \ls_t}}.
\end{align*}
We also have the following three properties, where the first two use that $\hat{\ls}_t$ is $1$-sparse, and the last follows from \pref{lem:ips_variance}:
\begin{enumerate}
  \item $\En\brk*{\nrm*{\hat{\ls}_{t}}_{1}\mid{}\cJ_t} = \sum_{a\in\brk*{K}}P_{t}^{\mu}(a)\hat{\ls}_{t}(a) = \sum_{a\in\brk*{K}}\ls_{t}(a) \leq{} K$.\\
\item $\En\brk*{\nrm*{\hat{\ls}_{t}}_{1}^{2}\mid\cJ_t} = \sum_{a\in\brk*{K}}P_{t}^{\mu}(a)\hat{\ls}_{t}^{2}(a) 
= \sum_{a\in\brk*{K}}\frac{\ls_t(a)}{P_{t}^{\mu}(a)} \leq{} \frac{K}{\mu}$.\\
  \item $\En\brk*{\En_{s_t\sim{}p_t}\tri*{s_t,\hat{\ls}_{t}}^{2}\mid{}\cJ_t}\leq{}K^{2}$.
\end{enumerate}

Together, these facts yield the bound
  {\small
  \begin{align*}
    \En\brk*{\sum_{t=1}^{T}\En_{s_t\sim{}p_t}\tri*{
      s_t, \ls_t} -\sum_{t=1}^{T}\tri*{\ramp(f(x_t)), \ls_t}}
      &\leq{} 
\frac{2\eta}{1/\mu}K^{2}T +  \frac{4}{\eta\mu}\log\cN_{\infty,\infty}(\beta/2, \ramp\circ\cF, T) + 3e^2\alpha{}KT\\
   &~~~~~~~~ + 24e\prn*{\frac{\lambda{}KT}{4} + \frac{1}{\lambda\mu}}\int_{\alpha}^{\beta}\sqrt{\log\cN_{\infty,\infty}(\veps, \ramp\circ\cF, T)}d\veps.
  \end{align*}
  }Optimizing $\eta$ and $\lambda$ (as in the proof of the second claim of~\pref{corr:full_info_scaled}) leads to a bound of 
    {\small
  \begin{align*}
    &\En\brk*{\sum_{t=1}^{T}\En_{s_t\sim{}p_t}\tri*{
      s_t, \ls_t} -\sum_{t=1}^{T}\tri*{\ramp(f(x_t)), \ls_t}} \\
      &\leq{} 
4\sqrt{2K^{2}T\log\cN_{\infty,\infty}(\beta/2, \ramp\circ\cF, T)} +  \frac{8}{\mu}\log\cN_{\infty,\infty}(\beta/2, \ramp\circ\cF, T) \\
   &~~~~~~~~ + 3e^2\alpha{}KT + 24e\sqrt{\frac{KT}{\mu}}\int_{\alpha}^{\beta}\sqrt{\log\cN_{\infty,\infty}(\veps, \ramp\circ\cF, T)}d\veps.
  \end{align*}
}Since $\ramp$ is $\frac{1}{\gamma}$-Lipschitz with respect to the $\ls_{\infty}$ norm (as a coordinate-wise mapping from $\bbR^{K}$ to $\bbR^{K}$), we can upper bound in terms of the covering numbers for the original class:
    {\small
  \begin{align*}
    \En\brk*{\sum_{t=1}^{T}\En_{s_t\sim{}p_t}\tri*{
      s_t, \ls_t} -\sum_{t=1}^{T}\tri*{\ramp(f(x_t)), \ls_t}}
      &\leq{} 
        4\sqrt{2K^{2}T\log\cN_{\infty,\infty}(\gamma\beta/2, \cF, T)} +  \frac{8}{\mu}\log\cN_{\infty,\infty}(\gamma\beta/2, \cF, T) \\
   &~~~~~~~~ + 3e^2\alpha{}KT + 24e\sqrt{\frac{KT}{\mu}}\int_{\alpha}^{\beta}\sqrt{\log\cN_{\infty,\infty}(\gamma\veps, \cF, T)}d\veps.
  \end{align*}
}Using a change of variables and the reparameterization $\alpha'=\alpha\gamma$, $\beta'=\beta\gamma$, the right hand side equals
    {\small
  \begin{align*}
        4\sqrt{2K^{2}T\log\cN_{\infty,\infty}(\beta'/2, \cF, T)} &+  \frac{8}{\mu}\log\cN_{\infty,\infty}(\beta'/2, \cF, T) \\
  & ~~~~~~~~~~ + \frac{1}{\gamma}\prn*{3e^2\alpha{}KT + 24e\sqrt{\frac{KT}{\mu}}\int_{\alpha'}^{\beta'}\sqrt{\log\cN_{\infty,\infty}(\veps, \cF, T)}d\veps}.
  \end{align*}
}Lastly, via~\pref{lem:calibration}, we have
\[
  \sum_{t=1}^{T}\En_{s_t\sim{}p_t}\tri*{s_t, \ls_t}
  \geq{}   \sum_{t=1}^{T}\En_{s_t\sim{}p_t}\frac{\sum_{a\in\brk*{K}}s_t(a)\ls_{t}(a)}{\sum_{a\in\brk*{K}}s_t(a)} = \sum_{t=1}^{T}\En_{a_t\sim{}P_t}\ls_t(a_t).
\]
Finally, the definition of the smoothed distribution $P_{t}^{\mu}$ and boundedness of $\ls$ immediately implies
\begin{align*}
\sum_{t=1}^{T}\En_{a_t\sim{}P_t}\ls_t(a_t) \geq{} \sum_{t=1}^{T}\En_{a_t\sim{}P_t^{\mu}}\ls_t(a_t) - \mu{}KT.\tag*\qedhere
\end{align*}
% (so $\cS=(\ramp\circ\cF)(\cX)$).
\end{proof}

\begin{proof}[\pfref{prop:entropy_growth}]
  Suppose $\log\cN_{\infty,\infty}(\veps, \cF, T)\propto \veps^{-p}$.
  \begin{itemize}
  \item When $p\geq{}2$, it suffices to set $\beta=\mathrm{rad}_{\infty,\infty}(\cF,T)$, $\mu=(KT)^{-1/(p+1)}\gamma^{-p/(p+1)}$, and $\alpha=1/(KT\mu)^{1/p}$ in \pref{thm:chaining_ramp} to obtain $\wt{O}\prn*{(KT/\gamma)^{p/(p+1)}}$ .
      \item When $p\in(0,2]$, it suffices to set $\alpha=1/(KT)$, $\mu=(KT)^{-2/(p+4)}\gamma^{-2p/(4+p)}$, and $\beta=\gamma^{2/(2+p)}/(KT\mu)^{1/(2+p)}$ in \pref{thm:chaining_ramp} to obtain $\wt{O}\prn*{(KT)^{(p+2)/(p+4)}\gamma^{-2p/(p+4)}}$.
  \end{itemize}
  
  For the parametric case, set $\alpha=\beta=\gamma/KT$ and $\mu=\sqrt{d\log(KT/\gamma)/KT}$ to conclude the bound.
  
  Similarly, in the finite class case, set $\alpha=\beta=0$ and $\mu=\sqrt{\log\abs*{\Pi}/KT}$.
\end{proof}

\begin{proof}[\pfref{ex:rademacher}]
Let $\cF|_a = \crl*{x\mapsto{}f(x)_a\mid{}f\in\cF}$. Then clearly it holds that 
\[
\log\cN_{\infty,\infty}(\veps, \cF, T) \leq \sum_{a\in\brk*{K}}\log\cN_{\infty}(\veps, \cF|_a, T) \leq{} K\max_{a\in\brk*{K}}\log\cN_{\infty}(\veps, \cF|_a, T),
\]
where have dropped the second ``$\infty$'' subscript on the right-hand side to denote that this is the covering number for a scalar-valued class. Let $a^{\star}$ be the action that obtains the maximum in this expression. Returning to the integral expression in \pref{thm:chaining_ramp}, we have just shown an upper bound of
\[
3e^2\alpha{}KT + 24eK\sqrt{\frac{T}{\mu}}\int_{\alpha}^{\beta}\sqrt{\log\cN_{\infty}(\veps, \cF|_{a^{\star}}, T)}d\veps.
\]
For any scalar-value function class $\cG\subseteq{}\prn*{\cX\to\brk*{0,1}}$, define
\[
\mathfrak{R}(\cG,T) = \sup_{\bx}\En_{\eps}\sup_{g\in\cG}\sum_{t=1}^{T}\eps_tg(\bx_{t}(\eps)).
\]
Following the proof of Lemma 9 in~\citet{rakhlin2015sequential}, by choosing $\beta=1$ and $\alpha=2\mathfrak{R}(\cF|_{a^{\star}},T)/T$, we may upper bound the $L_{\infty}$ covering number by the sequential Rademacher complexity (via fat-shattering), to obtain
\[
6eK\mathfrak{R}(\cF|_{a^{\star}},T) + 96\sqrt{2}eK\sqrt{\frac{1}{\mu}}\mathfrak{R}(\cF|_{a^{\star}},T)\int_{2\mathfrak{R}(\cF|_{a^{\star}},T)/T}^{1}\frac{1}{\veps}\sqrt{\log(2eT/\veps)}d\veps.
\]
Using straightforward calculation from the proof of Lemma 9 in~\citet{rakhlin2015sequential}, this is upper bounded by
\[
O\prn*{
\frac{K}{\sqrt{\mu}}\mathfrak{R}(\cF|_{a^{\star}},T)\log^{3/2}(T/\mathfrak{R}(\cF|_{a^{\star}},T))
}.
\]
Returning to the regret bound in \pref{thm:chaining_ramp}, we have shown an upper bound of
\[
O\prn*{
\frac{K}{\gamma\sqrt{\mu}}\mathfrak{R}(\cF|_{a^{\star}},T)\log^{3/2}(T/\mathfrak{R}(\cF|_{a^{\star}},T))
 + \mu{}KT},
\]
where we have used that $\log\cN_{\infty}(1, \cF|_{a^{\star}}, T)=0$ under the boundedness assumption on $\cF$.
Setting $\mu\propto(\mathfrak{R}(\cF|_{a^{\star}},T)/(T\gamma))^{2/3}$ yields the result.
\end{proof}

\begin{proof}[\pfref{ex:linear}]
This is an immediate consequence of \pref{ex:rademacher} and that Banach spaces for which the martingale type property holds with constant $\beta$ have sequential Rademacher complexity $O(\sqrt{\beta{}T})$~\citep{srebro2011universality}.
\end{proof}

  \subsection{Additional results}
\label{ssec:additional}
Here we briefly state an analogue of \pref{thm:chaining_ramp} for the hinge loss. Note that this bound leads to the same exponents for $T$ as \pref{thm:chaining_ramp}, but has worse dependence on the margin $\gamma$ and depends on the scale parameter $B$ explicitly.
\begin{theorem}[Contextual bandit chaining bound for hinge loss]
\label{thm:chaining_hinge}
For any fixed constants $\beta>\alpha>0$, hinge loss parameter $\gamma>0$, and smoothing parameter $\mu\in(0,1/K]$ there exists an adversarial contextual bandit strategy $(P_t)_{t\leq{}T}$ with expected regret bounded as
{\small
  \begin{align*}
   \En\brk*{\sum_{t=1}^{T}\ls_t(a_t)} 
    \leq{} &\frac{1}{K}\Biggl\{\inf_{f\in\cF}\En\brk*{\sum_{t=1}^{T}\tri*{\hinge(f(x_t)), \ls_t}} + 
        \frac{1}{\gamma}\sqrt{2K^{2}B^2T\log\cN_{\infty,\infty}(\beta/2, \cF, T)}  + \mu{}BK^2T \\
    &~~~~~~~+  \frac{8B}{\gamma\mu}\log\cN_{\infty,\infty}(\beta/2, \cF, T) + \frac{1}{\gamma}\prn*{3e\alpha{}KT + 24e\sqrt{\frac{KT}{\mu}}\int_{\alpha}^{\beta}\sqrt{\log\cN_{\infty,\infty}(\veps, \cF, T)}d\veps}\Biggr\},
  \end{align*}
}where we recall $B=\sup_{f\in\cF}\sup_{f\in\cX}\nrm*{f(x)}_{\infty}$.
\end{theorem}

%%% Local Variables:
%%% mode: latex
%%% TeX-master: "paper"
%%% End:

\section{Analysis of \lmc}
\label{app:lmc}
% !TEX root = arxiv.tex

This appendix contains the proofs of \pref{thm:hinge_lmc} and
the corresponding corollaries. The proof has many ingredients which we
compartmentalize into subsections. First, in \pref{ssec:lmc}, we analyze the sampling routine, showing
that Langevin Monte Carlo can be used to generate a sample from an
approximation of the exponential weights distribution. Then, in
\pref{ssec:continuous_hedge}, we derive the regret bound for
the continuous version of \expweights. Finally, we put the components together
together, instantiate all parameters, and compute the final regret and
running time in \pref{ssec:lmc_assembly}. The corollaries are straightforward and proved in~\pref{ssec:lmc_corollaries}

To begin, we restate the main theorem, with all the assumptions and
the precise parameter settings.
\begin{theorem}
\label{thm:proj_lmc}
Let $\Fcal$ be a set of functions parameterized by a compact convex
set $\Theta \subset \RR^d$ that contains the origin-centered Euclidean
ball of radius $1$ and is contained within a Euclidean ball of radius
$R$. Assume that $f(x;\theta)$ is convex in $\theta$ for each
$x \in \Xcal$, and that $\sup_{x,\theta} \|f(x;\theta)\|_{\infty} \leq B$,
that $f(x,a;\theta)$ is $L$-Lipschitz as a function of $\theta$ with
respect to the $\ell_2$ norm for each $x,a$. For any $\gamma$,
if we set
\begin{align*}
\eta = \sqrt{\frac{d\gamma^2 \log(RLTK/\gamma)}{5K^2B^2T}}, \qquad \mu = \sqrt{\frac{1}{K^{2}T}}, \qquad M = \sqrt{T},
\end{align*}
in \lmc, and further set
\begin{align*}
u &= \frac{1}{T^{3/2}LB_{\ls}R\eta\sqrt{d}}, \qquad\lambda = \frac{1}{8T^{1/2}R^3}, \qquad\alpha = \frac{R^2}{N},\\
%N &= \otil\left(T^{36}L^{48}R^{30} d^{18}(K\gamma)^{-24}\right), 
N &= \otil\prn*{R^{18}L^{12}T^{6}d^{12} + \frac{R^{24}L^{48}d^{12}}{K^{24}}},
\qquad m = \otil\left(T^{3}dR^4L^2B_{\ls}^{2}(K\gamma)^{-2}\right), 
\end{align*}%\dylan{where is $\tau$ in LMC specified?}
in each call to Projected LMC, then \lmc{} guarantees
\begin{align*}
\sum_{t=1}^T \EE \ell_t(a_t) &\leq \min_{\theta \in \Theta} \sum_{t=1}^T \EE \langle \ell_t, \hinge(f(x_t;\theta))\rangle + \frac{\sqrt{T}}{\gamma} + \frac{2d}{K\eta}\log(RLTK/\gamma) + \frac{10\eta}{\gamma^2}B^2KT\\
& \leq \min_{\theta \in \Theta} \sum_{t=1}^T \EE \langle \ell_t, \hinge(f(x_t;\theta))\rangle + \otil\left(\frac{B}{\gamma}\sqrt{dT}\right).
\end{align*}
Moreover, the running time of \lmc{} is  $\otil\prn*{
\frac{R^{22}L^{14}d^{14}B_{\ls}^{2}T^{10}}{K^2\gamma^2}
+ \frac{R^{28}L^{50}d^{14}B_{\ls}^{2}T^{4}}{K^{26}\gamma^{2}}
}$.
\end{theorem}

\subsection{Analysis of the sampling routine}
\label{ssec:lmc}

In this section, we show how Projected LMC can be used to generate a
sample from a distribution that is close to the exponential weights
distribution. Define
\begin{align}
  \label{eq:langevin_potential}
  F(\theta) = \eta \sum_{\tau=1}^t\langle \tilde{\ell}_\tau, \hinge(f(x_\tau;\theta))\rangle,\quad P(\theta) \propto \exp(-F(\theta)).
\end{align}
We are interested in sampling from $P(\theta)$.

\begin{algorithm}
  \begin{algorithmic}
      \State Input: Parameters
    $m,u,\lambda,N,\alpha$.

%    \State Input: Function $F$, domain $\Theta$, parameters
%    $m,u,\lambda,N,\alpha$.
    \State Set $\tilde{\theta}_0 \gets 0 \in \RR^d$
    \For{$k = 1,\ldots,N$}
    \State Sample $z_1,\ldots,z_m \iidsim \Ncal(0, u^2 I_{d})$ and form the function
    \begin{align*}
      \tilde{F}_k(\theta) = \frac{1}{m}\sum_{i=1}^m F(\theta+z_i) + \frac{\lambda}{2} \|\theta\|_2^2.
    \end{align*}
    \State Sample $\xi_k \sim \Ncal(0,I_d)$ and update
    \begin{align*}
      \tilde{\theta}_{k} \gets \Pcal_{\Theta}\left(\tilde{\theta}_{k-1} - \frac{\alpha}{2}\nabla \tilde{F}_k(\tilde{\theta}_{k-1}) + \sqrt{\alpha}\xi_k\right).
    \end{align*}
    \EndFor
    \State Return $\tilde{\theta}_N$. 
  \end{algorithmic}
  \caption{Smoothed Projected Langevin Monte Carlo for \pref{eq:langevin_potential}}
  \label{alg:langevin_mc}
\end{algorithm}

Let us define the Wasserstein distance. For random
variables $X,Y$ with density $\mu,\nu$ respectively
\begin{align*}
  \Wcal_1(\mu,\nu) \triangleq \inf_{\pi \in \Gamma(\mu,\nu)} \int \|X - Y\|_2 d\pi(X,Y) = \sup_{f : \textrm{Lip}(f) \leq 1} \left|\int f(d\mu(X) - d\nu(Y))\right|.
\end{align*}
Here $\Gamma(\mu,\nu)$ is the set of couplings between the two
densities, that is the set of joint distributions with marginals equal
to $\mu,\nu$. $\mathrm{Lip}(f)$ is the set of all functions that are $1$-Lipschitz with respect to $\ls_2$.

\begin{theorem}
  \label{thm:lmc}
  Let $\Theta \subset \RR^d$ be a convex set containing a Euclidean
  ball of radius $r=1$ with center $0$, and contained within a
  Euclidean ball of radius $R$. Let
  $f: \Xcal\times\Theta \to \RR^K_{=0}$ be convex in $\theta$ with
  $f_a(x;\cdot)$ being $L$-Lipschitz w.r.t. $\ell_2$ norm for each
  $a \in \Acal$. Assume $\|\tilde{\ell}_\tau\|_1 \leq B_\ell$ and
  define $F$ and $P$ as in~\pref{eq:langevin_potential}. Let a target accuracy $\tau>0$ be fixed. Then
  \pref{alg:langevin_mc} with
  parameters $m, N, \lambda, u, \alpha \in \textrm{\poly}(1/\tau, d, R, \eta,
  B_\ell, L)$ generates a sample from a distribution $\tilde{P}$
  satisfying
  \begin{align*}
    \Wcal_1(\tilde{P}, P) \leq \tau.
  \end{align*}
  Therefore, the algorithm runs in polynomial time. 
\end{theorem}
The precise values for each of the parameters $m,N,u,\lambda,\alpha$
can be found at the end of the proof, which will lead to a setting of
$\tau$ in application of the theorem.

Towards the proof, we will introduce the intermediate function
$\hat{F}(\theta) = \EE_{Z} F(\theta+Z) + \frac{\lambda}{2} \|\theta\|_2^2$,
where $Z$ is a random variable with distribution $\Ncal(0, u^2I_d)$. This is the randomized smoothing technique studied by
Duchi, Bartlett and Wainwright~\citep{duchi2012randomized}. The
critical properties of this function are
\begin{proposition}[Properties of $\hat{F}$]
  \label{prop:randomized_smoothing}
  Under the assumptions of \pref{thm:lmc}, The function $\hat{F}$
  satisfies
  \begin{enumerate}
    \item $F(\theta) \leq \hat{F}(\theta) \leq F(\theta) + \eta T B_\ell L u \sqrt{d}/\gamma + \frac{\lambda}{2} R^2$.
    \item $\hat{F}(\theta)$ is $\eta T B_\ell L/\gamma + \lambda{}R$-Lipschitz with respect to the $\ell_2$ norm.
    \item $\hat{F}(\theta)$ is continuously differentiable and its
      gradient is $\frac{\eta T B_\ell L}{u\gamma}+\lambda$-Lipschitz continuous
      with respect to the $\ell_2$ norm. 
    \item $\hat{F}(\theta)$ is $\lambda$-strongly convex with respect to the $\ls_2$ norm. 
    \item $\EE \nabla F(\theta+Z) = \nabla \hat{F}(\theta)$. 
  \end{enumerate}
\end{proposition}
\begin{proof}
  See~\citet[Lemma E.3]{duchi2012randomized} for the proof of all claims, except
  for claim 4, which is an immediate consequence of the $\ell_2$ regularization term.
\end{proof}

Using property 1 in~\pref{prop:randomized_smoothing}
and setting $\veps_1 \triangleq \eta T B_\ell L u \sqrt{d}/\gamma + \lambda R^2$,
we know that
\begin{align*}
  e^{-\veps_1} \exp(- F(\theta) \leq \exp(-\hat{F}(\theta)) \leq \exp(-F(\theta)),
\end{align*}
pointwise. Therefore, defining $\hat{P}$ to be the distribution with
density $\hat{p}(\theta) = \exp(-\hat{F}(\theta))/\hat{Z}$, where
$\hat{Z} = \int \exp(-\hat{F}(\theta))d\theta$, we have
\begin{align*}
TV(P \dmid{} \hat{P}) = \int \frac{e^{-F(\theta)}}{Z} \left| \frac{e^{-\hat{F}(\theta)+F(\theta)}}{\hat{Z}/Z} - 1\right| d\theta \leq e^{\veps_1}-1 \leq 2\veps_1,
\end{align*}
for $\veps_1 \le 1$. This shows that $\hat{P}$ approximates $P$ well when $u$ and $\lambda$ are sufficiently small. The next lemma further shows that the $\tilde{F}_k$ functions themselves approximate $\hat{F}$ well.

\begin{lemma}[Properties of $\tilde{F}_k$]
\label{lem:smoothing_concentration}
For any fixed $\theta$, $k \in [N]$, and constant $\veps_2>0$, 
\begin{align*}
  \PP\brk*{\nrm*{\nabla \hat{F}(\theta) - \nabla \tilde{F}_k(\theta)}_2 \geq \veps_2 + \frac{2}{\sqrt{m}}\cdot\frac{\eta{}TB_{\ls}L}{\gamma}} \leq \exp \left(\frac{-4\veps_2^2\gamma^2 m}{(\eta TLB_\ell{})^2}\right).
\end{align*}
\end{lemma}
\begin{proof}[\pfref{lem:smoothing_concentration}]
Let $k$ be fixed. Since $\tilde{F}_k$ are identically distributed for all $k$ we will henceforth abbreviate to $\tilde{F}$.

  We proceed using a crude concentration argument. Observe that by
  \pref{prop:randomized_smoothing},
  $\EE\nabla \tilde{F}(\theta) = \nabla \hat{F}(\theta)$ and moreover
  $\nabla\tilde{F}(\theta)$ is a sum of $m$ i.i.d., vector-valued
  random variables (plus the deterministic regularization term). 
  
Via the Chernoff method, for any fixed $\theta$, we have
  \begin{align*}
    \PP\left[\|\nabla \tilde{F}(\theta) - \nabla\hat{F}(\theta)\|_2 \ge t\right] &\leq \inf_{\beta > 0 }\exp(-t\beta) \EE \exp(\beta\|\nabla \tilde{F}(\theta) - \nabla\hat{F}(\theta)\|_2)
    \intertext{Using the sum structure and symmetrizing:}
        & \leq \inf_{\beta > 0} \exp(-t\beta)\EE_{z_{1:m}}\En_{\eps} \exp \prn*{2\beta\nrm*{\frac{1}{m}\sum_{i=1}^{m}\eps_t\grad{}G(\theta+z_i)}_{2}},
    \end{align*}
    where $G(\theta) = \eta \sum_{\tau=1}^t\langle \tilde{\ell}_\tau, \hinge(f(x_\tau;\theta))\rangle$.
    Condition on $z_{1:m}$ and let $W(\eps) = \nrm*{\frac{1}{m}\sum_{i=1}^{m}\eps_i\grad{}G(\theta+z_i)}_{2}$. Then for any $i$,
    \begin{align*}
    \abs*{W(\eps_1,\ldots,\eps_i,\ldots,\eps_m)-W(\eps_1,\ldots,-\eps_i,\ldots,\eps_m)}
    &\leq{} \frac{1}{m}\nrm*{\grad{}G(\theta+z_i)}_{2} \\
    &\leq{} \frac{\eta}{m}\sum_{\tau=1}^{t}\nrm*{\tilde{\ls}_{\tau}}_{1}\nrm*{\grad{}\hinge(f(x_{\tau};\theta+z_i))_a}_{2} \\
     &\leq{} \frac{\eta{}TB_{\ls}L}{m\gamma}.
    \end{align*}
    By the standard bounded differences argument (e.g. \citep{boucheron2013concentration}), this implies that $W-\En{}W$ is subgaussian with variance proxy $\sigma^{2}=\frac{1}{4m}\prn*{\frac{\eta{}TB_{\ls}L}{\gamma}}^{2}$. Furthermore, the standard application of Jensen's inequality implies that $\En{}W\leq{}2\sigma$.
    
    Returning to the upper bound, these facts together imply
    \[
    \En_{\eps} \exp \prn*{2\beta\nrm*{\frac{1}{m}\sum_{i=1}^{m}\eps_t\grad{}G(\theta+z_i)}_{2}}
    \leq{} \exp(2\beta^{2}\sigma^{2} + 4\beta\sigma).
    \]
    The final bound is therefore,
    \[
        \PP\left[\|\nabla \tilde{F}(\theta) - \nabla\hat{F}(\theta)\|_2 \ge t\right] \leq \inf_{\beta > 0 }\exp(-t\beta + 2\beta^{2}\sigma^{2} + 4\beta\sigma).
    \]
Rebinding $t=t'+4\sigma$ for $t'\geq{}0$, we have
    \[
        \PP\left[\|\nabla \tilde{F}(\theta) - \nabla\hat{F}(\theta)\|_2 \ge t'+4\sigma\right] \leq \inf_{\beta > 0 }\exp(-t'\beta + 2\beta^{2}\sigma^{2}) = \exp(-(t')^{2}/8\sigma^{2}).
    \]    
\end{proof}

Now, for the purposes of the proof, suppose we run the Projected LMC
algorithm on the function $\hat{F}$, which generates the iterate
sequence $\hat{\theta}_0 = 0$
\begin{align*}
  \hat{\theta}_k \gets \Pcal_\Theta\left(\hat{\theta}_{k-1} - \frac{\alpha}{2}\nabla\hat{F}(\hat{\theta}_{k-1}) + \sqrt{\alpha}\xi_k\right).
\end{align*}

Owing to the smoothness of $\hat{F}$, we may apply the analysis of
Projected LMC due to Bubeck, Eldan, and
Lehec~\citep{bubeck2015sampling} to bound the total variation distance
between the random variable $\hat{\theta}_N$ and the distribution with
density proportional to $\exp(-\hat{F}(\theta))$.

\begin{theorem}[\cite{bubeck2015sampling}]
  \label{thm:bubeck_lmc}
  Let $\hat{P}$ be the distribution on $\Theta$ with density
  proportional to $\exp(-\hat{F}(\theta))$. For any $\veps > 0$ and
  with $\alpha = \tilde{\Theta}(R^2/N)$, we have
  $TV(\hat{\theta}_N, \hat{P}) \leq \veps$ with
  \begin{align*}
    N \ge \tilde{\Omega}\left(\frac{R^6\max\{d, R\eta T B_\ell L/\gamma+R^2\lambda,
    R(\eta TB_\ell L/(u\gamma)+\lambda)\}^{12}}{\veps^{12}}\right).
  \end{align*}
\end{theorem}
This specializes the result of~\citet{bubeck2015sampling}
to our setting, using the Lipschitz and smoothness constants from~\pref{prop:randomized_smoothing}.

Unfortunately, since we do not have access to $\hat{F}$ in closed
form, we cannot run the Projected LMC algorithm on it
exactly. Instead, \pref{alg:langevin_mc} runs LMC on the sequence of approximations $\tilde{F}_k$ and
generates the iterate sequence $\tilde{\theta}_k$. The last step in the
proof is to relate our iterate sequence $\tilde{\theta}_k$ to a hypothetical
iterate sequence $\hat{\theta}_k$ formed by running Projected LMC on
the function $\hat{F}$.

\begin{lemma}
  \label{lem:lmc_contraction}
  Let $\veps_2$ be fixed. Assume the conditions of~\pref{thm:lmc}---in particular that
  \begin{align*}
  m \ge 16(\eta T L B_\ell/\gamma)^2 \log(4R/\alpha\veps_2)/\veps_2^2, \qquad \alpha \leq 2 (\eta T B_\ell L/(u\gamma)+\lambda)^{-1}. 
  \end{align*}
  Then for any $k\in\brk*{N}$ we have
  \begin{align*}
    \Wcal_1(\hat{\theta}_k, \tilde{\theta}_k) \leq k \alpha \veps_2.
  \end{align*}
\end{lemma}
\begin{proof}[\pfref{lem:lmc_contraction}]

  The proof is by induction, where the base case is obvious, since
  $\hat{\theta}_0 = \tilde{\theta}_0$. Now, let $\pi^\star_{k-1}$
  denote the optimal coupling  for
  $\tilde{\theta}_{k-1},\hat{\theta}_{k-1}$ and extend this coupling
  in the obvious way by sampling $z_1,\ldots,z_m$ i.i.d. and by using the
  same gaussian random variable $\xi_k$ in both LMC updates.  Let
  $\Ecal_k = \{z_{1:m} : \|\nabla\tilde{F}(\tilde{\theta}_{k-1}) -
  \nabla\hat{F}(\tilde{\theta}_{k-1})\| \leq \veps_2 + \veps'\}$, where $\veps'\defeq{}\frac{2}{\sqrt{m}}\cdot\frac{\eta{}TB_{\ls}L}{\gamma}$; this is the
  ``good'' event in which the samples provide a high-quality approximation to
  the gradient at $\tilde{\theta}_{k-1}$. We then have
  \begin{align*}
     &\Wcal_1(\hat{\theta}_k, \tilde{\theta}_k) \\&= \inf_{\pi \in \Gamma(\hat{\theta}_k,\tilde{\theta}_k)}\int \|\hat{\theta}_k - \tilde{\theta}_k\|_{2} d\pi\\
    & \leq \int \EE_{z_{1:m}} (\one\{\Ecal_k\} + \one\{\Ecal_k^C\})\|\Pcal_\Theta( \hat{\theta}_{k-1} - \frac{\alpha}{2}\nabla\hat{F}(\hat{\theta}_{k-1})-\sqrt{\alpha}\xi_k) - \Pcal_{\Theta} (\tilde{\theta}_{k-1} - \frac{\alpha}{2}\nabla\tilde{F}(\tilde{\theta}_{k-1})-\sqrt{\alpha}\xi_k) \|_{2} d\pi^\star_{k-1}\\
    & \leq \int \EE_{z_{1:m}}\one\{\Ecal_k\}\|\hat{\theta}_{k-1} - \frac{\alpha}{2}\nabla\hat{F}(\hat{\theta}_{k-1}) - (\tilde{\theta}_{k-1} - \frac{\alpha}{2}\nabla\tilde{F}(\tilde{\theta}_{k-1}))\|_{2} d\pi^\star_{k-1} + 2R\int \PP[\Ecal_{k-1}^C] d\pi^\star_{k-1}\\
    & \leq \int \EE_{z_{1:m}}\one\{\Ecal_k\}\|\hat{\theta}_{k-1} - \frac{\alpha}{2}\nabla\hat{F}(\hat{\theta}_{k-1}) - (\tilde{\theta}_{k-1} - \frac{\alpha}{2}\nabla\tilde{F}(\tilde{\theta}_{k-1}))\|_{2} d\pi^\star_{k-1} + 2R \exp \left(\frac{-4\veps_2^2\gamma^2 m}{(\eta TLB_\ell{})^2}\right).
  \end{align*}
  The first inequality introduces the potentially suboptimal coupling
  $\pi^\star_{k-1}$. In the second inequality we first use that the
  projection operator is contractive, and we also use that the domain
  is contained in a Euclidean ball of radius $R$, providing a coarse
  upper bound on the second term. For the third inequality, we apply
  the concentration argument in~\pref{lem:smoothing_concentration}.  Working just with the
  first term, using the event in the indicator, we have
  \begin{align*}
    & \int \EE_{z_{1:m}}\one\{\Ecal_k\}\|\hat{\theta}_{k-1} - \frac{\alpha}{2}\nabla\hat{F}(\hat{\theta}_{k-1}) - (\tilde{\theta}_{k-1} - \frac{\alpha}{2}\nabla\tilde{F}(\tilde{\theta}_{k-1}))\|_{2} d\pi^\star_{k-1}\\
    & \leq \int \|\hat{\theta}_{k-1} - \frac{\alpha}{2}\nabla\hat{F}(\hat{\theta}_{k-1}) - (\tilde{\theta}_{k-1} - \frac{\alpha}{2}\nabla \hat{F}(\tilde{\theta}_{k-1}))\|_{2} d\pi^\star_{k-1} + \frac{\alpha(\veps_2+\veps')}{2}.
  \end{align*}
  Now, observe that we are performing one step of gradient descent on
  $\hat{F}$ from two different starting points, $\hat{\theta}_{k-1}$
  and $\tilde{\theta}_{k-1}$. Moreover, we know that $\hat{F}$ is
  smooth and strongly convex, which implies that the gradient descent
  update is \emph{contractive}. Thus we will be able to upper bound
  the first term by
  $\Wcal_1(\hat{\theta}_{k-1},\tilde{\theta}_{k-1})$, which will lead
  to the result.

Here is the argument. Consider two arbitrary points
  $\theta,\theta' \in \Theta$. Let
  $G: \theta \to \theta - \alpha/2 \nabla\hat{F}(\theta)$ be a vector
  valued function, and observe that the Jacobian is
  $I - \alpha/2\nabla^2\hat{F}(\theta)$. By the mean value theorem,
  there exists $\theta''$ such that
  \begin{align*}
    \| \theta - \frac{\alpha}{2}\nabla \hat{F}(\theta) - (\theta' - \frac{\alpha}{2}\nabla \hat{F}(\theta')) \|_{2} &\leq \| (I - \alpha/2 \nabla^2 \hat{F}(\theta'')) (\theta-\theta')\|_{2} \\
    &\leq \|I - \alpha/2\nabla^2\hat{F}(\theta'')\|_{\sigma} \|\theta-\theta'\|_{2}.
  \end{align*}
  Now, since $\hat{F}$ is $\lambda$-strongly convex and
  $\eta T B_\ell L/u+\lambda$ smooth, we know that all eigenvalues of
  $\nabla^2 \hat{F}(\theta'')$ are in the interval
  $[\lambda,\eta T B_\ell L/(u\gamma)+\lambda]$. Therefore, if
  $\alpha \leq 2 (\eta TB_\ell L/(u\gamma) + \lambda)^{-1} \leq 1/\lambda$, the
  spectral norm term here is at most $1$, implying that gradient
  descent is contractive. Thus, we get
  \begin{align*}
    \Wcal_1(\hat{\theta}_k, \tilde{\theta}_k) & \leq \int \|\hat{\theta}_{k-1} - \tilde{\theta}_{k-1}\|_{2} d\pi^\star_{k-1} + \frac{\alpha(\veps_2+\veps')}{2} + 2R \exp \left(\frac{-4\veps_2^2\gamma^2 m}{(\eta TLB_\ell{})^2}\right)\\
    & \leq \Wcal_1(\hat{\theta}_{k-1},\tilde{\theta}_{k-1}) + \frac{\alpha}{2}\veps_2+\frac{\alpha}{\sqrt{m}}\cdot\frac{\eta{}TB_{\ls}L}{\gamma} +2R \exp \left(\frac{-4\veps_2^2\gamma^2 m}{(\eta TLB_\ell{})^2}\right).
  \end{align*}
  The choice of $m$ ensures that the second and third term together are at most
  $\alpha\veps_2$, from which the result follows. 
\end{proof}

\begin{fact}
  \label{fact:was_to_tv}
  For any two distributions $\mu,\nu$ on $\Theta$, we have
  \begin{align*}
    \Wcal_1(\mu,\nu) \leq R \cdot TV(\mu,\nu).
  \end{align*}
\end{fact}
\begin{proof}
  We use the coupling characterization of the total variation distance:
  \begin{align*}
    \Wcal_1(\mu,\nu) = \inf_{\pi} \int \|\theta-\theta'\|_{2}d\pi \leq \textrm{diam}(\Theta) \inf_{\pi} \PP_{\pi}[\theta \ne \theta'] \leq R \cdot TV(\mu,\nu). \tag*\qedhere
  \end{align*}
\end{proof}

\begin{proof}[Proof of~\pref{thm:lmc}]
  By the triangle inequality and \pref{fact:was_to_tv} we have
  \begin{align*}
    \Wcal_1(\tilde{\theta}_N, P) \leq \Wcal_1(\tilde{\theta}_N, \hat{\theta}_N) + R \cdot\left(TV(\hat{\theta}_N, \hat{P}) + TV(\hat{P},P) \right).
  \end{align*}
  The first term here is the Wasserstein distance between our true
  iterates $\tilde{\theta}_N$ and the idealized iterates from running
  LMC on $\hat{F}$, which is controlled by~\pref{lem:lmc_contraction}. The second is the total variation
  distance between the idealized iterates and the smoothed density
  $\hat{P}$, which is controlled in~\pref{thm:bubeck_lmc}. Finally, the third term is the
  approximation error between the smoothed density $\hat{P}$ and the
  true, non-smooth one $P$. Together, for any choice of $\veps>0$ and $\veps_2>0$ we obtain the bound
  \begin{align}
    \Wcal_1(\tilde{\theta}_N, P) \leq N \alpha\veps_2 + R\veps + 2R(\eta T B_\ell L u\sqrt{d}/\gamma + \lambda R^2),
    \label{eq:was_final}
  \end{align}
under the requirements
  \begin{align}
    N &\geq \frac{c_0 R^6 \max\{d, R\eta T B_\ell L/\gamma+R^2\lambda, R(\eta T B_\ell L/(u\gamma)+\lambda)\}^{12}}{\veps^{12}},\label{eq:N_constraint} \\m &\geq \frac{16(\eta T L B_\ell/\gamma)^2\log(4R/\alpha\veps_2)}{\veps_2^2}.\notag
  \end{align}
  There are also two requirements on $\alpha$, one arising from~\pref{thm:bubeck_lmc} and the other from \pref{lem:lmc_contraction}. These are:
  \begin{align}
    \alpha &\leq 2 ( \eta T B_\ell L/(u\gamma)+\lambda)^{-1}, \quad\text{and}\quad\alpha = c_1 R^2/N,
    \label{eq:alpha_constraints}
  \end{align}
  for any constant $c_1$.

  Returning to the error bound, if we set
  \begin{align*}
    u = \frac{\tau}{8 R \eta T B_\ell L \sqrt{d}}, \quad \textrm{and} \quad \lambda= \frac{\tau}{8R^3},
  \end{align*}
  the last term in \pref{eq:was_final} is at most $\tau/2$. 
  
  We will make the choice $\alpha = c_1R^2/N$. In this case, the values for $u$ and $\lambda$ above, combined with the inequality \pref{eq:alpha_constraints} give the constraint
  \begin{equation}
  \label{eq:N_constraint2}
  N \geq{} 2c_1{}R^{2}\cdot{}\prn*{
  \frac{8\prn*{\eta{}TLB_{\ls}}^{2}R\sqrt{d}}{\gamma{}\tau} + \frac{\tau}{8R^{2}}
  }.
  \end{equation}
  
  Now for the first term in \pref{eq:was_final}, plug in
  the choice $\alpha = c_1R^2/N$ and set $\veps_2 = \tau/(4c_1R^2)$
  so that this term is at most $\tau/4$. For the second term, set
  $\veps = \tau/(4R)$ so that this term is also at most $\tau/4$. With these choices, the requirements on $m$ and $N$ become:
  \begin{align*}
    m \geq \frac{64c_1^2R^4 (\eta T B_\ell/\gamma)^2 \log(\tau/(16RN))}{\tau^2}, \quad\text{and}\quad N \geq c'_0 R^{18}\max\{d, (R\eta T B_\ell
  L/\gamma)^2\sqrt{d}/\tau\}^{12}/\tau^{12},
  \end{align*}
  where we have noted that the first constraint \pref{eq:N_constraint} clearly implies the second constraint \pref{eq:N_constraint2}, and this proves the theorem. 
\end{proof}

\subsection{Continuous exponential weights.}
\label{ssec:continuous_hedge}

The focus of this section of the appendix is \pref{lem:continuous_hedge}, which analyzes a continuous version of the Hedge/exponential weights algorithms in the full
information setting. This lemma appears in various forms in several
places, e.g. \cite{PLG}. For
the setup, consider an online learning problem with a parametric
benchmark class $\cF=\{f(\cdot;\theta)\mid{} \theta \in \Theta\}$ where
$f(\cdot;\theta) \in (\Xcal \to \RR^K_{=0})$ and further assume that
$\Theta \in \RR^d$ contains the centered Euclidean ball of radius
$r=1$ and is contained in the Euclidean ball of radius $R$. Finally,
assume that $f(x;\cdot)_{a}$ is $L$-Lipschitz with respect to $\ell_2$
norm in $\theta$ for all $x\in\cX$.  On each round $t$ an adversary chooses a
context $x_t \in \Xcal$ and a loss vector $\ell_t \in \RR^K_+$, the
learner then choose a distribution $p_t \in \Delta(\Fcal)$ and suffers
loss:
\begin{align*}
  \EE_{f \sim p_t} \langle\ell_t, \hinge(f(x_t))\rangle.
\end{align*}
The entire loss vector $\ell_t$ is then revealed to the learner.
Here, performance is measured via regret:
\begin{align*}
  \Reg(T,\Fcal) \triangleq \sum_{t=1}^T \EE_{f \sim p_t}\langle\ell_t, \hinge(f(x_t))\rangle - \inf_{f \in \Fcal}\sum_{t=1}^T\langle \ell_t, \hinge(f(x_t))\rangle.
\end{align*}
Our algorithm is a continuous version of exponential weights. Starting
with $w_0(f) \triangleq 0$, we perform the updates:
\begin{align*}
  p_t(f) = \frac{\exp(-\eta w_t(f))}{\int_{\Fcal} \exp(-\eta w_t(f))d\lambda(f)}, \quad\text{and}\quad w_{t+1}(f) = w_t(f) + \langle \ell_t, \hinge(f(x_t))\rangle.
\end{align*}
Here $\eta$ is the learning rate and $\lambda$ is the Lebesgue measure
on $\Fcal$ (identifying elements $f\in\cF$ with their representatives $\theta\in\bbR^{d}$). 

With these definitions, the continuous Hedge algorithm enjoys the
following guarantee.

\begin{lemma}
  \label{lem:continuous_hedge}
  Assume that the losses $\ell_t$ satisfy
  $\|\ell_t\|_{\infty} \leq B_\ell$, $\Theta \subset \RR^d$ is
  contained within the Euclidean ball of radius $R$, and
  $f(x;\cdot)_{a}$ is $L$-Lipschitz continuous in the third argument
  with respect to $\ell_2$. Let the margin parameter $\gamma$ be fixed. Then the continuous Hedge algorithm with
  learning rate $\eta>0$ enjoys the following regret guarantee:
  \begin{align*}
    \mathrm{Regret}(T,\Fcal) \leq \inf_{\veps>0}\left\{\frac{TKB_\ell\veps}{\gamma} + \frac{d}{\eta}\log(RL/\veps) + \frac{\eta}{2}\sum_{t=1}^T\EE_{f\sim p_t}\langle\ell_t, \hinge(f(x_t))\rangle^2\right\}.
  \end{align*}
\end{lemma}
\begin{proof}
  Following the standard analysis for continuous Hedge 
(e.g. Lemma 10 in \cite{narayanan2017efficient}), 
  we know that the regret to some benchmark distribution
  $Q \in \Delta(\Fcal)$ is
  \begin{align*}
    \sum_{t=1}^T (\EE_{f \sim p_t} - \EE_{f \sim Q})(\langle \ell_t,\hinge(f(x_t))\rangle = \frac{\KL{}(Q\dmid{}p_0) - \KL{}(Q\dmid{}p_T)}{\eta} + \frac{1}{\eta} \sum_{t=1}^T\KL{}(p_{t-1}\dmid{}p_t).
  \end{align*}
  For the $\KL{}$ terms, using the standard variational representation, we have
  \begin{align*}
    \KL{}(p_{t-1} \dmid{} p_t) &= \log \EE_{f \sim p_{t-1}} \exp\prn*{-\eta \bigg\langle\ell_t, \hinge(f(x_t)) - \EE_{f \sim p_{t-1}}\hinge(f(x_t))\bigg\rangle}\\
                       & \leq \log\prn*{1 + \frac{\eta^{2}}{2}\EE_{f \sim p_{t-1}} \bigg\langle \ell_t, \hinge(f(x_t)) - \EE_{f \sim p_{t-1}}\hinge(f(x_t))\bigg\rangle^2}\\
                       & \leq \frac{\eta^{2}}{2} \EE_{f \sim p_{t-1}} \langle \ell_t, \hinge(f(x_t))\rangle^2.
  \end{align*}
  Here the first inequality is $e^{-x} \leq 1 - x + x^2/2$, using that the term inside the exponential is centered. The second
  inequality is $\log(1+x) \leq x$. 

  Using non-negativity of $\KL{}$, we only have to worry about the
  $\KL{}(Q\dmid{}p_0)$ term. Let $f^\star$ be the minimizer of the cumulative
  hinge loss. Let $\theta^{\star}\in\Theta$ be a representative for $f^{\star}$ and let $Q$ be the uniform distribution on
  $\Fcal_\veps(\theta^\star, x_{1:T}) \triangleq \{\theta : \max_{t\in[T]} \|f(x_t;\theta)- f(x_t;\theta^\star)\|_{\infty} \leq \veps\}$, then we have
  \begin{align*}
    \KL{}(Q\dmid{}p_0) = \int_{f} q(f)\log \prn*{q(f)/p_0(f)} d\lambda(f) = \int dQ(f) \cdot\log \frac{\int_{\Fcal} d\lambda(f)}{\int_{\Fcal_\veps}d\lambda(f)} = \log \frac{\Vol(\Fcal)}{\Vol(\Fcal_\veps(\theta^{\star},x_{1:T})},
  \end{align*}
  where $\Vol(S)$ denotes the Lebesgue integral. We know that
  $\Vol(\Theta) \leq c_d R^d$ where $c_d$ is the Lebesgue volume of
  the unit Euclidean ball and $R$ is the radius of the ball containing
  $\Theta$, and so we must lower bound the volume of
  $\Fcal_{\veps}(f^\star,x_{1:T})$. For this step, observe that by
  the Lipschitz-property of $f$,
  \begin{align*}
    \sup_{x \in \Xcal} \|f(x;\theta) - f(x;\theta^\star)\|_{\infty} \leq L \|\theta - \theta^\star\|_2,
  \end{align*}
  and hence
  $\Fcal_\veps(\theta^\star,x_{1:T}) \supset B_2(\theta^\star,
  \veps/L)$. Thus the volume ratio is
  \begin{align*}
    \frac{\Vol(\Fcal)}{\Vol(\Fcal_{\veps}(\theta^\star, x_{1:T}))} \leq \frac{c_dR^d}{c_d(\veps/L)^d} = (RL/\veps)^{d}.
  \end{align*}

  Finally, using the fact that the hinge surrogate is
  $1/\gamma$-Lipschitz, we know that
  \begin{align*}
    \sum_{t=1}^T\EE_{f \sim Q}\langle \ell_T, \hinge(f(x_t)) - \hinge(f^\star(x_t))\rangle &\leq 
    T B_\ell \sup_{t \in [T], f \in \textrm{supp}(Q)} \| \hinge(f(x_t)) - \hinge(f^\star(x_t))\|_1\\
    & \leq \frac{TKB_\ell \veps}{\gamma}.\tag*\qedhere
\end{align*}

\end{proof}

\subsection{From full information to bandits.}
\label{ssec:lmc_assembly}

We now combine the results of \pref{ssec:lmc} and \pref{ssec:continuous_hedge} to give the final guarantee for \lmc.

We begin by translating the regret bound in~\pref{lem:continuous_hedge}, followed by many steps of
approximation. At round $t$, let $P_t$ denote the Hedge distribution
on $\Theta$ using the losses $\tilde{\ell}_{1:t-1}$. 
Let $\tilde{P}_t$ denote the distribution from which $\theta_t\in\Theta$ is sampled in~\pref{alg:langevin_mc}. 

Let $p_t\in\Delta(\Acal)$ denote the induced distributions on actions induced by $P_{t}$, i.e. the distribution induced by the process $\theta\sim{}P_{t}$, $p_t(a)\propto\hinge(f(x_t;\theta))$. Likewise, let $\sim{p}_t\in\Delta(\Acal)$ be the distribution induced by $\theta\sim{}\tilde{P}_{t}$, $\tilde{p}_t(a)\propto\hinge(f(x_t;\theta))$; in this notation $\tilde{p}_{t}^{\mu}$ is precisely the distribution from which actions are sampled in \pref{alg:hinge_lmc}.

 Recall that we use $\mu$ in the superscript to denote
smoothing (e.g. $p_{t}^{\mu}$). Let $m_t$ denote the random variable sampled at round $t$
to approximate the importance weight. 

We also let $\hat{\ls}_{t}(a)=\frac{\ls_{t}(a)}{\tilde{p}_{t}^{\mu}(a)}\one\crl*{a_t=a}$ denote estimated losses under the true importance weights, which are not explicitly used by \pref{alg:hinge_lmc} but are used in the analysis.

Let $\one_{a}\in\bbR^{K}$ be the vector with $1$ at coordinate $a$ and $0$ at all other coordinates.

\begin{proof}[\pfref{thm:proj_lmc}]The thrust of this proof is to show that the full information bound in \pref{lem:continuous_hedge} does not degrade significantly under importance weighting and under the approximate LMC implementation of continuous exponential weights.\\
\textbf{Variance control}~~
Controlling the variance term in \pref{lem:continuous_hedge} requires
an application of~\pref{lem:ips_variance}.  After taking conditional expectations,
the variance term is
\begin{align*}
  \sum_{t=1}^T \EE_{\theta \sim P_t} \EE_{a_t \sim \tilde{p}_t^\mu} \EE_{m_t} \langle \tilde{\ell}_t, \hinge(f(x_t;\theta))\rangle^2 =   \sum_{t=1}^T \EE_{s \sim P_t} \EE_{a_t \sim \tilde{p}_t^\mu} \EE_{m_t} m_t^2 \langle\ell_t(a_t)\one_{a_t}, s\rangle^2.
\end{align*}
Here we are identifying $s$ with $\hinge(f(x_t;\theta))$ and
marginalizing out $\theta$ in the outermost expectation.  Note that
this is the same definition of $s$ as in~\pref{lem:ips_variance}.

First let us handle the $m_t$ random variable. Note that conditional
on everything up to round $t$ and $a_t$, $m_t$ is distributed
according to a geometric distribution with mean
$\tilde{p}^\mu_t(a_t)$, truncated at $M$. It is straightforward (cf. \cite{neu2013efficient})
to show that $m_t$ is stochastically dominated by a geometric random variable with
mean $\frac{1}{\tilde{p}^\mu_t(a_t)}$ and hence the second moment of this random
variable is at most $\frac{2}{\tilde{p}^\mu_t(a_t)^2}$. Thus, we are
left with
\begin{align*}
  & \leq 2 \sum_{t=1}^T \EE_{s \sim P_t} \EE_{a_t \sim \tilde{p}_t^\mu} \frac{1}{\tilde{p}_t^\mu(a_t)^2} \langle \ell_t(a_t)\one_{a_t},s\rangle^2 \\ & = 2\sum_{t=1}^T \EE_{s \sim P_t} \EE_{a_t\sim \tilde{p}_t^\mu} \langle \hat{\ell}_t,s\rangle^2\\
& \leq 2\sum_{t=1}^T (\EE_{s \sim \tilde{P}_t} - \EE_{s\sim P_t}) \EE_{a_t \sim \tilde{p}_t^\mu} \langle\hat{\ell}_t,s\rangle^2 + \EE_{s \sim \tilde{P}_t} \EE_{a_t \sim \tilde{p}_t^\mu} \langle \hat{\ell}_t,s\rangle^2.
 \end{align*}
 We can apply~\pref{lem:ips_variance} on the second term, since the only condition for the lemma is that the action distribution is induced from the distribution in the outer expectation. It follows that this term is bounded as
 \begin{align*}
   \sum_{t=1}^T \EE_{s \sim \tilde{P}_t} \EE_{a_t \sim \tilde{p}_t^\mu} \langle \hat{\ell}_t,s\rangle^2  \leq  TK^2(1+B/\gamma)^2.
 \end{align*}
 For the first term, evaluating the inner expectation, using the fact
 that $\tilde{p}_t^\mu(a) \ge \mu$ and applying the Lipschitz
 properties of $\hinge(\cdot), f(x;\cdot)$ (in particular that $f(x;\cdot)$ is $L$-Lipschitz with respect to $\ls_2$ and that the Wasserstein distance we work with is defined relative to $\ls_2$) we have
 \begin{align*}
    (\EE_{s \sim \tilde{P}_t} - \EE_{s\sim P_t}) \EE_{a_t \sim \tilde{p}_t^\mu} \langle\hat{\ell}_t,s\rangle^2
   &= \sum_a (\EE_{\theta \sim \tilde{P}_t} - \EE_{\theta \sim P_t}) \frac{\ell^2_t(a)}{\tilde{p}_t^\mu(a)} \hinge(f(x_t;\theta)_a)^2\\
   & \leq 2\frac{(1+B/\gamma)K L}{\gamma \mu} \sup_{g, \|g\|_{\textrm{Lip}} \leq 1} \left| \int g (dP_t - d\tilde{P}_t) \right|
   = 2\frac{(1+B/\gamma)K L}{\gamma \mu} \Wcal_1(P_t, \tilde{P}_t).
 \end{align*}

 Finally, using the Wasserstein guarantee $\Wcal_1(P_t, \tilde{P}_t)\leq{}\tau$ from ~\pref{thm:lmc}, we conclude that the
 cumulative variance term is upper bounded as
\begin{align*}
  \sum_{t=1}^T \EE \langle \tilde{\ell}_t, \hinge(f(x_t;\theta))\rangle^2 \leq \frac{4(1+B/\gamma)KT L\tau}{\gamma\mu} + 2(1+B/\gamma)^2K^2T.
\end{align*}

\paragraph{Bounding regret}
We first relate the cumulative loss under \pref{alg:hinge_lmc} to the cumulative loss of continuous exponential weights. Observe that
\begin{align*}
\sum_{t=1}^T \langle \ell_t, \tilde{p}_t^\mu\rangle &\leq \mu KT + \sum_{t=1}^T \langle \ell_t, \tilde{p}_t\rangle \\
&\leq \mu KT + \frac{1}{K}\sum_{t=1}^T \EE_{\theta \sim \tilde{P}_t} \langle \ell_t, \hinge(f(x_t;\theta))\rangle\\
& \leq \mu KT + \frac{TL\tau}{\gamma} + \frac{1}{K}\sum_{t=1}^T \EE_{\theta \sim P_t} \langle \ell_t, \hinge(f(x_t;\theta))\rangle.
\end{align*}
This first inequality is a straightforward consequence of smoothing,
while the second is a direct application of~\pref{lem:calibration}.

The third inequality is based on the fact
that $\langle \ell_t, \hinge(f(x_t;\theta))\rangle$ is
$KL/\gamma$-Lipschitz in $\theta$ with respect to $\ell_2$ norm under
our assumptions. This step also uses the Wasserstein guarantee in~\pref{thm:lmc} which produces the approximation factor $\tau$.

Following the analysis in \cite{neu2013efficient} and using the boundedness of $\hinge$, the bias introduced due to using geometric resampling with truncation at $M$ instead of exact inverse propensity scores is
\begin{align*}
& \sum_{t=1}^T \EE_{\theta \sim P_t}\langle \ell_t, \hinge(f(x_t;\theta))\rangle 
\leq{} \EE_{a_{1:T},m_{1:T}} \sum_{t=1}	^T \EE_{\theta \sim P_t} \langle \tilde{\ell}_t, \hinge(f(x_t;\theta))\rangle
 + \frac{T(1+B/\gamma)}{eM}.
\end{align*}
For the remaining term, we apply~\pref{lem:continuous_hedge} with
$\veps = \gamma/(TKM)$, since $M$ is an upper bound on the norm $\nrm*{\tilde{\ls}_{t}}_{1}$ of the losses
to the full information algorithm.
\begin{align*}
 & \EE_{a_{1:T},m_{1:T}} \sum_{t=1}	^T \EE_{\theta \sim P_t} \langle \tilde{\ell}_t, \hinge(f(x_t;\theta))\rangle  \\ & \leq \EE \inf_{\theta \in \Theta} \sum_{t=1}^T \langle \tilde{\ell}_t, \hinge(f(x_t,\theta))\rangle + 1+ \frac{d}{\eta}\log(RLTKM/\gamma) + \eta\left(\frac{(1+B/\gamma)KTL\tau}{\gamma\mu} + (1+B/\gamma)^2K^2T\right).
\end{align*}
The first term here is the benchmark we want to compare to, since
$\EE\inf(\cdot) \leq \inf \EE[\cdot]$ and so the regret contains several terms:
\begin{align*}
& \mu KT + \frac{TL\tau}{\gamma} + \frac{T(1+B/\gamma)}{eMK}+ \frac{1}{K} + \frac{d}{K\eta}\log(RLTKM/\gamma) + \frac{\eta}{2K}\left(\frac{2(1+B/\gamma)KTL\tau}{\gamma\mu} + 4(1+B/\gamma)^2K^2T\right)\\
& \leq \mu KT + \frac{TL\tau}{\gamma} + \frac{T(1+B/\gamma)}{eMK} + \frac{1}{K} + \frac{d}{K\eta}\log(RLTKM/\gamma) + \frac{2\eta}{K\gamma^2}\left(\frac{BKTL\tau}{\mu} + 4B^2K^2T\right).
\end{align*}
Here we use the assumption $B/\gamma\ge 1$. We will simplify the
expression to obtain an $\otil(\sqrt{dKT})$-type bound, first set
$\mu = 1/(K\sqrt{T}), M = \sqrt{T}$ and $\tau = \sqrt{1/(TL^2)}$. This gives
\begin{align*}
  & 2\sqrt{T} + \frac{2B}{\gamma}\sqrt{T} + \frac{2d}{K\eta}\log(RLTK/\gamma) + \frac{2\eta}{K\gamma^2}\left(BK^{2}T + 4B^2K^2T\right)\\
  & \leq O(B\sqrt{T}/\gamma) + \frac{2d}{K\eta}\log(RLTK/\gamma) + \frac{10\eta}{\gamma^2}B^2K T.
\end{align*}
Finally set $\eta = \sqrt{\frac{d \gamma^2\log(RLTK/\gamma)}{5 K^2 B^2 T}}$ to get
\begin{align*}
O(\sqrt{T}/\gamma) + O\left(\frac{B}{\gamma} \sqrt{dT\log(RLTK/\gamma)}\right) = \otil(\frac{B}{\gamma}\sqrt{dT}).
\end{align*}
This concludes the proof of the regret bound.

\paragraph{Running time calculation.} 
At each round make $M+1$ calls to the LMC sampling
routine for a total of $O(T^{3/2})$ calls across all rounds. We now bound the running time for a single call.

We always use parameter
$\tau = \sqrt{1/(TL^2)}$ and we know $\|\tilde{\ell}\|_1 \leq 1/\mu =
K\sqrt{T}$ and $\eta = \otil(\sqrt{\frac{d}{K^2T}})$. Plugging into the parameter choices at the end of the proof of~\pref{thm:lmc}, we must sample 
\begin{align*}
  m = \otil(T^3 d R^4L^2B_{\ls}^{2}/(K\gamma)^2)
\end{align*}
samples from a gaussian distribution on each iteration, and the number
of iterations to generate a single sample is:
\begin{align*}
N &= \otil\prn*{R^{18}L^{12}T^{6}d^{12} + \frac{R^{24}L^{48}d^{12}}{K^{24}}}.
\end{align*}
Therefore, the total running time across all rounds is
\[
\otil\prn*{
\frac{R^{22}L^{14}d^{14}B_{\ls}^{2}T^{10}}{K^2\gamma^2}
+ \frac{R^{28}L^{50}d^{14}B_{\ls}^{2}T^{4}}{K^{26}\gamma^{2}}
}.
\]
\end{proof}

\subsection{Proofs for  corollaries}
\label{ssec:lmc_corollaries}
\pref{corr:bandit_multiclass} is an immediate consequence
of~\pref{thm:hinge_lmc}. For~\pref{corr:realizable}, we apply~\pref{lem:hinge_realizable}, since
$\theta^\star \in \Theta$ satsifies the conditions of the lemma
pointwise. Thus
\begin{align*}
K^{-1} \EE \langle \ell_t, \hinge(f(x_t;\theta^\star))\rangle = K^{-1} \EE [\langle \bar{\ell}_t, \hinge(f(x_t;\theta^\star))\rangle \mid x_t] = \EE [\min_a \bar{\ell}_t(a) | x_t].
\end{align*}
Therefore, letting $a^{\star}_t$ denote the optimal action minimizing $\bar{\ls}_t$, we obtain the expected regret bound
\begin{align*}
  \sum_{t=1}^T \EE [\langle \bar{\ell}_t, a_t - a_t^\star\rangle] \leq \otil((B/\gamma)\sqrt{dT}).
\end{align*}

\section{Analysis of \ftl}
\label{app:ftl}
% !TEX root = arxiv.tex

Recall we are in the stochastic setting. Let $\Dcal$ denote the
distribution over $(\Xcal,\RR_+^K)$.

The bulk of the analysis is the following uniform convergence lemma,
which is based on chaining for the function class $\Fcal$. Recall that
$\Ncal_{\infty,\infty}(\veps,\Fcal)$ is the 
$L_{\infty}/\ell_{\infty}$ covering number from \pref{def:cover}.
\begin{lemma}
  \label{lem:ftl_chaining}
  Fix a predictor $\hat{f}$ and let $\{x_i,a_i,\ell_i(a_i)\}_{i=1}^n$ be a dataset of $n$ samples, Suppose that $(x_i,\ls_i)$ are drawn i.i.d. from some distribution $\cD$ and $a_i$ is sampled from $p_i \defeq (1-K\mu)\pihinge{\hat{f}(x_i)} + \mu$. Define $\Rhathinge_n(f) \defeq \frac{1}{n}\sum_{i=1}^n\langle\hat{\ell}_i,\hinge(f(x_i))\rangle$, where $\hat{\ell}_i$ is the importance-weighted loss. Then:
  \begin{align*}
   & \EE \sup_{f \in \Fcal} |\Rhinge(f) - \Rhathinge_n(f)\rangle| \\
& \leq \frac{1}{\gamma} \inf_{\beta \ge 0} \crl*{2K\beta + 12\int_{\beta}^{2} \left(\sqrt{\frac{2K}{n\mu}\log(n\Ncal_{\infty,\infty}(\veps,\Fcal,n))} + \frac{3\log(n\Ncal_{\infty,\infty}(\veps,\Fcal,n))}{n\mu}\right)d\veps}.
  \end{align*}
\end{lemma}
\begin{proof}[\pfref{lem:ftl_chaining}]
Note that since the data-collection policy $\hat{f}$ is fixed, and since we are in the stochastic setting with $(x_i,\ls_i)\sim\cD$, the samples $\{x_i,a_i,\ell_i(a_i)\}_{i=1}^n$ are i.i.d. Consequently, we can apply the standard symmetrization upper bound for uniform convergence. Beginning with
  \begin{align*}
    & \EE_{x_{1:n},a_{1:n},\ell_{1:n}} \sup_{f \in \Fcal} \brk*{\Rhinge(f) - \Rhathinge_n(f)},\\
    \intertext{we introduce a second ``ghost'' dataset of samples $\tau=n+1,\ldots,2n$ via Jensen's inequality.}
    & \leq \EE_{x_{1:2n},a_{1:2n},\ell_{1:2n}} \sup_{f \in \Fcal} \frac{1}{n}\sum_{\tau=n+1}^{2n}\langle \hat{\ell}_\tau,\hinge(f(x_\tau))\rangle - \frac{1}{n}\sum_{\tau=1}^{n} \langle \hat{\ell}_\tau,\hinge(f(x_\tau))\rangle.
    \intertext{Introducing Rademacher random variables and splitting the supremum:}
    & \leq 2 \EE_{x_{1:n},a_{1:n},\ell_{1:n},\epsilon_{1:n}} \sup_{f \in \Fcal} \frac{1}{n}\sum_{\tau=1}^n\epsilon_\tau \langle \hat{\ell}_\tau, \hinge(f(x_\tau))\rangle.
%    & \leq 2 \sup_{x_{1:n}} \EE_{a_{1:n},\ell_{1:n},\epsilon_{1:n}} \sup_{f \in \Fcal} \frac{1}{n}\sum_{\tau=1}^n\epsilon_\tau \langle \hat{\ell}_\tau, \hinge(f(x_\tau))\rangle .
  \end{align*}
  Now condition on $x_{1:n}$ and define a sequence $\beta_i = 2^{1-i}$ for
  $i \in \{0,1,2,\ldots,N\}$, where $N$ is such that $\beta_{N+1}\geq{}\beta\geq{}\beta_{N+2}$ for the value of $\beta$ in the lemma statement. For each
  $\beta_i$ let $V_i$ be a (classical) $L_{\infty}/\ell_{\infty}$ cover for $f$ at scale
  $\beta_i$ on $x_{1:n}$, that is
  \begin{align*}
    \forall f \in \Fcal, \forall i, \exists v \in V_i \textrm{ s.t. } \max_{t\in [n]} \|f(x_t) - v_t\|_{\infty} \leq \beta_i.
  \end{align*}
  We can always ensure $|V_i| \leq{} \Ncal_{\infty,\infty}(\beta_i,\Fcal, n)$ and
  since $\|f(x)\|_{\infty} \leq 1$, we know that
  $\Ncal_{\infty,\infty}(\beta_0,\Fcal,n) \leq 1$. Now, letting $v^{(i)}(f)$ denote the
  covering element for $f$ at scale $\beta_i$, we have
  \begin{align*}
    & \EE_{a_{1:n},\ell_{1:n},\epsilon_{1:n}} \sup_{f \in \Fcal} \frac{1}{n}\sum_{\tau=1}^n\epsilon_\tau \langle \hat{\ell}_\tau, \hinge(f(x_\tau))\rangle \\
    & \leq \EE_{a_{1:n},\ell_{1:n},\epsilon_{1:n}}\sup_{f \in \Fcal} \frac{1}{n}\sum_{\tau=1}^n\epsilon_\tau\langle \hat{\ell}_\tau, \hinge(f(x_\tau)) - \hinge(v^{(N)}_\tau(f))\rangle\\
    & ~~~~+ \sum_{i=1}^N\sup_{f\in\cF}\frac{1}{n}\sum_{\tau=1}^n\epsilon_\tau \langle \hat{\ell}_\tau,\hinge(v^{(i)}_\tau(f)) - \hinge(v^{(i-1)}_\tau(f))\rangle \\
    &~~~~+ \sup_{f\in\cF}\frac{1}{n}\sum_{\tau=1}^n\epsilon_\tau\langle\hat{\ell}_\tau,\hinge(v^{(0)}_\tau(f))\rangle.
  \end{align*}
  Since $|V_0|\leq{}1$, the expected value of the third term is zero. The remaining work is to bound the first and second terms.

  For the first term note that by \Holder's inequality, for any $f\in\cF$,
  \begin{align*}
  \frac{1}{n}\sum_{\tau=1}^n \epsilon_{\tau} \langle \hat{\ell}_\tau, \hinge(f(x_\tau)) - \hinge(v^{(N)}_\tau(f))\rangle & \leq 
  \frac{1}{n}\sum_{\tau=1}^n \|\hat{\ell}_\tau\|_1 \| \hinge(f(x_\tau)) - \hinge(v_\tau^{(N)}(f)) \|_{\infty} \\
& \leq \frac{\beta_N}{\gamma} \frac{1}{n}\sum_{\tau=1}^n\|\hat{\ell}_\tau\|_1,
  \end{align*}
  since $\hinge$ is $1/\gamma$-Lipschitz. Thus for the first term, we have
  \begin{align*}
    \EE_{a_{1:n},\ell_{1:n},\epsilon_{1:n}} \sup_{f \in \Fcal} \frac{1}{n}\sum_{\tau=1}^n \epsilon_\tau\langle \hat{\ell}_\tau, \hinge(f(x_\tau)) - \hinge(v_\tau^{(N)}(x_\tau))\rangle &\leq \frac{\beta_N}{\gamma} \EE_{a_{1:n},\ell_{1:n}} \frac{1}{n}\sum_{\tau=1}^n \|\hat{\ell}_\tau\|_1
    \leq \frac{\beta_NK}{\gamma}.
  \end{align*}
  Note that there is no dependence on the smoothing parameter $\mu$
  here.

  For the second term, let us denote the $i$th term in the summation by
  \begin{align*}
    \EE_{a_{1:n},\ell_{1:n},\epsilon_{1:n}} \underbrace{\sup_{f \in \Fcal} \frac{1}{n}\sum_{\tau=1}^n\epsilon_\tau\langle \hat{\ell}_\tau,\hinge(v^{(i)}_\tau(f)) - \hinge(v^{(i-1)}_\tau(f))\rangle}_{\triangleq\; \Ecal_i}.
  \end{align*}
We control $\Ecal_i$ using Bernstein's inequality and a union bound. First, note that the individual elements in the sum satisfy the deterministic bound
  \begin{align}
  \label{eq:ftl_range}
    |\epsilon_\tau\langle \hat{\ell}_\tau, \hinge(v^{(i)}_\tau(f)) - \hinge(v^{(i-1)}_\tau(f))\rangle |  \leq \frac{3\beta_{i}}{\mu\gamma},
    \end{align}
and the variance bound,
\begin{align}
    \EE \langle \hat{\ell}_\tau, \hinge(v^{(i)}_\tau(f)) - \hinge(v^{(i-1)}_\tau(f))\rangle^2 &\leq \sum_{a} \EE_{a_\tau} \frac{\one\{a_\tau=a\}}{p_\tau(a)^2} (\hinge(v_\tau^{(i)}(f)_a) - \hinge(v_\tau^{(i-1)}(f)_a))^2 \notag\\
    & \leq \sum_a \frac{1}{\mu} (3\beta_i/\gamma)^2 = \frac{9 \beta_i^2K}{\mu\gamma^2}.\label{eq:ftl_var}
  \end{align}
  Here we are using that $v^{(i)}(f)$ and $v^{(i-1)}(f)$ are the
  covering elements for $f$, Lipschitzness of $\hinge$, and the definition of the
  importance weighted loss $\hat{\ell}_\tau$.

Using \pref{eq:ftl_range} and \pref{eq:ftl_var}, Bernstein's inequality (e.g. \cite{boucheron2013concentration}, Theorem 2.9) implies that for any $\delta \in (0,1)$,
  \begin{align*}
    \frac{1}{n}\sum_{\tau=1}^n\epsilon_\tau\langle \hat{\ell}_\tau,\hinge(v_\tau^{(i)}(f)) - \hinge(v_\tau^{(i-1)}(f))\rangle \leq 6 \sqrt{ \frac{\beta_i^2K}{n\mu\gamma^2} \log(1/\delta)} + \frac{6\beta_i}{n\mu\gamma} \log(1/\delta),
  \end{align*}
  with probability at least $1-\delta$. The important point here is
  that $1/(n\mu)$ appears in the square root, as opposed to
  $1/(n\mu^2)$. Via a union bound, for any $\delta\in(0,1)$, with probability at least
  $1-\delta$,
  \begin{align*}
    & \sup_f\frac{1}{n}\sum_{\tau=1}^n\epsilon_\tau\langle \hat{\ell}_\tau,\hinge(v_\tau^{(i)}(f)) - \hinge(v_\tau^{(i-1)}(f))\rangle \\ 
&\leq 6 \sqrt{ \frac{\beta_i^2K}{n\mu\gamma^2} \log(|V_i||V_{i-1}|/\delta)} + \frac{6\beta_i}{n\mu\gamma} \log(|V_i||V_{i-1}|/\delta)\\
    & \leq \frac{6\beta_i}{\gamma} \left(\sqrt{\frac{2K}{n\mu}\log(|V_i|/\delta)} + \frac{2\log(|V_i|/\delta)}{n\mu}\right),
  \end{align*}
  since $|V_{i-1}| \leq |V_i|$. Now, recalling the shorthand definition $\Ecal_i$
  \begin{align*}
    & \EE_{a_{1:n},\ell_{1:n},\epsilon_{1:n}} \Ecal_i \leq \inf_{\zeta} \EE \one\{\Ecal_i \le \zeta\} \cdot \zeta + \EE \one\{\Ecal_i > \zeta\} \cdot \frac{3\beta_i}{\mu\gamma}\\
    & \leq \inf_{\delta \in (0,1)}\frac{6\beta_i}{\gamma} \left(\sqrt{\frac{2K}{n\mu}\log(|V_i|/\delta)} + \frac{2\log(|V_i|/\delta)}{n\mu}\right) + \frac{3\beta_i\delta}{\mu\gamma}.
    \intertext{Choosing $\delta=1/n$:}
    & \leq \frac{6\beta_i}{\gamma} \left(\sqrt{\frac{2K}{n\mu}\log(n|V_i|)} + \frac{3\log(n|V_i|)}{n\mu}\right).
  \end{align*}
  Thus, the second term in the chaining decomposition is
  \begin{align*}
    & \frac{6}{\gamma} \sum_{i=1}^N \beta_i \left(\sqrt{\frac{2K}{n\mu}\log(n|V_i|)} + \frac{3\log(n|V_i|)}{n\mu}\right)\\
    & = \frac{12}{\gamma}\sum_{i=1}^N (\beta_i - \beta_{i+1})\left(\sqrt{\frac{2K}{n\mu}\log(n|V_i|)} + \frac{3\log(n|V_i|)}{n\mu}\right)\\
    & \leq \frac{12}{\gamma}\int_{\beta_{N+1}}^{\beta_0} \left(\sqrt{\frac{2K}{n\mu}\log(n\Ncal_{\infty,\infty}(\beta,\Fcal))} + \frac{3\log(n\Ncal_{\infty,\infty}(\beta,\Fcal))}{n\mu}\right)d\beta.
  \end{align*}
  This concludes the uniform deviation statement. Exactly the same
  argument applies to the other tail, so the bound holds on the
  absolute value.
\end{proof}

\begin{proof}[\pfref{thm:ftl}]
Let us denote the right hand side of~\pref{lem:ftl_chaining}, when the dataset is size $n$, as
$\Delta_n$. Define,
\[
f^\star = \argmin_{f \in \Fcal} \EE \langle \ell,\hinge(f(x))\rangle,
\] 

Since the $m^{\textrm{th}}$ epoch proceeds for $n_m \triangleq 2^m$
rounds, and the predictor that we use in the $m^{\textrm{th}}$ epoch
is the ERM on all of the data from the $(m-1)^{\textrm{st}}$ epoch, the expected cumulative hinge regret for the $m^{\textrm{th}}$ epoch is
\begin{align}
%% & \left(\EE \langle \ell_{t+1}, \hinge(\hat{f}_t(x_{t+1}))\rangle - \EE \langle \ell, \hinge(f^\star(x)) \\&= 
& 2^m \cdot \prn*{\EE \Rhinge(\hat{f}_{m-1}) - \Rhinge(f^\star)}.\notag\\
 \intertext{Using the optimality guarantee for ERM: }
& \leq 2^m \cdot \prn*{\EE \Rhinge(\hat{f}_{m-1}) - \frac{1}{n_{m-1}}\sum_{\tau=n_{m-1}}^{n_{m}-1} \langle \hat{\ell}_\tau, \hinge(\hat{f}_{m-1}(x_\tau))\rangle + \frac{1}{n_{m-1}}\sum_{\tau=n_{m-1}}^{n_{m}-1}\langle \hat{\ell}_\tau, \hinge(f^\star(x_\tau)) - \Rhinge(f^{\star})}\notag\\
& \leq 2^{m+1} \EE \sup_{f} \left|\Rhinge(f) - \Rhathinge_{n_{m-1}}(f)\right|. \notag
\intertext{Using the guarantee from \pref{lem:ftl_chaining}:}
&\leq 2^{m+1}\Delta_{n_{m-1}}.\label{eq:erm_delta}
\end{align}
Summing this bound over all rounds, the cumulative expected regret after the zero-th epoch is $\sum_{m=1}^{\log_2(T)} 2^{m+1} \Delta_{n_{m-1}}$. The zero-th epoch contributes $1/\gamma$ to the regret, which will be lower order. This gives the following upper bound on the cumulative expected hinge loss regret.
\begin{align*}
&\Reg(T,\Fcal) \leq  \sum_{m=1}^{\log_2(T)} 2^{m+1} \Delta_{n_{m-1}}\\
& \leq  \frac{4}{\gamma} \sum_{m=1}^{\log_2(T)} \inf_{\beta > 0} \crl*{n_m K\beta+ 12\cdot2^{m-1}\cdot\int_{\beta}^{2B} \left(\sqrt{\frac{2K}{n_{m-1}\mu} \log(n_{m-1}\Ncal_{\infty,\infty}(\veps,\Fcal))} + \frac{3\log(n_{m-1}\Ncal_{\infty,\infty}(\veps,\Fcal))}{n_{m-1}\mu}\right)d\veps}\\
& \leq  \frac{4}{\gamma} \inf_{\beta > 0} \crl*{KT\beta + 12\log_2(T) \cdot \int_{\beta}^{2B}\left(\sqrt{\frac{2KT}{\mu} \log(T\Ncal_{\infty,\infty}(\veps,\Fcal))} + \frac{3\log(T\Ncal_{\infty,\infty}(\veps,\Fcal))}{\mu}\right)d\veps}.\\
&\defeq{}C.
\end{align*}
Let $z_{t} = \hat{f}_{m-1}(x_t)$ for each time $t$ in epoch $m$. We have just shown 
\[
\sum_{t=1}^{T}\En\tri*{\ls_t,\hinge(z_t)} \leq{} T\cdot{}\EE \langle \ell,\hinge(f^{\star}(x))\rangle + C.
\]
Using \pref{lem:calibration}, this implies
\[
\sum_{t=1}^{T}\En\tri*{\ls_t,\pihinge(z_t)} \leq {}\frac{T}{K}\cdot{}\EE \langle \ell,\hinge(f^{\star}(x))\rangle + \frac{C}{K}.
\]
Finally since $p_{t}=(1-K\mu)\pihinge(z_t) + \mu$ and $\nrm*{\ls_t}_{\infty}\leq{}1$, this implies the bound
\[
\sum_{t=1}^{T}\En\ls_{t}(a_t) \leq {}\frac{T}{K}\cdot{}\EE \langle \ell,\hinge(f^{\star}(x))\rangle + \underbrace{\frac{C}{K} + \mu{}KT}_{\defeq{}\;C'}.
\]
We proceed to bound the final regret $C'$ under the specific covering number behavior assumed in the theorem statement.
Assume that
$\log(\Ncal_{\infty,\infty}(\veps,\Fcal)) \leq \veps^{-p}$ for
some $p > 2$. Omitting the $\log(T)$ additive terms, which will
contribute $O(B\gamma^{-1}\sqrt{KT\log(T)/\mu} +
B\gamma^{-1}\log(T)/\mu)$ to the overall regret, the bound is now
\begin{align*}
& \mu{}KT + \frac{1}{\gamma{}K}\prn*{\inf_{\beta > 0} 4KT\beta + 12\log_2(T)\cdot\int_{\beta}^{2}\sqrt{\frac{2KT}{\mu\veps^{p}}}d\veps + 36\log_2(T)\cdot\int_{\beta}^{2}\frac{1}{\mu\veps^{p}}d\veps}.
\end{align*}
Choosing $\beta = (KT\mu)^{-1/p}$, this bound becomes
\[
O\prn*{\mu{}KT + \frac{1}{\gamma{}K}\log(T) (KT)^{1-1/p}\mu^{-1/p}}.
\]
Finally, we choose $\mu=\gamma^{-\frac{p}{p+1}}T^{-\frac{1}{p+1}}K^{-1}$, leading to a final bound of
$O\prn*{(T/\gamma)^{\frac{p}{p+1}}}$.
\end{proof}

\section{\ftl\xspace for Lipschitz CB}
\label{app:lipschitz}
Here we analyze \ftl\xspace in a stochastic Lipschitz contextual bandit
setting. To describe the setting, let $\Xcal$ be a metric space
endowed with metric $\rho$ and with covering dimension $p$. This
latter fact means that for each $0 < \veps \leq 1$, $\Xcal$ can be
covered using at most $C_\Xcal \veps^{-p}$ balls of radius
$\veps$. Let $\Acal$ be a finite set of $K$ actions.  In this
section, we define the benchmark class
$\Gcal \subset (\Xcal \to \Delta(\Acal))$ to be the set of
$1$-Lipschitz functions, meaning that
$\|g(x) - g(x')\|_1 \leq \rho(x,x')$ for all $g,x,x'$ (The choice of
$\ell_1$ norm is natural since we are operating over the simplex).

We focus on the stochastic setting where there is a distribution
$\Dcal$ over $\Xcal \times [0,1]^K$. At each round
$(x_t,\ell_t)\sim\Dcal$ is drawn and $x_t$ is presented to the
learner. The learner chooses a distribution $p_t \in \Delta(\Acal)$,
 samples an action $a_t \in \Acal$ from $p_t$, and observes the
loss $\ell_t(a_t)$. We measure regret via
\begin{align*}
\Reg(T,\Gcal) = \sum_{t=1}^T \EE \ell_t(a_t) - \inf_{g \in \Gcal}T \EE \langle g(x),\ell\rangle.
\end{align*}    

In this setting, \ftl\xspace takes the following form. Before the $m^{\textrm{th}}$ epoch, we
choose a function $\hat{g}_{m-1}$ by solving the empirical risk minimization (ERM) problem
\begin{align*}
\hat{g}_{m-1} = \argmin_{g \in \Gcal} \sum_{\tau=n_{m-1}}^{n_m-1} \langle \hat{\ell}_\tau, g(x_\tau)\rangle,
\end{align*}
where $\hat{\ell}_\tau$ is the importance weighted loss. Then, we use
$\hat{g}_{m-1}$ for all the rounds in the $m^{\textrm{th}}$ epoch, which means
that after observing $x_t$, we set
$p_t(a) = (1-K\mu)\hat{g}_m(x_t,a) + \mu$. We sample $a_t \sim p_t$,
observe $\ell_t(a_t)$ and use the standard importance weighting
scheme:
\begin{align*}
\hat{\ell}_t(a) = \frac{\ell_t(a_t)\one\{a = a_t\}}{p_t(a)}.
\end{align*}

For this algorithm, we have the following guarantee.
\begin{theorem}
\ftl~in the Lipschitz CB setting enjoys a regret of $\otil((KT)^{\frac{p}{p+1}})$. 
\end{theorem}
This theorem improves upon the recent result
of~\citet{cesa2017algorithmic}, who obtain
$\otil(T^{\frac{p+1}{p+2}})$ in this setting.

\begin{proof}
  We are in a position to apply~\pref{lem:ftl_chaining}. The main
  difference is that there is no margin parameter, since our functions
  are $1$-Lipschitz, instead of $1/\gamma$-Lipschitz after applying
  the surrogate loss. The $\ell_{\infty}$-metric entropy at scale
  $\veps$ is $C_\Xcal \veps^{-p}$ up to polynomial factors in
  $K$ and logarithmic factors, and so in the $m^\textrm{th}$ epoch the ERM has
  sub-optimality (see \pref{eq:erm_delta}) at most
\begin{align*}
\otil\left(\inf_{\beta} K\beta + \int_{\beta}^1 \sqrt{\frac{K\beta^{-p}}{n_{m-1}\mu}} + \frac{\beta^{-p}}{n_{m-1}\mu}\right),
\end{align*}
where $\Otilde$ hides dependence on $C_{\Xcal}$.
Following the argument in the proof of~\pref{thm:ftl}, the
overall regret is then
\begin{align*}
\Reg(T,\Gcal) = \otil\left(\mu{}KT + \inf_{\beta} TK\beta + \int_{\beta}^1 \sqrt{\frac{TK\beta^{-p}}{\mu}} + \frac{\beta^{-p}}{\mu} \right).
\end{align*}
Set $\beta = (TK\mu)^{-1/p}$ and then $\mu = (TK)^{\frac{-1}{p+1}}$ now to
obtain the result.
\end{proof}

In principle our technique can be further extended to the setting
where the action space is also a general metric space, and the losses
are Lipschitz, which is the more general setting addressed
by~\citet{cesa2017algorithmic}. If the action space has covering
dimension $p_\Acal$ then we discretize the action space to resolution
$\epsilon$, set $K = \epsilon^{-p_\Acal}$ in the above argument, and
balance $\epsilon$ with an additional $T\epsilon$ factor that we pay
for discretization. This is the approach used
in~\citet{cesa2017algorithmic} to obtain
$T^{\frac{p+p_\Acal+1}{p+p_\Acal+2}}$. Unfortunately, our argument
above obtains a somewhat poor dependence on $K$ ($K^{\frac{p}{p+1}}$
as opposed to $K^{\frac{1}{p+1}}$, which is more
natural). Consequently, the argument produces a bound of
$\otil(T^{\frac{p + p p_\Acal}{p + p p_\Acal +1}})$ which only
improves on~\citet{cesa2017algorithmic} when $p_\Acal \leq 1/(p-1)$.

\end{document}